\input ./style/arxiv-general.cfg
\documentclass[MSNbibl,number,citesort,seceqn,dvips]{arxbj}
\makeatletter
   \@ifpackageloaded{graphicx}{}{\usepackage{graphicx}}
\makeatother
\usepackage{subenv,subeqn,mathbh,upgreek}

\volume{22}
\issue{2}
\pubyear{2016}
\firstpage{1184}
\lastpage{1226}
\doi{10.3150/14-BEJ690}
\docsubty{FLA}

\makeatletter
\def\pi{\uppi}
\newcommand{\rrvert}{\vert}
\newcommand{\rrVert}{\Vert}
\newcommand{\llvert}{\vert}
\newcommand{\llVert}{\Vert}
\def\mathbbm{\mathbh}

\newremark{rem}{Remark}[section]

\newcommand{\argmin}{\mathop{\arg\min}}
\newcommand{\calP}{\mathcal{P}}
\newcommand{\calId}{\mathcal{I}_d}
\newcommand{\sgn}{\operatorname{sgn}}
\newcommand{\tr}{\operatorname{tr}}

\def\wh{\widehat}
\def\wt{\widetilde}
\def\T{\bar T}
\def\Tp{{\bar T}^{\perp}}

\def\RR{\mathbb R}
\def\EE{\mathbb{E}}
\def\PP{\mathbb{P}}
\def\wh{\widehat}
\def\sumsum{\mathop{\sum\sum}}

\def\calF{\mathcal{F}}
\def\calS{\mathcal{S}}

\newtheorem{theorem}{Theorem}[section]
\newtheorem{lemma}[theorem]{Lemma}
\newtheorem{proposition}[theorem]{Proposition}
\newtheorem{corollary}[theorem]{Corollary}
\newproclaim{definition}[theorem]{Definition}
\makeatother

\begin{document}
\begin{frontmatter}

\title{Adaptive estimation of the copula correlation matrix for
semiparametric elliptical copulas}
\runtitle{Adaptive estimation of elliptical copula correlation matrix}

\begin{aug}
\author[A]{\inits{M.}\fnms{Marten} \snm{Wegkamp}\thanksref{A}\ead[label=e1]{marten.wegkamp@cornell.edu}}
\and
\author[B]{\inits{Y.}\fnms{Yue} \snm{Zhao}\corref{}\thanksref{B}\ead[label=e2]{yz453@cornell.edu}}
\runauthor{M. Wegkamp and Y. Zhao}
\address[A]{Department of Mathematics and Department of Statistical Science,
Cornell University, Ithaca, NY 14853, USA. \printead{e1}}
\address[B]{Department of Statistical Science, Cornell University, Ithaca, NY 14853, USA.\\
\printead{e2}}
\end{aug}

%
\received{\smonth{1} \syear{2014}}
%
\revised{\smonth{9} \syear{2014}}

%
\begin{abstract}
We study the adaptive estimation of copula correlation matrix $\Sigma$
for the semi-parametric elliptical copula model. In this context, the
correlations are connected to Kendall's tau through a sine function
transformation. Hence, a natural estimate for $\Sigma$ is the plug-in
estimator $\widehat\Sigma$ with Kendall's tau statistic. We first
obtain a sharp bound on the operator norm of $\widehat\Sigma- \Sigma
$. Then we study a factor model of $\Sigma$, for which we propose a
refined estimator $\widetilde\Sigma$ by fitting a low-rank matrix plus
a diagonal matrix to $\widehat\Sigma$ using least squares with a
nuclear norm penalty on the low-rank matrix. The bound on the operator
norm of $\widehat\Sigma- \Sigma$ serves to scale the penalty term,
and we obtain finite sample oracle inequalities for $\widetilde\Sigma$.
We also consider an elementary factor copula model of $\Sigma$, for
which we propose closed-form estimators. All of our estimation
procedures are entirely data-driven.
\end{abstract}

%
%
\end{frontmatter}

\section{Introduction}\label{sec1}

\subsection{Background}
\label{sec:Background}

A popular model for high dimensional data is the \emph{semi-parametric
elliptical copula model} \cite
{Embrechts03,Kluppelberg09,Kluppelberg08,NIPS2012_0380}, the family of
distributions whose dependence structures are specified by parametric
elliptical copulas but whose marginal distributions are left
unspecified. The elliptical copula of a $d$-variate distribution from
the semi-parametric elliptical copula model is uniquely characterized
by a \textit{characteristic generator} $\phi$ and a \textit{copula
correlation matrix} $\Sigma\in\RR^{d\times d}$. We refer the readers to
Appendix \ref{sec:def_elliptical_distribution} for a more detailed
discussion about these concepts. For simplicity of presentation, we
will make the blanket assumption that all random vectors we consider
have continuous marginals.

The semi-parametric elliptical copula model includes numerous families
of distributions of popular interest. For instance, we recover from
this model distributions with Gaussian copulas, sometimes referred to
in recent literature as the nonparanormal model \cite{Liu09}, by
choosing the particular characteristic generator $\phi(t)=\exp(-t/2)$.

Throughout the paper, we assume that the random vector $X\in\RR^d$
follows a distribution from the semi-parametric elliptical copula
model, and in particular we let $X$ have copula correlation matrix
$\Sigma$. We let $X^1,\ldots,X^n\in\RR^d$, with $X^i=(X^i_1,\ldots,X^i_d)^T$, be a sequence of independent copies of $X$. We recall the
formulas for (the population version of) Kendall's tau between the
$k$th and $\ell$th coordinates,
%
\begin{equation}
\label{tau} \tau_{k\ell} = \EE \bigl[\sgn\bigl(X^1_k-X^2_k
\bigr)\sgn\bigl(X^1_{\ell
}-X^2_{\ell
}
\bigr) \bigr],
\end{equation}
and the corresponding Kendall's tau statistic,
%
\begin{equation}
\label{tauhat} \wh\tau_{k\ell} = \frac{2}{n(n-1)} \sumsum_{1\le i<j\le n}
\bigl[ \sgn \bigl(X^i_{k}-X^j_{k}
\bigr)\sgn\bigl(X^i_{\ell}-X^j_{\ell}
\bigr) \bigr].
\end{equation}
We let (the population version of) the Kendall's tau matrix $T$ have entries
\[
[T]_{k\ell} = \tau_{k\ell}\qquad \mbox{for all } 1\le k,\ell\le d,
\]
and estimate $T$ using the empirical Kendall's tau matrix $\wh T$ with entries
%
\begin{equation}
[\wh T]_{k\ell} = \wh\tau_{k\ell} \qquad\mbox{for all } 1\le k,\ell \le
d. \label{eq:Kendall_tau_matrix_est}
\end{equation}
We note that $\wh T$ is a matrix $U$-statistic because it can be written as
\[
\wh T = \frac{2}{n(n-1)} \sumsum_{1\le i<j\le n} \bigl[ \sgn
\bigl(X^i-X^j\bigr) \sgn\bigl(X^i-X^j
\bigr)^T \bigr].
\]
In addition, we note the basic facts that $T$ is the correlation matrix
of the centered random vector $\operatorname{sgn}(X^1-X^2)$ and so in
particular is positive semidefinite, that $\wh T$, as a scaled sum of
rank-one positive semidefinite matrices $\sgn(X^i-X^j)\sgn(X^i-X^j)^T$
for $1\le i<j\le n$, is also positive semidefinite, and that $\EE[\wh T]=T$.

For the semi-parametric elliptical copula model, we can relate the
elements of the copula correlation matrix $\Sigma$ to the elements of
the Kendall's tau matrix $T$ independently of the characteristic
generator via the formula
%
\begin{equation}
\label{R_tau} \Sigma= \sin \biggl( \frac{\pi}{2} T \biggr);
\end{equation}
see \cite{Fang02,Hult02,Kendall90,Kruskal58,Lindskog03}. Here and
throughout the paper, we use the convention that the sign, sine and
cosine functions act component-wise when supplied with a vector or a
matrix as their argument; hence equation (\ref{R_tau}) specifies that
\[
[\Sigma]_{k\ell} = \sin \biggl( \frac{\pi}{2} \tau_{k\ell}
\biggr) \qquad\mbox{for all } 1\le k,\ell\le d.
\]
This simple and elegant relationship has contributed to the popularity
of elliptical distributions and the semi-parametric elliptical copula
model, and has led to the widespread application of the \textit{plug-in
estimator} $\wh\Sigma$ of $\Sigma$ given by
%
\begin{equation}
\label{R_hattau} \wh\Sigma= \sin \biggl( \frac{\pi}{2} \wh T \biggr);
\end{equation}
see, for instance, \cite
{Demarta05,Embrechts03,Kluppelberg09,Kluppelberg08,Liu12,Yuan12}. Here,
we briefly review some recent advances involving the plug-in estimator.
\cite{Kluppelberg09} studies the property of $\wh\Sigma$ as an
estimator of $\Sigma$ in the \textit{asymptotic} setting with the
dimension $d$ fixed under the assumption of an \textit{elliptical
copula correlation factor model}, whose precise definition will be
introduced later in Section~\ref{sec:proposed_research}. For
distributions with Gaussian copulas, \cite{Liu12} employs $\wh\Sigma$
to study the estimation of precision matrix, that is, $\Sigma^{-1}$,
under a sparsity assumption on $\Sigma^{-1}$, and a sharp bound on the
element-wise $\ell_{\infty}$ norm of $\wh\Sigma- \Sigma$ is central
to their analysis.\footnote{We note that, under the setting of
distributions with Gaussian copulas, analogous to equation (\ref
{R_tau}), we also have $\Sigma=2\sin((\pi/6)R)$ for $R$ the matrix of
(the population version of) Spearman's rho. Inspired by this
observation, both \cite{Liu12} and \cite{Xue12b} employ $\wh\Sigma
^{\rho
}$, a variant of $\wh\Sigma$ using Spearman's rho statistic, to study
the estimation of precision matrix under this setting. In contrast to
Kendall's tau, however, once we generalize from distributions with
Gaussian copulas to the semi-parametric elliptical copula model,
Spearman's rho is no longer invariant within the family of
distributions with the same copula correlation matrix \cite{Hult02},
that is, a simple relationship analogous to equation (\ref{R_tau})
ceases to exist for Spearman's rho in this wider context. Hence, we do
not pursue an estimation procedure using Spearman's rho.}

\subsection{Proposed research}
\label{sec:proposed_research}

We aim to present in this paper precise estimators of the copula
correlation matrix $\Sigma$.

In Section~\ref{sec:operator_norm}, we focus on the plug-in estimator
$\wh\Sigma$, and present a sharp (upper) bound on the operator norm of
$\wh\Sigma-\Sigma$, which we denote by $\|\wh\Sigma-\Sigma\|_2$.
To the
best of our knowledge, our bound on $\|\wh\Sigma-\Sigma\|_2$ is new,
even for distributions with Gaussian copulas. Here, we list some of the
potential applications of this bound. First, it has often been observed
that the plug-in estimator $\wh\Sigma$ is not always positive
semidefinite \cite{Demarta05,Kluppelberg09}. This not only is a
discomforting problem by itself but also limits the potential
application of the plug-in estimator; for example, certain Graphical
Lasso algorithms \cite{Friedman08} may fail on input that is not
positive semidefinite. We refer the readers to \cite{Xue12a} for a more
detailed discussion and another example involving the Markowitz
portfolio optimization problem. Our bound on $\|\wh\Sigma-\Sigma\|_2$
will precisely quantify the extent to which the nonpositive
semidefinite problem may happen; for instance, if the smallest
eigenvalue of $\Sigma$ exceeds the bound on $\|\wh\Sigma-\Sigma\|_2$,
then $\wh\Sigma$ will be positive definite.

As we were completing this manuscript, we became aware of a result by
Fang Han and Han Liu in \cite{HCL13} that is similar to (our)
inequality (\ref{eq:Sigma_operator_norm}) in Theorem~\ref
{thmm:main_operator_norm_bound}. In deriving their result, they also
employed matrix concentration inequalities to arrive at a version of
inequality (\ref{eq:T_hat_operator_norm_1}); then they invoked
different proof techniques to arrive at a version of Lemma~\ref
{lemma:first_order}, which led to their version of inequality (\ref
{Lipschitz_Kendall}). Our work is independent.

A second application of the bound on $\|\wh\Sigma-\Sigma\|_2$ appears
in Section~\ref{sec:factor_model}. Here, we study the elliptical copula
correlation factor model, which postulates that the copula correlation
matrix $\Sigma$ of $X$ admits the decomposition
%
\begin{equation}
\label{factor_general} \Sigma= \Theta^* + V^*
\end{equation}
for some low-rank or nearly low-rank, positive semidefinite matrix
$\Theta^*\in\RR^{d\times d}$ and some diagonal matrix $V^*\in\RR
^{d\times d}$ with nonnegative diagonal entries. In this case, if
$\Theta^*$ admits the decomposition $\Theta^*=L L^T$ for some $L\in
\RR
^{d\times r}$, then there exists elliptically distributed $\xi\in
\mathcal{E}_{r+d}(0,I_{r+d},\phi)$ (here we invoke the notation of
Definition~\ref{def:elliptical_distribution}) for the $(r+d)\times
(r+d)$ identity matrix $I_{r+d}$ and some characteristic generator
$\phi
$ such that $X$ and $(L,V^{*1/2})\xi$ have the same copula. Here, we
note that the components of $\xi$ are merely uncorrelated, instead of
independent as in the case for standard factor analysis where normality
is assumed. Consideration of the potential dimension reduction offered
by the factor model and the fact that the diagonal elements of the
target copula correlation matrix $\Sigma$ are all equal to one leads us
to propose a refined estimator $\widetilde{\Sigma}$ of $\Sigma$. In
short, we fit the off-diagonal elements of a low-rank matrix to the
off-diagonal elements of $\wh\Sigma$ using least squares with a
nuclear norm penalty on the low-rank matrix; then we obtain the refined
estimator $\wt\Sigma$ from the low-rank matrix by setting the diagonal
elements of the latter to one. The bound on $\|\wh\Sigma-\Sigma\|_2$
will serve to scale the penalty term. As we will discuss in detail in
Section~\ref{sec:refined_estimator}, our problem is a variant of the
matrix completion problem, but in contrast to the existing literature,
the special diagonal structure of $V^*$ enables us to perform much more
precise analysis. In the end, our oracle inequality for $\wt\Sigma$
holds under a single, very mild condition on the low-rank component
$\Theta^*$, and balances the approximation error with the estimation
error, with the latter roughly proportional to the number of parameters
in the model divided by the sample size.

As a warm-up to the general setting above, we will also consider the
\textit{elementary factor copula model}, a special instance of the
elliptical copula correlation factor model in which $V^*$ is
proportional to the $d\times d$ identity matrix $I_d$. For this model,
we will propose and study closed-form estimators.

Throughout our studies, we will provide entirely data-driven estimation
procedures involving explicit constants and measurable quantities. In
addition, we will establish positive semidefinite versions of the
plug-in estimator, the closed-form estimator and the refined estimator
of the copula correlation matrix, with minimal loss in performance.

\subsection{Notation}

For any matrix $A$, we will use $[A]_{k\ell}$ to denote the $k,\ell$th
element of $A$ (i.e., the entry on the $k$th row and $\ell$th column of
$A$). For a vector $x\in\RR^{m}$, we denote by $\operatorname{diag}^{\star
}(x)\in\RR^{m\times m}$ the diagonal matrix with $[\operatorname
{diag}^{\star
}(x)]_{ii}=x_i$ for $i=1,\ldots,m$. We let the constant $\alpha$ with
$0<\alpha<1$ be arbitrary, but typically small; we will normally bound
stochastic events with probability at least $1-\mathcal{O}(\alpha)$. We
let $I_d$ denote the identity matrix in $\RR^{d\times d}$. In this
paper, the majority of the vectors will belong to $\RR^{d}$, and the
majority of the matrices will be symmetric and belong to $\RR^{d\times
d}$; notable exceptions to the latter rule include some matrices of
left or right singular vectors. For notational brevity, we will not
always explicitly specify the dimension of a matrix when such
information could be inferred from the context. The Frobenius inner
product $ \langle\cdot,\cdot \rangle$ on the space of matrices is defined
as $ \langle A,B \rangle=\tr(A^T B)$ for commensurate matrices $A,B$. For
norms on matrices, we use $\|\cdot\|_2$ to denote the operator norm,
$\|
\cdot\|_*$ the nuclear norm (i.e., the sum of singular values), $\|
\cdot
\|_F$ the Frobenius norm resulting from the Frobenius inner product, $\|
\cdot\|_{\infty}$ the element-wise $\ell_{\infty}$ norm (i.e., $\|
A\|
_{\infty}=\max_{k,\ell}|[A]_{k\ell}|$), and $\|\cdot\|_{1}$ the
element-wise $\ell_1$ norm. The effective rank of a positive
semidefinite matrix $A$ is defined as $r_e(A) = \tr(A)/\|A\|_2$. We let
$\lambda_{\max}(\cdot)$ and $\lambda_{\min}(\cdot)$ denote the largest
and the smallest eigenvalues, respectively, and let $\calS^d_+$ be the
set of $d\times d$ correlation matrices, that is, positive semidefinite
matrices with all diagonal elements equal to one. We use $\circ$ to
denote the Hadamard (or Schur) product. For notational brevity when
studying the factor model, for an arbitrary matrix $A\in\RR^{d\times
d}$, we let $A_o\in\RR^{d\times d}$ be the matrix with the same
off-diagonal elements as $A$, but with all diagonal elements equal to
zero, that is,
%
\begin{equation}
A_o = A - I_d\circ A. \label{eq:M_o}
\end{equation}
Again for notational brevity, this time when establishing probability
bounds involving Kendall's tau statistics, we will assume throughout
that the number of samples, $n$, is even, and denote
\[
f(n,d,\alpha) = \sqrt{\frac{16}{3}\cdot\frac{d\cdot\log
(2\alpha^{-1}d)}{n}}.
\]
\textit{Remark}. When $n$ is odd, the appropriate $f$ to use is
\[
f(n,d,\alpha) = \sqrt{\frac{16}{3}\cdot\frac{d\cdot\log
(2\alpha
^{-1}d)}{2 \lfloor n/2\rfloor}}.
\]
This is due to the fact that when $n$ is odd, we can group $X^1,\ldots
,X^n$ into at most $\lfloor n/2\rfloor$ \textit{pairs} of $(X^i,X^j)$'s
such that the different pairs are independent.


\section{Plug-in estimation of the copula correlation matrix}
\label{sec:operator_norm}

In this section, we focus on the plug-in estimator $\wh\Sigma$ of the
copula correlation matrix $\Sigma$ and in particular provide a bound on
$\|\wh\Sigma-\Sigma\|_2$. We recall that $\Sigma$ is related to the
Kendall's tau matrix $T$ via a sine function transformation as in
equation (\ref{R_tau}), and $\wh\Sigma$ is related to the empirical
Kendall's tau matrix $\wh T$ via the same transformation as in
equation (\ref{R_hattau}). We note that a typical proof for a bound on
$\|\wh\Sigma-\Sigma\|_{\infty}$ in the existing literature first
establishes a bound on $\|\wh T-T\|_{\infty}$ through a combination of
Hoeffding's classical bound for the (scalar) $U$-statistic applied to
each element of $\wh T-T$ and a union bound argument, and then
establishes the bound on $\|\wh\Sigma-\Sigma\|_{\infty}$ through the
Lipschitz property of the sine function transformation \cite{Liu12}.
Our proof for the bound on $\|\wh\Sigma-\Sigma\|_2$ is similarly
divided into two essentially independent stages:
\begin{longlist}[1.]
\item[1.]
First, in Section~\ref{sec:T_hat}, we establish a bound on $\|\wh T-T\|
_2$. This stage can be considered as the matrix counterpart in terms of
the operator norm to Hoeffding's classical bound for the (scalar) $U$-statistic;
\item[2.]
Next, in Section~\ref{sec:T_hat_to_Sigma_hat}, we bound $\|\wh\Sigma
-\Sigma\|_2$ by a constant times $\|\wh T-T\|_2$ up to an additive
quadratic term in $f(n,d,\alpha)$. This stage can be considered as the
matrix counterpart in terms of the operator norm to the Lipschitz
property of the sine function transformation. Then, combined with the
bound on $\|\wh T-T\|_2$, we establish the bound on $\|\wh\Sigma
-\Sigma
\|_2$.
\end{longlist}

\subsection{Bounding \texorpdfstring{$\|\widehat{T}-T\|_2$}{||widehat{T}-T||2}}
\label{sec:T_hat}
In this section, we bound $\|\wh T-T\|_2$, establishing both
data-driven and data-independent versions. We rely on the results from
\cite{Tropp14} out of the vast literature on matrix concentration
inequalities (see \cite{Bunea12,Vershynin12} for a glimpse of the literature).

\begin{theorem}
\label{thmm:T_hat_operator_norm}
We have, with probability at least $1-\alpha$,
\begin{subequations}
%
\begin{eqnarray}
\|\wh T-T\|_2 & <& \max \bigl\{ \sqrt{\|T\|_2}f(n,d,
\alpha), f^2(n,d,\alpha) \bigr\} \label{eq:T_hat_operator_norm_1}
\\
&\le&\sqrt{\|\widehat{T}\|_2 f^2(n,d,\alpha) +
\tfrac
{1}{4}f^4(n,d,\alpha) } + \tfrac{1}{2}f^2(n,d,
\alpha) \label
{eq:T_hat_operator_norm_2}
\\
&<& \max \bigl\{ \sqrt{\|{T}\|_2} f(n,d,\alpha), f^2(n,d,
\alpha) \bigr\} + f^2(n,d,\alpha). \label{eq:T_hat_operator_norm_3}
\end{eqnarray}
\end{subequations}
\end{theorem}

\textit{Remark}.
By decoupling the matrix $U$-statistic $\widehat{T}-T$ using (\ref
{eq:T_hat_minus_T_decomposition}), and \cite{Tropp14}, inequality (6.1.3) in
Theorem~6.1.1, we can also obtain a bound on $\EE [ \|
\wh
T-T\|_2  ]$. We omit the details.

\begin{pf*}{Proof of Theorem~\ref{thmm:T_hat_operator_norm}}
The proof can be found in Section~\ref{sec:proof_operator_norm}.
\end{pf*}

We elaborate the results presented in Theorem~\ref
{thmm:T_hat_operator_norm}. First, we note that the bound offered by
inequality (\ref{eq:T_hat_operator_norm_1}) is the tightest, but
contains the possibly unknown population quantity $\|T\|_{2}$. Hence,
we also derive a data-driven bound (\ref{eq:T_hat_operator_norm_2}),
whose performance is in turn guaranteed by~(\ref
{eq:T_hat_operator_norm_3}) in terms of the deterministic $\|T\|_2$.
Theorem~\ref{thmm:T_hat_operator_norm} also shows that the right-hand
side of (\ref{eq:T_hat_operator_norm_2}) is no more than
$f^2(n,d,\alpha
)$ away from the right-hand side of (\ref{eq:T_hat_operator_norm_1}).
This is because the former is sandwiched between the right-hand sides
of (\ref{eq:T_hat_operator_norm_1}) and (\ref
{eq:T_hat_operator_norm_3}), and the latter two terms differ by
$f^2(n,d,\alpha)$.

Next, for latter convenience, we note that when $n$ is large enough
such that
%
\begin{equation}
\|T\|_2\ge f^2(n,d,\alpha) = \frac{16 }{3} \cdot
\frac{d \cdot\log
(2\alpha^{-1}d)}{n}, \label{eq:elementary_n_condition}
\end{equation}
the first term dominates the second term in the curly bracket on the
right-hand side of (\ref{eq:T_hat_operator_norm_1}), that is,
%
\begin{equation}
\max \bigl\{\sqrt{\|T\|_2}f(n,d,\alpha), f^2(n,d,\alpha)
\bigr\} = \sqrt {\|T\|_2}f(n,d,\alpha). \label{eq:T_hat_bound_first_term_dominate}
\end{equation}

Finally, we discuss the optimality of Theorem~\ref
{thmm:T_hat_operator_norm}, specifically inequality (\ref
{eq:T_hat_operator_norm_1}). First, we compare our result to some
recent upper bounds established by other authors under conditions
related to but more restrictive than the semi-parametric elliptical
copula model. Under the same model but with the additional ``sign
sub-Gaussian condition,'' \cite{HCL13} establishes in their
Theorem~4.10 that
%
\begin{equation}
\|\wh T-T\|_2 = \mathcal{O} \biggl( \|T\|_2 \sqrt{
\frac{d + \log
(\alpha
^{-1})}{n}} \biggr) \label{eq:HCL13_bound}
\end{equation}
with probability at least $1-2\alpha$. Meanwhile, for distributions
with Gaussian copulas, \cite{Mitra14} establishes in their Corollary~3
a more complicated bound which, in the regime $n\ge d$, $\|T\|_2\ge
\max
\{\log(d),\log(\alpha^{-1})\}$ and $\|\Sigma\|_{2,\max} \le\|
\Sigma\|
_2^{1/2}$, reduces to that inequality (\ref{eq:HCL13_bound}) holds with
probability at least $1-\alpha$. Here $\|\Sigma\|_{2,\max} = \max_{\|u\|
=1}\|\Sigma u\|_{\max}$ with $\|\cdot\|$ and $\|\cdot\|_{\max}$ being
the Euclidean norm and the element-wise $\ell_{\infty}$ norm for
vectors, respectively.

Such bounds, which are based on Gaussian concentration inequalities,
are of a different flavor. Nevertheless, here we will attempt a very
crude comparison. We set $\alpha=1/d$ so that both our
inequality (\ref
{eq:T_hat_operator_norm_1}) and inequality (\ref{eq:HCL13_bound}) hold
with probability at least $1-\mathcal{O}(1/d)$. We also assume that $n$
is large enough such that inequality (\ref{eq:elementary_n_condition})
holds. Then the right-hand sides of (\ref{eq:T_hat_operator_norm_1})
and (\ref{eq:HCL13_bound}) are $\mathcal{O} ( \sqrt{\|T\|_2 d
\log
(d)/n}  )$ and $\mathcal{O} ( \|T\|_2 \sqrt{d/n}  )$,
respectively. Hence, the bound provided by our inequality (\ref
{eq:T_hat_operator_norm_1}) sheds an operator norm factor $\sqrt{\|T\|
_2}$ at the expense of an extra log factor $\sqrt{\log(d)}$.

From another angle, we contrast our upper bound (\ref
{eq:T_hat_operator_norm_1}) to the corresponding lower bound implied by
the argument presented in the proof of \cite{Lounici14}, Theorem~2, in
the context of covariance matrix estimation. Such a comparison reveals
that our bound (\ref{eq:T_hat_operator_norm_1}) is optimal up to the
(aforementioned) operator norm factor $\sqrt{\|T\|_2}$ and the log
factor $\sqrt{\log(d)}$ in $f(n,d,\alpha)$. The study of if and when
these factors can be removed is beyond the scope of this
paper.\footnote{By our proof of Theorem~\ref
{thmm:T_hat_operator_norm}, inequality (\ref
{eq:T_hat_operator_norm_1}) also holds with the replacement of $\wh T$
by its decoupled version $\wt T$ defined in (\ref{eq:T_tilde}). Then,
by the argument of \cite{Tropp14}, Section~6.1.2, we can show that the
operator norm factor $\sqrt{\|T\|_2}$ is in fact necessary in this
variant of (\ref{eq:T_hat_operator_norm_1}) in terms of $\wt T$ at
least in certain scenarios. Unfortunately, the same argument does not
apply directly to (\ref{eq:T_hat_operator_norm_1}) in terms of the
matrix $U$-statistic $\wh T$.} We also note that, by \cite{Tropp14},
Chapter~7, in inequality (\ref{eq:T_hat_operator_norm_1}),
we could replace the ambient dimension $d$ inside the log function in
$f(n,d,\alpha)$ by $\widetilde{d}=4d/\|T\|_2$. Here, $\widetilde{d}$ is
the effective rank of a semidefinite upper bound of $\EE[(\widetilde
{T}-T)^2]$ with $\widetilde{T}$ defined in equation (\ref{eq:T_tilde}).
Hence, if $\|T\|_2$ is comparable to~$d$, then the log factor is
effectively removed. In large sample size or large dimension setting,
it is customary to set $\alpha$ to be $1/\max\{n,d\}$ so that the
exclusion probability $\alpha$ tends to zero as $n$ or $d$ increases.
For such a setting of $\alpha$, we shed at most a constant
multiplicative factor in the bound on $\|\wh T-T\|_2$ by setting $d$ to
$\widetilde{d}$ inside the log function. Thus, for brevity of
presentation in later sections, we have avoided invoking the effective rank.


\subsection{Bounding \texorpdfstring{$\|\wh\Sigma-\Sigma\|_2$}{||widehatSigma-Sigma||2} in terms of \texorpdfstring{$\|\widehat{T}-T\|_2$}{||widehat{T}-T||2}}
\label{sec:T_hat_to_Sigma_hat}

In this section, we establish in Theorem~\ref
{thmm:main_operator_norm_bound} the promised link between $\|\wh\Sigma
-\Sigma\|_2$ and $\|\wh T-T\|_2$. Based on this result, we establish
bounds on $\|\wh\Sigma-\Sigma\|_2$ in the same theorem.

We also establish in Theorem~\ref{thmm:main_operator_norm_bound} a link
between $\|\wh T'-T\|_2$ and $\|\wh\Sigma'-\Sigma\|_2$, for $\wh T'$
that is any generic estimator of $T$ (i.e., $\wh T'$ is not necessarily
the empirical Kendall's tau matrix $\wh T$), and $\wh\Sigma'$ the
resulting generic plug-in estimator, that is,
\[
\wh\Sigma'=\sin \biggl(\frac{\pi}{2}\wh T'
\biggr).
\]
Possibilities of generic estimators $\wh T'$ of $T$ include regularized
estimators such as thresholding \cite{Bickel08,Cai12} or tapering
\cite
{Cai10} estimator. Such generic estimators $\wh T'$ of $T$ and the
resulting generic plug-in estimators $\wh\Sigma'$ of $\Sigma$ have the
potential to provide faster convergence rate than the empirical
Kendall's tau matrix $\wh T$ and the plug-in estimator $\wh\Sigma$
\textit{if} appropriate structure of $T$ is known in advance so a
regularized estimator $\wh T'$ could be used. Hence, we briefly include
the consideration of generic estimators in Theorem~\ref
{thmm:main_operator_norm_bound}.

An auxiliary result relating $\|T\|_2$ to $\|\Sigma\|_2$ is provided by
Theorem~\ref{thmm:T_Sigma}.

\begin{theorem}
\label{thmm:main_operator_norm_bound}
Let $\wh T'$ be a generic estimator of $T$, and $\wh\Sigma'$ the
resulting generic plug-in estimator of $\Sigma$. We have, for some
absolute constants $C_1',C_2'$ (we may take $C_1' = \pi$ and $C_2' =
\pi
^2/8<1.24$),
%
\begin{equation}
\bigl\| \wh\Sigma' - \Sigma\bigr\|_2 \le C_1'
\bigl\| \widehat{T}'-T \bigr\|_2 + C_2'
\bigl\| \widehat{T}'-T \bigr\|_2^2. \label{Lipschitz_generic}
\end{equation}

Recall $\wh T$ as defined in equation (\ref{eq:Kendall_tau_matrix_est})
and the resulting plug-in estimator $\wh\Sigma$ as defined in
equation (\ref{R_hattau}). We have, for some absolute constants
$C_1,C_2$ (we may take $C_1 = \pi$ and $C_2 = 3 \pi^2/16<1.86$), with
probability at least $1-\frac{1}{4}\alpha^2$,
%
\begin{equation}
\| \wh\Sigma- \Sigma\|_2 \le C_1 \| \widehat{T}-T
\|_2 + C_2 f^2(n,d,\alpha). \label{Lipschitz_Kendall}
\end{equation}

Recall that Theorem~\ref{thmm:T_hat_operator_norm} bounds $\| \widehat
{T}-T \|_2$. Hence, starting from inequality (\ref{Lipschitz_Kendall}),
we have, with probability at least $1-\alpha-\frac{1}{4}\alpha^2$,
\begin{subequations}
%
\begin{eqnarray}
\| \wh\Sigma- \Sigma\|_2 &<& C_1 \max \bigl\{ \sqrt{\|T\|
_2}f(n,d,\alpha), f^2(n,d,\alpha) \bigr\} +
C_2 f^2(n,d,\alpha) \label
{eq:Sigma_operator_norm}
\\
&\le& C_1 \sqrt{\|\widehat{T}\|_2 f^2(n,d,
\alpha) + \tfrac
{1}{4}f^4(n,d,\alpha) } + \bigl(
\tfrac{1}{2}C_1+C_2 \bigr) f^2(n,d,
\alpha) \label{eq:Sigma_operator_norm_2}
\\
&<& C_1 \max \bigl\{ \sqrt{\|{T}\|_2} f(n,d,\alpha),
f^2(n,d,\alpha) \bigr\} + (C_1+C_2
)f^2(n,d,\alpha). \label
{eq:Sigma_operator_norm_3}
\end{eqnarray}
\end{subequations}
\end{theorem}

\begin{pf}
The proof can be found in Section~\ref{sec:proof_operator_norm}.
\end{pf}

We elaborate the results presented in Theorem~\ref
{thmm:main_operator_norm_bound}. First, the relationship between the
bounds~(\ref{eq:Sigma_operator_norm}), (\ref{eq:Sigma_operator_norm_2})
and (\ref{eq:Sigma_operator_norm_3}) is analogous to the relationship
between the bounds (\ref{eq:T_hat_operator_norm_1}), (\ref
{eq:T_hat_operator_norm_2}) and (\ref{eq:T_hat_operator_norm_3}) as has
been discussed following Theorem~\ref{thmm:T_hat_operator_norm}. Next,
we discuss the relative merits of inequalities (\ref
{Lipschitz_generic}) and (\ref{Lipschitz_Kendall}). We note that
\begin{longlist}[1.]
\item[1.]
For the plug-in estimator $\wh\Sigma$, instead of starting from
inequality (\ref{Lipschitz_Kendall}), we can also start from
inequality (\ref{Lipschitz_generic}), take the particular choices $\wh
T'=\wh T$ and $\wh\Sigma'=\wh\Sigma$, and establish a bound on $\|
\wh
\Sigma-\Sigma\|_2$ via inequality (\ref{eq:T_hat_operator_norm_1}) in
Theorem~\ref{thmm:T_hat_operator_norm} as
\[
\| \wh\Sigma- \Sigma\|_2 \le\max \bigl\{ C_1'
\sqrt{\|T\| _2}f(n,d,\alpha) + C_2'\|T
\|_2 f^2(n,d,\alpha), C_1'
f^2(n,d,\alpha) + C_2' f^4(n,d,
\alpha) \bigr\}
\]
with probability at least $1-\alpha$. However, it is obvious that this
bound is not as tight as the one presented in inequality (\ref
{eq:Sigma_operator_norm}), which we obtained via inequality (\ref
{Lipschitz_Kendall}).
\item[2.]
On the other hand, suppose that we have a generic plug-in estimator
$\wh
\Sigma'$ of $\Sigma$ based on a generic estimator $\wh T'$ of $T$ that
achieves a rate $\| \wh T'-T \|_2\ll f(n,d,\alpha)$ (a rate faster than
the one for $\| \wh T-T \|_2$). Then, inequality (\ref
{Lipschitz_generic}) would yield
\[
\bigl\| \wh\Sigma' - \Sigma\bigr\|_2 \ll C_1'
f(n,d,\alpha) + C_2' f^2(n,d,\alpha),
\]
which is tighter than the bound offered by inequality (\ref
{eq:Sigma_operator_norm}).
\end{longlist}
Therefore, whether inequality (\ref{Lipschitz_generic}) or (\ref
{Lipschitz_Kendall}) should be preferred depends on the available
estimator of $T$ and the rate of convergence of the estimator.

Inequalities \textup{(\ref{eq:Sigma_operator_norm})} and \textup{(\ref
{eq:Sigma_operator_norm_3})} in Theorem~\ref
{thmm:main_operator_norm_bound} contain the term $\|T\|_2$. Using the
result of Theorem~\ref{thmm:T_Sigma}, we could relate $\|T\|_2$ back to
$\|\Sigma\|_2$, so that we bound $\| \wh\Sigma- \Sigma\|_2$ directly
in terms of the copula correlation matrix $\Sigma$.

\begin{theorem}
\label{thmm:T_Sigma}
We have
%
\begin{equation}
\frac{2}{\pi}\|\Sigma\|_2\le\|T\|_2\le\|\Sigma
\|_2. \label{eq:T_Sigma}
\end{equation}
Hence, inequalities \textup{(\ref{eq:Sigma_operator_norm})} and \textup{(\ref
{eq:Sigma_operator_norm_3})} hold with $\|T\|_2$ replaced by $\|\Sigma\|_2$.
\end{theorem}

\textit{Remark}. The second half of inequality (\ref{eq:T_Sigma}) is
tight: $\|T\|_2=\|\Sigma\|_2$ when $T=\Sigma=I_d$.

\begin{pf*}{Proof of Theorem~\ref{thmm:T_Sigma}}
The proof can be found in Section~\ref{sec:proof_operator_norm}.
\end{pf*}

\subsection{Obtaining a positive semidefinite estimator \texorpdfstring{$\wh\Sigma^+$}{widehatSigma^+}
from the plug-in estimator \texorpdfstring{$\wh\Sigma$}{widehatSigma}}
\label{sec:Sigma_psd}

As has been mentioned in Section~\ref{sec:proposed_research}, the
plug-in estimator $\wh\Sigma$ may fail to be positive semidefinite. In
this section, we demonstrate a procedure that, in such an event,
obtains an explicitly positive semidefinite estimator $\wh\Sigma^+$ of
$\Sigma$ from $\wh\Sigma$ with minimal loss in performance. The
procedure is suggested by a referee and is inspired by \cite{Xue12a}.
Note that, when $\wh\Sigma$ is not positive semidefinite, we cannot
simply set all the negative eigenvalues of $\wh\Sigma$ to zero, because
the resulting estimator will still not be a correlation matrix,
specifically because some of the diagonal elements of the resulting
estimator will exceed one.

In order to also cover the closed-form estimator and the refined
estimator when we study a factor model for $\Sigma$, we will consider a
more general situation. We let $\|\cdot\|$ be a generic matrix norm and
$\wh\Sigma^{\mathrm{generic}}$ a generic estimator of $\Sigma$. We do
not require $\wh\Sigma^{\mathrm{generic}}$ to be a correlation matrix.
We let the feasible region $\calF\subset\mathbb{R}^{d\times d}$ be such
that $\calF$ is nonempty, closed and convex, satisfies $\calF\subset
\calS^d_+$, but is otherwise arbitrary at this stage. From $\wh\Sigma
^{\mathrm{generic}}$, we construct an estimator $\wh\Sigma^{\mathrm
{generic}+}$ as
%
\begin{equation}
\wh\Sigma^{\mathrm{generic}+} = \argmin_{\Sigma'\in\calF}\bigl\| \Sigma'-\wh
\Sigma^{\mathrm{generic}}\bigr\|. \label{eq:Sigma_psd}
\end{equation}
We note that a solution to the right-hand side of (\ref{eq:Sigma_psd})
always exists. If the norm $\|\cdot\|$ is strictly convex (which is the
case for the Frobenius norm), the solution $\wh\Sigma^{\mathrm
{generic}+}$ is uniquely determined, while if multiple solutions to the
right-hand side of (\ref{eq:Sigma_psd}) exist, we arbitrarily choose
one of the solutions to be $\wh\Sigma^{\mathrm{generic}+}$. By
construction, $\wh\Sigma^{\mathrm{generic}+}$ is a correlation matrix
and so in particular is positive semidefinite. In addition,
Theorem~\ref
{thmm:Sigma_psd} shows that, when $\Sigma\in\calF$, the performance of
$\wh\Sigma^{\mathrm{generic}+}$ is comparable to the performance of
$\wh
\Sigma^{\mathrm{generic}}$ as measured by the deviation from $\Sigma$
in the norm $\|\cdot\|$.

\begin{theorem}
\label{thmm:Sigma_psd}
Suppose that $\Sigma\in\calF$. Then the estimator $\wh\Sigma
^{\mathrm
{generic}+}$ in (\ref{eq:Sigma_psd}) satisfies
\[
\bigl\|\wh\Sigma^{\mathrm{generic}+}-\Sigma\bigr\|\le2\bigl\|\wh\Sigma^{\mathrm
{generic}}-\Sigma\bigr\|.
\]
\end{theorem}

\begin{pf}
The proof can be found in Section~\ref{sec:proof_operator_norm}.
\end{pf}

Theorem~\ref{thmm:Sigma_psd} enables us to obtain from the plug-in
estimator $\wh\Sigma$ a positive semidefinite estimator $\wh\Sigma^+$
of $\Sigma$ such that $\|\wh\Sigma^+-\Sigma\|_2$ is comparable to
$\|\wh
\Sigma-\Sigma\|_2$ and, if necessary, $\|\wh\Sigma^+-\Sigma\|
_{\infty}$
is comparable to $\|\wh\Sigma-\Sigma\|_{\infty}$, as we demonstrate in
Corollary~\ref{cor:Sigma_psd}. As we have mentioned in Section~\ref
{sec:Background}, a sharp bound on the element-wise $\ell_{\infty}$
norm is central in some existing procedures for estimating the
precision matrix $\Sigma^{-1}$.

\begin{corollary}
\label{cor:Sigma_psd}
In (\ref{eq:Sigma_psd}), we let the generic matrix norm $\|\cdot\|$ be
replaced by the operator norm $\|\cdot\|_2$,\vspace*{1pt} the generic estimator
$\wh
\Sigma^{\mathrm{generic}}$ be replaced by the plug-in estimator $\wh
\Sigma$, and the solution $\wh\Sigma^{\mathrm{generic}+}$ be replaced
by $\wh\Sigma^+$. First, we choose $\calF=\calS^d_+$. Then, $\wh
\Sigma
^+$ satisfies
%
\begin{equation}
\bigl\|\wh\Sigma^+-\Sigma\bigr\|_2 \le2\|\wh\Sigma- \Sigma\|_2.
\label{eq:Sigma_hat_p_operator_norm}
\end{equation}
Alternatively, we choose $C_3=\sqrt{3\pi^2/8}<1.93$, and
%
\begin{equation}
\calF= \bigl\{\Sigma'\dvt \Sigma'\in
\calS^d_+ \mbox{ and } \bigl\|\Sigma '-\wh \Sigma
\bigr\|_{\infty}\le C_3 d^{-1/2} f(n,d,\alpha) \bigr\}.
\label{eq:calF_Sigma_hat_infinity}
\end{equation}
Then, with probability at least $1-\frac{1}{4}\alpha^2$, $\wh\Sigma^+$
satisfies inequality (\ref{eq:Sigma_hat_p_operator_norm}) and
%
\begin{equation}
\bigl\|\wh\Sigma^+-\Sigma\bigr\|_{\infty} \le2 C_3 d^{-1/2}
f(n,d,\alpha) \label{eq:Sigma_hat_p_infinity}
\end{equation}
simultaneously. We recall that $\|\wh\Sigma-\Sigma\|_2$ is bounded as
in Theorem~\ref{thmm:main_operator_norm_bound}.
\end{corollary}

\begin{pf}
The proof can be found in Section~\ref{sec:proof_operator_norm}.
\end{pf}


\section{Estimating the copula correlation matrix in the factor model}
\label{sec:factor_model}

In this section, we assume an elliptical copula correlation factor
model for $X\in\RR^{d}$. Recall that, under this assumption, the copula
correlation matrix $\Sigma$ of $X$ can be written as
\[
\Sigma= \Theta^* + V^*
\]
as in equation (\ref{factor_general}), with $\Theta^*\in\RR
^{d\times
d}$ a low-rank or nearly low-rank positive semidefinite\footnote{The
case that $\Theta^*$ is not positive semidefinite, though unnatural
because in the factor model $\Theta^*$ should equal $L L^T$ for some
matrix $L$, can be easily accommodated. We restrict our argument to
positive semidefinite matrices only to take advantage of the notational
brevity offered by the fact that their singular value decomposition and
eigen-decomposition coincide.} matrix, and $V^*\in\RR^{d\times d}$ a
diagonal matrix with nonnegative diagonal entries. Our goal of this
section is to present estimators that take advantage of the potential
dimension reduction offered by the factor model and the special
diagonal structure of $V^*$.

As a prelude to the main result of this section, in Section~\ref
{sec:elementary_model}, we first consider the elementary factor copula
model, for which we study closed-form estimators. Sections \ref
{sec:refined_estimator_preliminary} and \ref{sec:refined_estimator}
form an integral part: in the former, we introduce additional
notation, while in the latter we present our main result of
Section~\ref{sec:factor_model}, specifically by constructing the
refined estimator $\widetilde{\Sigma}$ of $\Sigma$ based on the plug-in
estimator $\wh\Sigma$ and establishing its associated oracle inequality.

\subsection{Analysis of closed-form estimators in the elementary factor
copula model}
\label{sec:elementary_model}

The elementary factor copula model assumes that $\Theta^*\in\RR
^{d\times d}$ is a positive semidefinite matrix of unknown rank $r$
with positive eigenvalues $\lambda_1(\Theta^*)\ge\cdots\ge\lambda
_r(\Theta^*)$, and
%
\begin{equation}
V^*=\sigma^2 I_d
\end{equation}
with $\sigma^2>0$. In other words, the copula correlation matrix
$\Sigma
$ admits the decomposition
\[
\Sigma= \Theta^* + \sigma^2 I_d.
\]
Comparison of the eigen-decomposition
\[
\Theta^* + \sigma^2 I_d = U \operatorname{diag}^{\star}
\bigl(\lambda_1\bigl(\Theta^*\bigr)+ \sigma^2,\ldots,
\lambda_r\bigl(\Theta^*\bigr) +\sigma^2,
\sigma^2,\ldots ,\sigma ^2\bigr) U^T
\]
of $\Sigma$, with the eigen-decomposition $\sum_{k=1}^d\wh\lambda_k
\wh
u_k \wh u_k^T$ (with $\wh\lambda_1\ge\cdots\ge\wh\lambda_d$) of the
plug-in estimator~$\wh\Sigma$, leads us to propose the following
closed-form estimators:
%
\begin{eqnarray}\label{eq:elementary_Theta_hat}
\wh r& =& \sum_{k=1}^d \mathbbm{1} \{
\wh\lambda_k - \wh \lambda _d \ge\mu \},
\nonumber
\\
\wh\sigma^2 & = &\frac{1}{d-\wh r} \sum
_{k>\wh r} \wh\lambda_k,
\\
 \wh\Theta& = & \sum_{k=1}^{\wh r}
\bigl( \wh\lambda_k -\wh\sigma^2 \bigr) \wh
u_k \widehat u_k^T\nonumber
\end{eqnarray}
to estimate $r$, $\sigma^2$ and $\Theta^*$, respectively. Here, $\mu$
is a regularization parameter specified by (\ref{MU}) in Theorem~\ref
{thmm:elementary} below, and is based on the bounds on $\|\wh\Sigma
-\Sigma\|_2$ established earlier. Then we let
%
\begin{equation}
\wt\Sigma^e = \wh\Theta_o + I_d
\label{eq:wt_Sigma_e}
\end{equation}
be the closed-form estimator of $\Sigma$. Note that we do not require
$\wt\Sigma^e = \wh\Theta+ \wh\sigma^2 I_d$. Such a requirement could
be imposed by solving a convex program like (\ref{eq:convex_program})
with the additional constraint that the diagonal elements of $\Theta$
are all equal and are between 0 and 1, but in this section we focus on
closed-form estimators.

Note that, by the construction of $\wh\Theta$ as in (\ref
{eq:elementary_Theta_hat}), the estimated nonzero eigenvalues of $\wh
\Theta$, namely $\wh\lambda_k -\wh\sigma^2$ for $1\le k\le\wh r$,
are always positive. Thus, $\wh\Theta$ is positive semidefinite. On the
other hand, $\wh\sigma^2$ may become negative in the pathological case
when $\wh\Sigma$ is not positive semidefinite. To address this problem,
we could impose a large enough lower bound on $\sigma^2$ so that $\wh
\sigma^2>0$ with high probability. Alternatively, we could replace
$\wh
\Sigma$ by its positive semidefinite version $\wh\Sigma^+$ as
constructed in Corollary~\ref{cor:Sigma_psd} from the very beginning,
and avoid the pathological case altogether. With the bound on $\|\wh
\Sigma^+-\Sigma\|_2$ established in the same corollary, all our
analysis will follow except for some minor changes in absolute
constants. For brevity we omit the details of these changes.


The following theorem summarizes the performance of our closed-form estimators.
%
\begin{theorem}\label{thmm:elementary}
Let $0<\alpha<1/2$, $C_1=\pi$ and $C_2=3\pi^2/16 < 1.86$. We set the
regularization parameter $\mu$ as
%
\begin{equation}
\mu= 2 \bigl\{ C_1 \sqrt{ \| \wh T\|_2
f^2(n,d,\alpha) + \tfrac{1}4 f^4(n,d,\alpha) } +
\bigl(\tfrac{1}{2}C_1+ C_2 \bigr)
f^2(n,d,\alpha) \bigr\}, \label{MU}
\end{equation}
and set
%
\begin{equation}
\bar\mu= 2 \bigl\{ C_1 \sqrt{ \| T\|_2} f(n,d,\alpha) +
(C_1 + C_2 ) f^2(n,d,\alpha) \bigr\}.
\label{BARMU}
\end{equation}
Suppose that $\Theta^*$ satisfies $0<r<d$ and $\lambda_r(\Theta^*)
\ge
2 \bar\mu$, and $n$ is large enough such that inequality (\ref
{eq:elementary_n_condition}) holds. Then, on an event with probability
exceeding $1-2\alpha$,
%
\begin{eqnarray}
\label{eq:r_hat_equal_r} \wh r &=& r,
\\
\label{GRENS} \bigl\| \wt\Sigma^e-\Sigma\bigr\|_F^2
&\le&\bigl\| \wh\Theta-\Theta^*\bigr\|_F^2 \le 2r\bar
\mu^2,
\\
\label{GRENS1} \bigl| \widehat\sigma^2 - \sigma^2 \bigr| & \le&
\tfrac{1}2\bar\mu
\end{eqnarray}
hold simultaneously. If, in addition, the common value of the diagonal
elements of $\Theta^*$ is upper bounded by $1-\sqrt{2r\bar\mu^2}$,
then $\wt\Sigma^e$ is positive semidefinite on the same event.
\end{theorem}

\begin{pf}
The proof can be found in Section~\ref{sec:proof_elementary_model}.
\end{pf}

We elaborate the results presented in Theorem~\ref{thmm:elementary}.
First, the regularization parameter $\mu$, and hence our closed-form
estimators are constructed entirely with explicit constants and
measurable quantities. In addition, in the regime specified by (\ref
{eq:elementary_n_condition}), that is, (roughly) when $n {\|T\|_2}
\gtrsim d\log(2\alpha^{-1}d )$, the rate $2r \bar\mu^2 =\mathcal{O}(
\| T\|_2 \cdot rd \log( 2\alpha^{-1}d) /n)$ in (\ref{GRENS}) is, up to
the operator norm factor $\| T\|_2$ and the logarithmic factor $\log
(2\alpha^{-1}d)$, proportional to the number of parameters in the model
divided by the sample size. Hence, our estimation procedure achieves
correct rank identification for the low-rank component $\Theta^*$, and
near-optimal recovery rate in terms of Frobenius norm deviation for
both $\Theta^*$ and the copula correlation matrix $\Sigma$, in a fully
data-driven manner.

Theorem~\ref{thmm:elementary} also shows that, under appropriate
conditions, if the diagonal elements of $\Theta^*$ are sufficiently
less than one, then the estimator $\wt\Sigma^e$ is positive
semidefinite with high probability. In any case, if $\wt\Sigma^e$ is
not positive semidefinite, we can employ Theorem~\ref{thmm:Sigma_psd} to
obtain from $\wt\Sigma^e$ a positive semidefinite estimator $\wt
\Sigma
^{e+}$ of $\Sigma$ such that $\|\wt\Sigma^{e+}-\Sigma\|_F$ is
comparable to $\|\wt\Sigma^e-\Sigma\|_F$. We defer the details of this
treatment to Corollary~\ref{cor:wt_Sigma_p}.


\subsection{Analysis of the refined estimator: Preliminaries}
\label{sec:refined_estimator_preliminary}

We denote
\[
r^*=\operatorname{rank}\bigl(\Theta^*\bigr).
\]
Let $\Theta^*$ have the eigen-decomposition
\[
\Theta^*=U^* \operatorname{diag}^{\star}\bigl(\lambda_1\bigl(\Theta^*
\bigr),\ldots ,\lambda _{r^*}\bigl(\Theta^*\bigr)\bigr)
U^{*T}.
\]
Here, $\lambda_1(\Theta^*)\ge\cdots\ge\lambda_{r^*}(\Theta^*)$
are the
positive eigenvalues of $\Theta^*$ in descending order, and
\[
U^* = \bigl(u^1,\ldots,u^{r*} \bigr)
\]
is the $d\times r^*$ matrix of the orthonormal eigenvectors of $\Theta
^*$, with the eigenvector $u^i$ corresponding to the eigenvalue
$\lambda
_i(\Theta^*)$.

Furthermore, for all $r$ with $0\le r\le r^*$, we let
%
\begin{equation}
U^*_r= \bigl(u^1,\ldots,u^{r} \bigr)
\label{eq:U_star_r}
\end{equation}
be the $d\times r$ truncated matrix of orthonormal eigenvectors of
$\Theta^*$, let
%
\begin{equation}
\gamma_r = \bigl\|U^*_r U^{*T}_r
\bigr\|_{\infty}, \label{eq:gamma_r}
\end{equation}
and let $\Theta^*_r$ be the best rank-$r$ approximation to $\Theta^*$
in the Frobenius norm, that is, $\Theta^*_r=\argmin_{\Theta\in\RR
^{d\times d},\operatorname{rank}(\Theta)=r}\|\Theta-\Theta^*\|_F$. We note that
$\gamma_r$ is nondecreasing in $r$ on $0\le r\le r^*$, and $\gamma
_{r^*}\le1$. In addition, by Schmidt's approximation theorem~\cite
{Schmit07} or the Eckart--Young theorem~\cite{Eckart36}, for $0\le
r\le
r^*$, we have
%
\begin{equation}
\Theta^*_r = U^*_r \operatorname{diag}^{\star}\bigl(
\lambda_1\bigl(\Theta^*\bigr),\ldots ,\lambda _r\bigl(
\Theta^*\bigr)\bigr) U^{*T}_r, \label{eq:Theta_star_r}
\end{equation}
and $\| \Theta^*_r - \Theta^* \|_F^2 = \sum_{j\dvt r<j\le r^*} \lambda
_j^2(\Theta^*)$.

\subsection{Analysis of the refined estimator: Main result}
\label{sec:refined_estimator}

We first observe that in the elliptical copula correlation factor
model, alternative to (\ref{factor_general}), we can write the copula
correlation matrix $\Sigma$ as
\[
\Sigma= \Theta^*_o + I_d.
\]
This motivates us to set our refined estimator $\widetilde{\Sigma}$ of
$\Sigma$ to be
%
\begin{equation}
\label{eq:Sigma_refined} \wt\Sigma= \wt\Theta_o + I_d.
\end{equation}
Here, $\wt\Theta$ is our estimator of the low-rank component $\Theta
^*$, and is obtained as the solution to a convex program:
%
\begin{equation}
\wt\Theta= \argmin_{\Theta\in\RR^{d\times d}} \bigl\{ \tfrac
{1}{2}\|
\Theta_o-\wh\Sigma_o\|_F^2 +
\mu\|\Theta\|_{*} \bigr\}. \label{eq:convex_program}
\end{equation}
(By its optimality, $\wt\Theta$ must be symmetric, though this
particular property is not used in our subsequent analysis.) In (\ref
{eq:convex_program}), $\mu$ is a regularization parameter chosen
according to (\ref{MU_2}) in Theorem~\ref{thmm:recovery_bound_optimize}
below, and is based on the bounds on $\|\wh\Sigma-\Sigma\|_2$
established earlier.

We now elaborate the construction of the refined estimator. Note that:
\begin{longlist}[1.]
\item[1.]
In the factor model, the off-diagonal elements of $\Sigma$ and $\Theta
^*$ agree, so the off-diagonal elements of $\wh\Sigma$ are natural
estimators of the corresponding elements of $\Theta^*$;
\item[2.]
The plug-in estimator $\wh\Sigma$, similar to the target copula
correlation matrix $\Sigma$, has all its diagonal elements equal to one
irrespective of the low-rank component $\Theta^*$. As a consequence, we
critically lack estimators for the diagonal elements of $\Theta^*$.
\end{longlist}
Because of these observations, when constructing the estimator $\wt
\Theta$ of $\Theta^*$ through the convex program (\ref
{eq:convex_program}), we minimize the Frobenius norm for only the
off-diagonal elements of the deviation between $\wh\Sigma$ and the
estimator of $\Theta^*$ subject to a penalty. The penalty is the
nuclear norm of the estimator of $\Theta^*$ scaled by the
regularization parameter $\mu$, and is implemented to encourage the
estimator of $\Theta^*$ to be appropriately low-rank while keeping
(\ref
{eq:convex_program})\vspace*{1pt} convex \cite{Fazel02}. Then, when constructing the
refined estimator $\wt\Sigma$ of $\Sigma$ from the estimator $\wt
\Theta
$ of $\Theta^*$ through (\ref{eq:Sigma_refined}), we explicitly set all
the diagonal elements of $\wt\Sigma$ to one. It is clear that any bound
on $\widetilde{\Sigma}-\Sigma$ also acts as a bound on the off-diagonal
elements of $\wt\Theta-\Theta^*$ and vice versa. We bound the diagonal
elements of $\wt\Theta-\Theta^*$ in Appendix \ref{sec:diagonal_deviation}.

We briefly contrast our refined estimator $\wt\Sigma$, which is
tailor-made for our special setting of the elliptical copula
correlation factor model, to some of the existing estimation procedures
in related but different contexts.
\begin{longlist}[1.]
\item[1.]
Our setting is an extension of the low-rank matrix approximation
problem \cite{Lounici14,Negahban11,Rohde11}. In particular, \cite
{Lounici14} studies the estimation of $\Theta^*$ that is a covariance
matrix\footnote{For this paragraph only, we use $\Theta^*$ to denote
the covariance matrix, because in the setting of \cite{Lounici14} it is
the covariance matrix itself that has low effective rank.} with low
effective rank, with the added complication that the observations
$X^1,\ldots,X^n$ are masked at \emph{random} coordinates. \cite
{Lounici14} constructs an unbiased initial estimator $\wh\Theta$ of
$\Theta^*$, and further obtains a refined estimator $\wt\Theta$ as the
solution of a convex program that is identical to (\ref
{eq:convex_program}) but with the term $\|\Theta_o-\wh\Sigma_o\|_F^2$
replaced by $\|\Theta-\wh\Theta\|_F^2$, which is a sum over \textit
{all} entries of the matrix $\Theta-\wh\Theta$.

Contrary to the setting of \cite{Lounici14}, $\Sigma$ in the factor
model (\ref{factor_general}) typically has neither low effective rank
nor low rank: because $\tr(\Sigma)=d$, the effective rank of $\Sigma$
is $r_e(\Sigma)=d/\|\Sigma\|_2$, which is large unless $\|\Sigma\|_2$
becomes comparable to $d$; in addition, because $\Theta^*$ is positive
semidefinite, if the diagonal elements of $V^*$ are all strictly
positive, then $\Sigma=\Theta^*+V^*$ has full rank. Hence, a naive
application of the method of \cite{Lounici14} to our setting amounts to
seeking a low-rank approximation to a matrix that is in fact not
low-rank. In contrast, our program (\ref{eq:convex_program}) seeks to
estimate the genuine low-rank or nearly low-rank component $\Theta^*$
of $\Sigma$, even though this choice leads to technical challenges in
our proof as compared to \cite{Lounici14}.
\item[2.]
By the observations we made earlier, our problem can be rephrased as
follows: Estimate the off-diagonal elements of $\Theta^*$ given only
their noisy observations, taking advantage of the fact that $\Theta^*$
is low-rank or nearly low-rank. Hence, as mentioned in Section~\ref
{sec:proposed_research}, our problem is a variant of the matrix
completion problem, in particular the version in which a matrix $\Sigma
$ (not necessarily a correlation matrix) admits a decomposition into
the sum of a low-rank component $\Theta^*$ and a sparse component $S^*$
with a general sparsity pattern (i.e., the locations of the nonzero
entries of the sparse component are unknown but fixed), and the goal is
to estimate $\Sigma$ based on its noisy observation $\wh\Sigma$
\cite
{Agarwal12,Chandrasekaran11,Chandrasekaran12a,Hsu11,Luo11,Zhang12}. In
particular, \cite{Chandrasekaran12a,Hsu11} let $\wt\Theta$, the
estimator of $\Theta^*$, and $\wt S$, the estimator of $S^*$, be the
solution of
%
\begin{equation}
(\wt\Theta,\wt S) = \argmin_{\Theta,S\in\RR^{d\times d}} \bigl\{ \tfrac
{1}{2}\|
\Theta+S-\wh\Sigma\|_F^2 + \mu\|\Theta\|_{*}
+ \lambda\| S\| _{1} \bigr\}. \label{eq:convex_program_generic_sparsity}
\end{equation}
This scenario is the closest to our setting. However, even though $V^*$
in the factor model is indeed a sparse matrix, and thus one could apply
(\ref{eq:convex_program_generic_sparsity}) to our setting, such an
approach would not be optimal because it obviously takes no advantage
of our knowledge of the sparsity pattern of $V^*$, namely the diagonal
pattern. For instance, \cite{Chandrasekaran12a,Hsu11} require
nontrivial specification of an additional regularization parameter
$\lambda=\lambda(\mu)$ for the element-wise $\ell_1$ penalty of the
sparse component. Because (\ref{eq:convex_program}) and (\ref
{eq:convex_program_generic_sparsity}) are distinct programs, it is also
not possible to infer the properties of our refined estimator $\wt
\Sigma
$ directly from the results of \cite{Chandrasekaran12a,Hsu11}.
\item[3.]
Finally, the low-rank and diagonal matrix decomposition problem in the
\textit{noiseless} setting is treated in \cite{Saunderson12}. These
authors employ a semidefinite program, the minimum trace factor
analysis (MTFA), to minimize the trace of the low-rank component
(subject to the constraint that the sum of the low-rank component and
the diagonal component agrees with the given matrix to be decomposed).
The optimality condition from semidefinite programming then gives
fairly simple conditions for the MTFA to exactly recover the decomposition.
\end{longlist}

We adopt the \textit{primal-dual certificate} approach advocated by
\cite{Hsu11,Zhang12}\footnote{Through delicate analysis, \cite
{Chandrasekaran12a} (which builds upon their earlier work \cite
{Chandrasekaran11} in the noiseless setting) guarantees optimal
convergence rate in terms of the operator norm, as well as consistent
rank recovery, for the estimator $\wt\Theta$ of the low-rank component
$\Theta^*$. On the other hand, their analysis requires that the minimum
nonzero singular value of the low-rank component $\Theta^*$ satisfies a
nontrivial lower bound, and hence at this stage is not particularly
well suited to study the case where the low-rank requirement only holds
approximately.} to analyze (\ref{eq:convex_program}). Our oracle
inequality for the refined estimator $\wt\Sigma$ is collected in the
following theorem.

\begin{theorem}
\label{thmm:recovery_bound_optimize}
Let $0<\alpha<1/2$, $C_1=\pi$, $C_2=3\pi^2/16 < 1.86$, and $C=6$. We
set the regularization parameter $\mu$ as
%
\begin{equation}
\mu= C \bigl\{ \label{MU_2}C_1 \sqrt{ \| \wh T\|_2
f^2(n,d,\alpha) + \tfrac{1}4 f^4(n,d,\alpha) } +
\bigl(\tfrac{1}{2}C_1+ C_2 \bigr)
f^2(n,d,\alpha) \bigr\},
\end{equation}
and set
%
\begin{equation}
\bar\mu= C \bigl\{ C_1 \max \bigl[ \sqrt{ \| T\|_2}
f(n,d,\alpha), f^2(n,d,\alpha) \bigr] + (C_1+
C_2 ) f^2(n,d,\alpha) \bigr\}. \label{BARMU_2}
\end{equation}
Recall $\gamma_r$ as defined in (\ref{eq:gamma_r}). We set
%
\begin{equation}
R=\max \bigl\{r\dvt 0\le r\le r^*, \gamma_r \le1/9 \bigr\}.
\label{eq:gamma_R_condition}
\end{equation}
Then, with probability exceeding $1-2\alpha$, the refined estimator
$\wt
\Sigma$, as introduced in (\ref{eq:Sigma_refined}), of $\Sigma$ satisfies
%
\begin{equation}
\| \widetilde{\Sigma} - \Sigma\|_F^2 \le\min
_{0\le r\le R} \biggl\{ \sum_{j\dvt r<j\le r^*}
\lambda_j^2\bigl(\Theta^*\bigr) + 8 r\bar\mu^2
\biggr\}. \label{eq:cor:recovery_bound_master}
\end{equation}
\end{theorem}

\textit{Remark}.
Theorem~\ref{thmm:recovery_bound_optimize} is a specific instance of
Corollary~\ref{cor:recovery_bound} which is a more general result; in
particular the constant $C=6$ in (\ref{MU_2}) and (\ref{BARMU_2}) and
the upper bound $1/9$ on $\gamma_r$ in (\ref{eq:gamma_R_condition}) are
chosen for ease of presentation but are not specifically optimized. For
instance, we could specify a smaller $C$ at the expense of a more
stringent upper bound on $\gamma_r$.

\begin{pf*}{Proof of Theorem \ref{thmm:recovery_bound_optimize}}
The proof can be found in Section~\ref{sec:proof_refined_estimator}.
\end{pf*}

We elaborate the results presented in Theorem~\ref
{thmm:recovery_bound_optimize}.

The oracle inequality (\ref{eq:cor:recovery_bound_master}) in fact
represents the minimum of a collection of upper bounds, and the minimum
is taken over all $r$ that satisfies $\gamma_r=\|U^*_r U^{*T}_r\|
_{\infty}\le1/9$, a range specified by~(\ref{eq:gamma_R_condition}).
Thus, for the oracle inequality (\ref{eq:cor:recovery_bound_master}) to
be as tight as possible, we should ideally have a large range of $r$
such that $\gamma_r\le1/9$. We discuss two concrete examples in which
this condition is satisfied:
\begin{longlist}[1.]
\item[1.]
If for some given $r$, the entries of $u^i$, $1\le i\le r$ are all
bounded by $c/\sqrt{d}$ for some constant $c\ge1$, then $\gamma_r\le
c^2 r/d$;
\item[2.]
Next, we consider the \textit{random orthogonal model} as in \cite
{Candes2009}. The first result of their Lemma~2.2 shows that, if $u^i$,
$1\le i\le r$ are sampled uniformly at random among all families of $r$
orthonormal vectors independently of each other, then there exist
constants $C$ and $c$ such that $\gamma_r\le C \max\{r,\log(d)\}/d$
with probability at least $1-cd^{-3}\log d$.
\end{longlist}
In both cases, $\gamma_r \le1/9$ is satisfied for all $r$'s that are
small compared to $d$ (in the second case when $d$ is large enough and
with high probability to be precise).

The estimation procedure (\ref{eq:convex_program}) is fully
data-driven; in particular, the penalty term in (\ref
{eq:convex_program}) is scaled by a regularization parameter $\mu$
specified by (\ref{MU_2}) with explicit constants and measurable
quantities. In addition, procedure (\ref{eq:convex_program})
automatically balances the approximation error with the estimation
error as if it knows the right model in advance to arrive at the oracle
inequality (\ref{eq:cor:recovery_bound_master}) with near-optimal
recovery rate in terms of Frobenius norm deviation. Specifically,
\begin{longlist}[1.]
\item[1.]
The primal-dual certificate approach yields an approximation error
term, that is, the first term in the curly bracket on the right-hand
side of (\ref{eq:cor:recovery_bound_master}), with leading
multiplicative constant one. Such a feature has become increasingly
common with the results obtained through convex optimization with
nuclear norm penalty \cite{Koltchinskii11,Lounici14};
\item[2.]
Meanwhile, the estimation error term, that is, the second term in the
curly bracket on the right-hand side of (\ref
{eq:cor:recovery_bound_master}), achieves a rate $8 r \bar\mu^2
=\mathcal{O}( \| T\|_2 \cdot rd \log(2\alpha^{-1}d) /n)$ with
probability exceeding $1-2\alpha$ if we focus on the regime specified
by (\ref{eq:elementary_n_condition}), that is, (roughly) when $n {\|T\|
_2} \gtrsim d\log(2\alpha^{-1}d )$. Again, this rate is, up to the
operator norm factor $\| T\|_2$ and the logarithmic factor $\log
(2\alpha
^{-1}d)$, proportional to the number of parameters in the model divided
by the sample size.\footnote{Again by the lower bound argument presented
in the proof of \cite{Lounici14}, Theorem~2, the rate of the estimation
error term in (\ref{eq:cor:recovery_bound_master}) is optimal up to the
operator norm factor and the log factor. We note that the lower bounds
(and in particular the one for Frobenius norm deviation) established by
\cite{Lounici14}, Theorem~2, contain explicit dependence on the operator
norm $\|\Sigma\|_2$ of the target covariance matrix $\Sigma$ in the
form of a multiplicative factor. However, a closer inspection of the
proof of \cite{Lounici14}, Theorem~2, reveals that this particular $\|
\Sigma\|_2$ is in fact restricted to be at most two times the maximum
of the diagonal elements of $\Sigma$, and thus in our case can at most
be two because $\Sigma$ is a correlation matrix. This restriction is
not ideal because $\|\Sigma\|_2$ in general can be as large as $d$. In
our opinion, it remains to be seen how a proper dependence on operator
norm can be obtained in lower bound for Frobenius norm deviation under
our setting of correlation matrix estimation. From another angle, we
have shown in the proof of Corollary~\ref{cor:Sigma_psd} that the
plug-in estimator $\wh\Sigma$ achieves $\|\wh\Sigma-\Sigma\|
_{\infty} =
\mathcal{O}(\sqrt{\log(2\alpha^{-1}d)/n})$ (with probability at least
$1-\frac{1}{4}\alpha^2$); thus $\|\wh\Sigma-\Sigma\|_F^2 =
\mathcal
{O}(d^2\cdot\log(2\alpha^{-1}d)/n)$ (with the same probability). This
rate is slower than $r\bar\mu^2$ so long as $r\|T\|_2\lesssim d$.
Therefore, the presence of $\|T\|_2$ in (\ref{GRENS}) and (\ref
{eq:cor:recovery_bound_master}) entails an upper bound on the rank of
the low-rank component $\Theta^*$ below which the refined estimator and
the closed form estimator in their respective contexts are preferable
to the plug-in estimator $\wh\Sigma$ in terms of Frobenius norm deviation.}
\end{longlist}


Finally, if the diagonal elements of the deviation $\wt\Theta-\Theta^*$
can be appropriately bounded, for instance, through Theorem~\ref
{thmm:Theta_hat_diagonal} in Appendix \ref{sec:diagonal_deviation}, and
if the diagonal elements of $\Theta^*$ are sufficiently smaller than
one, then the estimator $\wt\Sigma$ is positive semidefinite. Because
the argument is similar to the proof of the last statement of
Theorem~\ref{thmm:elementary}, we omit its details. In any case, if
$\wt
\Sigma$ is not positive semidefinite, we can employ Theorem~\ref
{thmm:Sigma_psd} to obtain from $\wt\Sigma$ a positive semidefintie
estimator $\wt\Sigma^{+}$ of $\Sigma$ such that $\|\wt\Sigma
^+-\Sigma\|
_F$ is comparable to $\|\wt\Sigma-\Sigma\|_F$, as Corollary~\ref
{cor:wt_Sigma_p} demonstrates.

\begin{corollary}
\label{cor:wt_Sigma_p}
In (\ref{eq:Sigma_psd}), we let the generic matrix norm $\|\cdot\|$ be
replaced by the Frobenius norm $\|\cdot\|_F$, and let $\calF=\calS
^d_+$. In addition, in the context of the elementary factor copula
model, we let the generic estimator $\wh\Sigma^{\mathrm{generic}}$ be
replaced by the closed-form estimator $\wt\Sigma^{e}$, and the solution
$\wh\Sigma^{\mathrm{generic}+}$ be replaced by $\wt\Sigma^{e+}$, while
in the context of the (general) elliptical copula correlation factor
model, we let the generic estimator $\wh\Sigma^{\mathrm{generic}}$ be
replaced by the refined estimator $\wt\Sigma$, and the solution $\wh
\Sigma^{\mathrm{generic}+}$ be replaced by $\wt\Sigma^{+}$. Then
$\wt
\Sigma^{e+}$ and $\wt\Sigma^+$ satisfy
%
\begin{equation}
\bigl\|\wt\Sigma^{e+}-\Sigma\bigr\|_F \le2\bigl\|\wt\Sigma^e
- \Sigma\bigr\|_F,\qquad \bigl\| \wt \Sigma^+-\Sigma\bigr\|_F \le2\|\wt
\Sigma- \Sigma\|_F. \label{eq:wt_Sigma_p_Frobenius_norm}
\end{equation}
We recall that $\|\wt\Sigma^e - \Sigma\|_F$ and $\|\wt\Sigma-
\Sigma\|
_F$ are bounded as in Theorems \ref{thmm:elementary} and \ref
{thmm:recovery_bound_optimize}, respectively.
\end{corollary}

\textit{Remark}.
We refer the readers to \cite{Qi06aquadratically} and the references
therein for the computational aspect of (\ref{eq:Sigma_psd}) in this
context of Frobenius norm minimization.

\begin{pf*}{Proof of Corollary \ref{cor:wt_Sigma_p}}
With the choice $\calF=\calS^d_+$, we clearly have $\Sigma\in\calF$.
Then (\ref{eq:wt_Sigma_p_Frobenius_norm}) follows straightforwardly
from Theorem~\ref{thmm:Sigma_psd}.
\end{pf*}

For both Corollaries \ref{cor:Sigma_psd} and \ref{cor:wt_Sigma_p}, we
have obtained positive semidefinite, rather than strictly positive
definite, versions of the existing estimators. To obtain strictly
positive definite estimators, we could replace the existing feasible
regions $\calF$ in Corollaries \ref{cor:Sigma_psd} and \ref
{cor:wt_Sigma_p} by an intersection of $\calF$ and the convex set $\{
\Sigma'\in\mathbb{R}^{d\times d}\dvt \lambda_{\min}(\Sigma')\ge
\varepsilon\}$
for some $\varepsilon>0$. Then the resulting estimator from~(\ref
{eq:Sigma_psd}) will be positive definite, with the smallest eigenvalue
lower bounded by $\varepsilon$. If in addition the copula correlation
matrix $\Sigma$ satisfies $\lambda_{\min}(\Sigma)\ge\varepsilon$, the
conclusions of Corollaries \ref{cor:Sigma_psd} and \ref{cor:wt_Sigma_p}
will continue to hold.


\section{Proofs for Section \texorpdfstring{\protect\ref{sec:operator_norm}}{2}}
\label{sec:proof_operator_norm}

\subsection{Proof of Theorem \texorpdfstring{\protect\ref{thmm:T_hat_operator_norm}}{2.1}}
The proof of Theorem~\ref{thmm:T_hat_operator_norm} is further divided
into two stages. In Section~\ref{app:T_hat_operator_norm}, we prove
inequality~(\ref{eq:T_hat_operator_norm_1}); in Section~\ref
{app:T_hat_operator_norm_2}, we prove the data-driven bound,
inequality (\ref{eq:T_hat_operator_norm_2}), and its performance
guarantee, inequality (\ref{eq:T_hat_operator_norm_3}).

\subsubsection{Proof of inequality \texorpdfstring{\textup{(\protect\ref{eq:T_hat_operator_norm_1})}}{(2.1a)}}
\label{app:T_hat_operator_norm}
We wish to apply a Bernstein-type inequality,\vspace*{1pt} specifically \cite{Tropp14},
Theorem~6.6.1, to bound the tail probability $\PP \{ \|
\widehat{T}-T\|_2\ge t  \}$. We note that this theorem on bounding
the tail probability of the maximum eigenvalue of a sum of random
matrices requires that the summands be independent. Clearly, the matrix
$U$-statistic $\wh T-T$ does not satisfy this condition. On the other
hand, this theorem relies on the Chernoff transform technique to
convert the tail probability into an expectation of a convex function
of $\widehat{T}-T$. A technique by Hoeffding \cite{Hoeffding63} then
allows us to convert the problem of bounding $\|\widehat{T}-T\|_2$ into
a problem involving a sum of independent random matrices.

\begin{proposition}
\label{prp:trace_exponential}
We define
%
\begin{equation}
\label{eq:T_tilde} \widetilde{T} = \frac{2}{n}\sum
_{i=1}^{n/2} \widetilde{T}^i
\end{equation}
with
%
\begin{equation}
\label{eq:T_tilde_i} \widetilde{T}^i = \sgn\bigl(X^{2i-1}-X^{2i}
\bigr)\sgn\bigl(X^{2i-1}-X^{2i}\bigr)^T.
\end{equation}
Then the tail probability $\PP \{ \|\widehat{T}-T\|_2\ge t
 \}
$ satisfies
\[
\PP \bigl\{ \|\widehat{T}-T\|_2\ge t \bigr\}  \le\inf_{\theta>0}
\bigl\{ \mathrm{e}^{-\theta t}\cdot\EE \bigl[
\tr \mathrm{e}^{ \theta(\widetilde{T} -T)
} \bigr] \bigr\}+ \inf
_{\theta>0} \bigl\{ \mathrm{e}^{-\theta t}\cdot\EE \bigl[\tr
\mathrm{e}^{ \theta(T-\widetilde{T}) } \bigr] \bigr\}.
\]
\end{proposition}

\begin{pf}
First, note that, because $\wh T -T$ is symmetric, we have
\[
\| \wh T -T\|_2=\max\bigl\{\lambda_{\max}( \wh T -T),-
\lambda_{\min}( \wh T -T)\bigr\} = \max\bigl\{\lambda_{\max}( \wh
T -T),\lambda_{\max}(T- \wh T )\bigr\}.
\]
Hence,
%
\begin{eqnarray}\label{eq:P_T_hat}
\PP \{ \| \wh T -T\|_2\ge t \} &=& \PP \bigl\{ \bigl\{\lambda
_{\max
}( \wh T -T)\ge t\bigr\}\cup\bigl\{\lambda_{\max}(T- \wh T
) \ge t\bigr\} \bigr\}
\nonumber
\\[-8pt]
\\[-8pt]
\nonumber
&\le&\PP \bigl\{ \lambda_{\max}( \wh T -T)\ge t \bigr\} + \PP \bigl\{
\lambda_{\max}(T- \wh T) \ge t \bigr\}.
\end{eqnarray}
Next we bound the first term on the right-hand side of inequality (\ref
{eq:P_T_hat}), that is, $\PP \{ \lambda_{\max}( \wh T -T)\ge t
 \}$. Applying the Chernoff transform technique (e.g., \cite{Tropp14},
Proposition~3.2.1), we have
%
\begin{equation}
\PP \bigl\{ \lambda_{\max}( \wh T -T)\ge t \bigr\}  \le\inf
_{\theta
>0} \bigl\{ \mathrm{e}^{-\theta t} \cdot\EE \bigl[\tr
\mathrm{e}^{\theta( \wh T -T)} \bigr] \bigr\}. \label{eq:P_T_hat_1}
\end{equation}

Now we introduce the technique of Hoeffding. We note the following facts:
\begin{longlist}[1.]
\item[1.]
We can equivalently write $\widehat{T}$ as
%
\begin{equation}
\widehat{T} = \frac{1}{n!} \sum_{n,n} V
\bigl(X^{{i}_1},\ldots,X^{{i}_n}\bigr). \label{eq:T_hat_minus_T_decomposition}
\end{equation}
Here, the function $V$ is defined as
\[
V\bigl(X^{{i}_1},\ldots,X^{{i}_n}\bigr) = \frac{2}{n} \bigl
\{ g\bigl(X^{i_1},X^{{i}_2}\bigr) + g\bigl(X^{{i}_3},X^{{i}_4}
\bigr) + \cdots+ g\bigl(X^{{i}_{n-1}},X^{{i}_{n}}\bigr) \bigr\},
\]
the kernel $g$ is defined as
\[
g\bigl(X^i,X^j\bigr) = \sgn\bigl(X^i-X^j
\bigr) \sgn\bigl(X^i-X^j\bigr)^T,
\]
and the sum $ \sum_{n,n}$ is taken over all permutations
$i_1,i_2,\ldots,i_n$ of the integers $1,2,\ldots,n$.

\item[2.]
The trace exponential function is convex on the set of Hermitian
matrices \cite{Petz94}.
\end{longlist}
Therefore, using first (\ref{eq:T_hat_minus_T_decomposition}) and then
Jensen's inequality, we have
%
\begin{eqnarray}\label{eq:P_T_hat_2}
\tr \mathrm{e}^{\theta( \wh T -T)} &= &\tr\exp \biggl\{\sum_{n,n}
\frac{1}{n!} \theta \bigl[ V\bigl(X^{{i}_1},\ldots,X^{{i}_n}
\bigr) - T \bigr] \biggr\}
\nonumber
\\[-8pt]
\\[-8pt]
\nonumber
&\le&\sum_{n,n} \frac{1}{n!} \tr\exp \bigl\{
\theta \bigl[ V\bigl(X^{{i}_1},\ldots,X^{{i}_n}\bigr) - T \bigr]
\bigr\}.
\end{eqnarray}
Then, plugging inequality (\ref{eq:P_T_hat_2}) into inequality (\ref
{eq:P_T_hat_1}), we have
\begin{eqnarray*}
\PP \bigl\{ \lambda_{\max}( \wh T -T)\ge t \bigr\} &\le&\inf
_{\theta>0} \biggl\{ \mathrm{e}^{-\theta t}\cdot\EE \biggl[ \sum
_{n,n} \frac{1}{n!} \tr \mathrm{e}^{
\theta [ V(X^{{i}_1},\ldots,X^{{i}_n}) - T  ] } \biggr] \biggr
\}
\\
& =& \inf_{\theta>0} \bigl\{ \mathrm{e}^{-\theta t}\cdot\EE \bigl[ \tr
\mathrm{e}^{
\theta [ V(X^{1},X^{2},\ldots,X^{n}) - T  ] } \bigr] \bigr\}
\\
&= &\inf_{\theta>0} \bigl\{ \mathrm{e}^{-\theta t}\cdot\EE \bigl[ \tr
\mathrm{e}^{\theta
(\widetilde{T}-T)} \bigr] \bigr\}.
\end{eqnarray*}

The second term on the right-hand side of inequality (\ref{eq:P_T_hat})
can be similarly bounded. The conclusion of the proposition then follows.
\end{pf}

In Proposition~\ref{prp:trace_exponential}, the argument of the trace
exponential function is proportional to
\[
\widetilde{T} -T = \sum_{i=1}^{n/2}
\frac{2}{n} \bigl( \widetilde {T}^i - T \bigr),
\]
with now independent summands $2n^{-1}(\widetilde{T}^i - T)$, $1\le
i\le n/2$, which are also symmetric. Therefore, we can proceed as in
the proof of \cite{Tropp14}, Theorem~6.6.1, to bound $\EE [ \tr \mathrm{e}^{
\theta(\widetilde{T} -T) } ]$ and $\EE [ \tr \mathrm{e}^{ \theta
(T-\widetilde{T}) } ]$. We calculate the quantities necessary for
applying the proof. First, (for any $i$) we clearly have $\EE
[\widetilde{T}^i-T]=0$. Next, by the representation of $\wt T^i$ as in
(\ref{eq:T_tilde_i}), we conclude that the only nonzero eigenvalue of
$\widetilde{T}^i$ is $d$ which corresponds to the eigenvector $\sgn
(X^{2i-1}-X^{2i})$; thus, $\lambda_{\max}(\widetilde{T}^i)=d$. This,
together with Weyl's inequality and the facts that $T$ is positive
semidefinite and $\|T\|_2\le d\cdot\|T\|_{\infty}\le d$, imply that
\begin{subequations}
%
\begin{eqnarray}
\lambda_{\max}\bigl(\widetilde{T}^i-T\bigr) &\le&
\lambda_{\max}\bigl(\widetilde{T}^i\bigr) = d, \label{eq:lambda_max_wt_T-T}
\\
\lambda_{\max}\bigl(T-\widetilde{T}^i\bigr) &\le&
\lambda_{\max}(T) \le d. \label
{eq:lambda_max_T-wt_T}
\end{eqnarray}
\end{subequations}
Finally, we calculate
\[
\sigma^2 = \Biggl\llVert \sum_{i=1}^{n/2}
\EE \biggl\{ \biggl[ \frac
{2}{n}\bigl(\widetilde{T}^i-T\bigr)
\biggr]^2 \biggr\} \Biggr\rrVert _2,
\]
the matrix variance statistic of the sum as defined in \cite{Tropp14},
Theorem~6.6.1. Note that
\begin{eqnarray*}
\bigl(\widetilde{T}^i \bigr)^2 &=& \sgn
\bigl(X^{2i-1}-X^{2i}\bigr)\sgn \bigl(X^{2i-1}-X^{2i}
\bigr)^T \sgn\bigl(X^{2i-1}-X^{2i}\bigr)\sgn
\bigl(X^{2i-1}-X^{2i}\bigr)^T
\\
&=& \sgn\bigl(X^{2i-1}-X^{2i}\bigr) \bigl[ \sgn
\bigl(X^{2i-1}-X^{2i}\bigr)^T \sgn
\bigl(X^{2i-1}-X^{2i}\bigr) \bigr] \sgn\bigl(X^{2i-1}-X^{2i}
\bigr)^T
\\
&=& d \cdot\sgn\bigl(X^{2i-1}-X^{2i}\bigr) \sgn
\bigl(X^{2i-1}-X^{2i}\bigr)^T \\
&= &d \cdot
\widetilde{T}^i.
\end{eqnarray*}
Then
%
\begin{equation}
\biggl(\frac{n}{2} \biggr)^2 \sigma^2  = \Biggl
\llVert \sum_{i=1}^{n/2} \EE \bigl[d\cdot
\widetilde{T}^i-T^2 \bigr] \Biggr\rrVert _2 =
\frac{n}{2} \bigl\llVert d\cdot T-T^2 \bigr\rrVert
_2 \le\frac{n}{2} d \|T\|_2. \label
{eq:matrix_variance_statistic}
\end{equation}
Hence, by Proposition~\ref{prp:trace_exponential} and the proof of
\cite{Tropp14},
inequality (6.6.3) in Theorem~6.6.1, as well as (\ref
{eq:lambda_max_wt_T-T}), (\ref{eq:lambda_max_T-wt_T}) and (\ref
{eq:matrix_variance_statistic}), we obtain the matrix Bernstein inequality
%
\begin{eqnarray}\label{eq:Tropp_split_Bernstein}
\PP \bigl(\|\widehat{T}-T\|_2 \ge t \bigr) &\le& 2d\cdot\exp \biggl(-
\frac{nt^2}{4d\|T\|_2 + 4\,\mathrm{d}t/3} \biggr)
\nonumber
\\[-8pt]
\\[-8pt]
\nonumber
&\le& 2d\cdot\max \biggl\{ \exp \biggl(-\frac{3}{16}\frac{nt^2}{d\|
T\|
_2}
\biggr), \exp \biggl( -\frac{3}{16}\frac{nt}{d} \biggr) \biggr\}.
\end{eqnarray}
(By Proposition~\ref{prp:trace_exponential} and the proof of \cite{Tropp12a},
Theorem~6.1, we can also obtain the tighter matrix Bennett
inequality.) Finally, setting the right-hand side of inequality (\ref
{eq:Tropp_split_Bernstein}) to $\alpha$ and solving for $t$ yields that
inequality (\ref{eq:T_hat_operator_norm_1}) holds with probability at
least $1-\alpha$. 

\subsubsection{Proof of inequalities \texorpdfstring{\textup{(\protect\ref{eq:T_hat_operator_norm_2})}}{(2.1b)}
and \texorpdfstring{\textup{(\protect\ref{eq:T_hat_operator_norm_3})}}{(2.1c)}}
\label{app:T_hat_operator_norm_2}
We abbreviate $f(n,d,\alpha)$ by $f$, $\|T\|_2$ by $t$, $\|\wh T\|_2$
by $\hat t$, and $\|\widehat{T}-T\|_2$ by $\delta$. We have already
established that we have an event with probability at least $1-\alpha$
on which inequality (\ref{eq:T_hat_operator_norm_1}), that is, $\delta
<\max \{ f \sqrt{t}, f^2 \}$, holds, and we concentrate on
this event.

We proceed to prove inequality (\ref{eq:T_hat_operator_norm_2}), which states
%
\begin{equation}
\max \bigl\{ f \sqrt{t}, f^2 \bigr\} \le\sqrt{ \hat t
f^2 + \bigl(\tfrac
{1}{2}f^2 \bigr)^2
} + \tfrac{1}{2}f^2. \label{eq:bound_t_f}
\end{equation}
Now, if $f \sqrt{t} \le f^2$ and so $\max \{f \sqrt{t} ,
f^2 \}
=f^2$, then inequality (\ref{eq:bound_t_f}) clearly holds. Thus, we
focus on the case $f \sqrt{t} > f^2$. In this case, by
inequality (\ref
{eq:T_hat_operator_norm_1}), we must have
%
\begin{equation}
\delta< f \sqrt{t}. \label{eq:bound_t_f_5}
\end{equation}
By the triangle inequality,
%
\begin{equation}
f\sqrt{t} \le f\sqrt{\delta+\hat t}. \label{eq:bound_t_f_1}
\end{equation}
Then, from inequalities (\ref{eq:bound_t_f_5}) and (\ref
{eq:bound_t_f_1}) we deduce
%
\begin{equation}
f\sqrt{t} < f \sqrt{f\sqrt{t}+\hat t}\label{eq:bound_t_f_2}.
\end{equation}
Squaring both sides of inequality (\ref{eq:bound_t_f_2}) yields $t f^2
< f^3 \sqrt{t} + \hat t f^2$, or equivalently
%
\begin{equation}
\bigl(f\sqrt{t} - \tfrac{1}{2}f^2 \bigr)^2 < \hat
t f^2 + \bigl(\tfrac
{1}{2}f^2
\bigr)^2. \label{eq:bound_t_f_3}
\end{equation}
Because in the current case $f \sqrt{t} >f^2> \frac{1}{2}f^2$,
inequality (\ref{eq:bound_t_f_3}) implies
\[
f \sqrt{t} < \sqrt{ \hat t f^2 + \bigl(\tfrac{1}{2}f^2
\bigr)^2 } + \tfrac{1}{2}f^2,
\]
which, together with $f \sqrt{t} >f^2$, again implies inequality (\ref
{eq:bound_t_f}). Hence, we have proved inequality (\ref
{eq:T_hat_operator_norm_2}).

Next, we prove inequality (\ref{eq:T_hat_operator_norm_3}). By the
triangle inequality,
%
\begin{equation}
\sqrt{\hat t f^2 + \bigl(\tfrac{1}{2}f^2
\bigr)^2 } + \tfrac{1}{2}f^2 \le\sqrt{t
f^2 + \delta f^2+ \bigl(\tfrac{1}{2}f^2
\bigr)^2 } + \tfrac
{1}{2}f^2. \label{eq:bound_t_h}
\end{equation}
First, assume that $\delta< f \sqrt{t}$. Then, from inequality (\ref
{eq:bound_t_h}) we deduce
%
\begin{eqnarray}\label{eq:bound_t_h_1}
\sqrt{ \hat t f^2 + \bigl(\tfrac{1}{2}f^2
\bigr)^2 } + \tfrac{1}{2}f^2 & <& \sqrt{t
f^2 + f^3 \sqrt{t}+ \bigl(\tfrac{1}{2}f^2
\bigr)^2 } + \tfrac{1}{2}f^2
\nonumber
\\[-8pt]
\\[-8pt]
\nonumber
&=& \bigl(f \sqrt{t} +\tfrac{1}{2}f^2 \bigr)+
\tfrac{1}{2}f^2.
\end{eqnarray}
Next, suppose instead $\delta\ge f \sqrt{t}$, so by inequality (\ref
{eq:T_hat_operator_norm_1}) we must have $f \sqrt{t} \le\delta<
f^2$. Then, from inequality (\ref{eq:bound_t_h}) we deduce
%
\begin{equation}
\sqrt{\hat t f^2 + \bigl(\tfrac{1}{2}f^2
\bigr)^2 } + \tfrac
{1}{2}f^2  <
\sqrt{f^4 + f^4+ \bigl(\tfrac{1}{2}f^2
\bigr)^2 } + \tfrac
{1}{2}f^2 =
\tfrac{3}{2}f^2 + \tfrac{1}{2}f^2.
\label{eq:bound_t_h_2}
\end{equation}
Both inequalities (\ref{eq:bound_t_h_1}) and (\ref{eq:bound_t_h_2})
further imply that
\[
\sqrt{\hat t f^2 + \bigl(\tfrac{1}{2}f^2
\bigr)^2 } + \tfrac
{1}{2}f^2 < \max \bigl\{ f
\sqrt{t} ,f^2 \bigr\} + f^2,
\]
which is just inequality (\ref{eq:T_hat_operator_norm_3}). 

\subsection{Proof of Theorem \texorpdfstring{\protect\ref{thmm:main_operator_norm_bound}}{2.2}}
\label{app:main_operator_norm_bound}

The proof of Theorem~\ref{thmm:main_operator_norm_bound} will be
established through the following three lemmas. Recall that we use
$\circ$ to denote the Hadamard product.

\begin{lemma}
\label{lemma:decompose_first_second_order}
We have
\[
\bigl\| \wh\Sigma' - \Sigma\bigr\|_2 \le\frac{\pi}{2}\cdot
\biggl\llVert \cos \biggl(\frac{\pi}{2}T \biggr) \circ\bigl(
\widehat{T}'-T\bigr) \biggr\rrVert _2 +
\frac
{\pi
^2}{8}\cdot\biggl\llVert \sin \biggl(\frac{\pi}{2}\overline{T}
\biggr)\circ \bigl(\widehat{T}'-T\bigr)\circ\bigl(
\widehat{T}'-T\bigr)\biggr\rrVert _2.
\]
Here, $\overline{T}$ is a symmetric, random matrix such that each entry
$[\overline{T}]_{k\ell}$ is a random number on the closed interval
between $[T]_{k\ell}$ and $[\widehat{T}']_{k\ell}$.
\end{lemma}

\begin{pf}
By Taylor's theorem, we have
%
\begin{eqnarray}\label{eq:T_hat_expansion}
\wh\Sigma' - \Sigma&=& \sin \biggl(\frac{\pi}{2}\widehat
{T}' \biggr) - \sin \biggl(\frac{\pi}{2}T \biggr)
\nonumber
\\[-8pt]
\\[-8pt]
\nonumber
&=& \cos \biggl(\frac{\pi}{2}T \biggr) \circ\frac{\pi
}{2}\bigl(\widehat
{T}'-T\bigr) - \frac{1}{2}\sin \biggl(\frac{\pi}{2}
\overline{T} \biggr)\circ \frac{\pi}{2}\bigl(\widehat{T}'-T
\bigr)\circ\frac{\pi}{2}\bigl(\widehat{T}'-T\bigr),
\end{eqnarray}
for some matrix $\overline{T}$ as specified in the theorem. Next,
applying the operator norm on both sides of equation (\ref
{eq:T_hat_expansion}) and then using the triangle inequality on the
right-hand side yields the lemma.
\end{pf}

Hence, it suffices to establish appropriate bounds separately for a
first-order term,
$\llVert \cos (\frac{\pi}{2}T ) \circ(\widehat
{T}'-T)\rrVert
_2$, and a second-order term, $\llVert \sin (\frac{\pi
}{2}\overline
{T} )\circ(\widehat{T}'-T)\circ(\widehat{T}'-T)\rrVert _2$.

\begin{lemma}
\label{lemma:first_order}
For the first-order term, we have
\[
\biggl\llVert \cos \biggl(\frac{\pi}{2}T \biggr) \circ\bigl(
\widehat{T}'-T\bigr) \biggr\rrVert _2 \le2\bigl\|
\widehat{T}'-T \bigr\|_2.
\]
\end{lemma}
\begin{pf}
Recall that $\sin (\frac{\pi}{2}T )=\Sigma$. Then, with $J_d$
denoting a $d\times d$ matrix with all entries identically equal to
one, and the square root function acting component-wise, we have
%
\begin{equation}
\cos \biggl(\frac{\pi}{2}T \biggr) = \sqrt{J_d- \sin \biggl(
\frac
{\pi
}{2}T \biggr)\circ\sin \biggl(\frac{\pi}{2}T \biggr)} = \sqrt
{J_d-\Sigma \circ\Sigma}. \label{eq:sqrt_cos}
\end{equation}
Next, using the generalized binomial formula
\[
(1 + x)^{\alpha} = \sum_{k=0}^{\infty}\pmatrix{
\alpha\cr k} x^k
\]
on equation (\ref{eq:sqrt_cos}) with $\alpha=\frac{1}{2}$ and $x$ being
the components of $-\Sigma\circ\Sigma$ (so the sum converges, in fact
absolutely, since $\alpha>0$ and $\|\Sigma\circ\Sigma\|_{\infty
}\le
1$), we have
\[
\cos \biggl(\frac{\pi}{2}T \biggr) =\sum_{k=0}^{\infty}
\pmatrix{1/2\cr k} (-1)^k \Sigma\circ_{2k}\Sigma.
\]
Here, by $\Sigma\circ_{l}\Sigma$ we mean the Hadamard product of $l$
$\Sigma$'s, that is, $\Sigma\circ\cdots\circ\Sigma$ with a total
of $l$ terms.
Hence,
%
\begin{eqnarray}\label{eq:T_hat_first_order_1}
\biggl\llVert \cos \biggl(\frac{\pi}{2}T \biggr) \circ\bigl(
\widehat{T}'-T\bigr) \biggr\rrVert _2 &=& \Biggl\llVert
\Biggl[\sum_{k=0}^{\infty} \pmatrix{1/2\cr k}
(-1)^k \Sigma \circ_{2k}\Sigma \Biggr] \circ\bigl(
\widehat{T}'-T\bigr) \Biggr\rrVert _2
\nonumber
\\[-8pt]
\\[-8pt]
\nonumber
&\le&\sum_{k=0}^{\infty} \left\llvert \pmatrix{1/2\cr k}
\right\rrvert \cdot\bigl\llVert (\Sigma\circ_{2k}\Sigma) \circ\bigl(
\widehat{T}'-T\bigr) \bigr\rrVert _2.
\end{eqnarray}
Because $\Sigma$ is positive semidefinite (since it is a correlation
matrix), by the Schur product theorem, $\Sigma\circ_{2k}\Sigma$ is
positive semidefinite for all $k$; moreover, $\Sigma\circ_{2k}\Sigma$'s
all have diagonal elements identically equal to one. Then, by \cite{Horn91},
Theorem~5.5.18, we have, for all $k$,
%
\begin{equation}
\bigl\llVert (\Sigma\circ_{2k}\Sigma) \circ\bigl(
\widehat{T}'-T\bigr) \bigr\rrVert _2 \le \bigl\|
\widehat{T}'-T \bigr\|_2. \label{eq:T_hat_first_order_2}
\end{equation}
Plugging (\ref{eq:T_hat_first_order_2}) into (\ref
{eq:T_hat_first_order_1}) and then using the fact that $\sum_{k=0}^{\infty}
\bigl\llvert {1/2\choose k}\bigl\rrvert  = 2$ yield
%
\begin{equation}
\biggl\llVert \cos \biggl(\frac{\pi}{2}T \biggr) \circ\bigl(
\widehat{T}'-T\bigr) \biggr\rrVert _2 \le \Biggl[ \sum
_{k=0}^{\infty} \left\llvert \pmatrix{1/2\cr k}\right\rrvert
\Biggr] \cdot\bigl\| \widehat{T}'-T \bigr\|_2 = 2 \bigl\|
\widehat{T}'-T \bigr\|_2, \label{eq:T_hat_first_order_3}
\end{equation}
which is the conclusion of the lemma.
\end{pf}

\begin{lemma}
\label{lemma:second_order}
For the second-order term, we have
%
\begin{equation}
\biggl\llVert \sin \biggl(\frac{\pi}{2}\overline{T} \biggr)\circ \bigl(
\widehat {T}'-T\bigr)\circ\bigl(\widehat{T}'-T\bigr)
\biggr\rrVert _2 \le\bigl\|\widehat{T}'-T
\bigr\|_2^2. \label{eq:second_order_raw_generic}
\end{equation}
Alternatively, for the particular case $\wh T'=\wh T$, we have, with
probability at least $1-\frac{1}{4}\alpha^2$,
%
\begin{equation}
\biggl\llVert \sin \biggl(\frac{\pi}{2}\overline{T} \biggr)\circ (\widehat
{T}-T)\circ(\widehat{T}-T)\biggr\rrVert _2 \le8\cdot
\frac{d\cdot\log
(2\alpha^{-1}d)}{n}. \label{eq:second_order_raw}
\end{equation}
\end{lemma}

\begin{pf}
First, we observe a simple fact: for two matrices $M,N\in\RR^{k\times
\ell}$ (for arbitrary $k,\ell$), if $|[M]_{ij}|\le[N]_{ij}$ for all
$1\le i\le k,1\le j\le\ell$, then $\|M\|_2\le\|N\|_2$.

To see this, we fix an arbitrary vector $u=(u_1,\ldots,u_\ell)^T\in
\mathbb{R}^{\ell}$ with $\|u\|=1$, with $\|\cdot\|$ being the Euclidean
norm for vectors. Let $\tilde{u}=(\tilde{u}_1,\ldots,\tilde{u}_{\ell
})^T\in\mathbb{R}^{\ell}$ be the vector such that $\tilde{u}_j = |u_j|$
for $j=1,\ldots,\ell$, that is, each component of $\tilde{u}$ is the
absolute value of the corresponding component of $u$. Clearly, $\|
\tilde
{u}\|=1$ as well. Then we have, for all $1\le i\le k$,
\[
\bigl\llvert [M u]_i\bigr\rrvert = \Biggl\llvert \sum
_{j=1}^\ell[M]_{ij}u_j \Biggr
\rrvert \le\sum_{j=1}^\ell\bigl|[M]_{ij}\bigr||u_j|
\le\sum_{j=1}^\ell[N]_{ij}
\tilde {u}_j = \bigl|[N \tilde{u}]_i\bigr|.
\]
Here, $[M u]_i$ and $[N \tilde{u}]_i$ are the $i$th component of the
vectors $M u$ and $N \tilde{u}$, respectively. Hence, clearly, $\|M u\|
\le\|N \tilde{u}\|$, which further implies that
\[
\sup \bigl\{ \llVert M u\rrVert \dvt \|u\|=1 \bigr\} \le\sup \bigl\{ \llVert N u
\rrVert \dvt \|u\|=1 \bigr\},
\]
and we conclude that $\|M\|_2\le\|N\|_2$.

Now, it is easy to see that
\[
\biggl\llvert \biggl[\sin \biggl(\frac{\pi}{2}\overline{T} \biggr)\circ
\bigl(\widehat {T}'-T\bigr)\circ\bigl(\widehat{T}'-T
\bigr) \biggr]_{ij}\biggr\rrvert \le \bigl[\bigl(\widehat
{T}'-T\bigr)\circ\bigl(\widehat{T}'-T\bigr)
\bigr]_{ij}\qquad \forall1\le i,j\le d.
\]
Hence, by the preceding observation, we have
%
\begin{equation}
\biggl\llVert \sin \biggl(\frac{\pi}{2}\overline{T} \biggr)\circ \bigl(
\widehat {T}'-T\bigr)\circ\bigl(\widehat{T}'-T\bigr)
\biggr\rrVert _2 \le\bigl\llVert \bigl(\widehat {T}'-T
\bigr)\circ \bigl(\widehat{T}'-T\bigr)\bigr\rrVert _2.
\label{eq:second_order_1}
\end{equation}
By \cite{Horn91}, Theorem~5.5.1, we further have
%
\begin{equation}
\bigl\llVert \bigl(\widehat{T}'-T\bigr)\circ\bigl(
\widehat{T}'-T\bigr)\bigr\rrVert _2 \le\bigl\|\wh
T'-T\bigr\|_2^2. \label{eq:second_order_2}
\end{equation}
Then inequality (\ref{eq:second_order_raw_generic}) follows from
inequalities (\ref{eq:second_order_1}) and (\ref{eq:second_order_2}).

Next, we prove the second half of the lemma. We have
%
\begin{equation}
\biggl\llVert \sin \biggl(\frac{\pi}{2}\overline{T} \biggr)\circ(\wh T-T)
\circ (\wh T-T)\biggr\rrVert _2 \le\bigl\llVert (\wh T-T)\circ(\wh
T-T)\bigr\rrVert _2 \le d \|\wh T-T\|_{\infty}^2.
\label{eq:second_order_T_hat_1}
\end{equation}
Here, the first inequality follows by inequality (\ref
{eq:second_order_1}) with the choice $\wh T'=\wh T$, and the second
inequality follows by the bound that $\|M\circ M\|_2 \le d\|M\circ M\|
_{\infty} = d \|M\|_{\infty}^2$ for arbitrary $M\in\RR^{d\times
d}$. By
Hoeffding's inequality for the scalar $U$-statistic \cite{Hoeffding63},
\[
\PP \bigl( | \widehat{T}_{jk}-T_{jk} | \ge t \bigr) \le2 \exp
\biggl(-\frac{nt^2}{4} \biggr),
\]
and so, by the union bound,
\[
\PP \bigl( \| \widehat{T} - T \|_{\infty} \ge t \bigr) \le d^2 \exp
\biggl(-\frac{nt^2}{4} \biggr).
\]
Thus, there exists an event $A$ with probability at least $1-\frac
{1}{4}\alpha^2$ such that
%
\begin{equation}
\| \widehat{T} - T \|_{\infty}^2 \le4\cdot\frac{\log(4\alpha
^{-2}d^2)}{n} =
8\cdot\frac{\log(2\alpha^{-1}d)}{n} \label{eq:second_order_T_hat_2}
\end{equation}
on the event $A$. Plugging inequality (\ref{eq:second_order_T_hat_2})
into inequality (\ref{eq:second_order_T_hat_1}) yields that
inequality (\ref{eq:second_order_raw}) holds on the same event. This
finishes the proof of the lemma.
\end{pf}

The conclusions of Theorem~\ref{thmm:main_operator_norm_bound} now
follow immediately. In particular, inequality (\ref{Lipschitz_generic})
follows from Lemmas~\ref{lemma:decompose_first_second_order},
\ref
{lemma:first_order} and inequality (\ref{eq:second_order_raw_generic})
in Lemma~\ref{lemma:second_order}, while inequality (\ref
{Lipschitz_Kendall}) follows from Lemmas~\ref
{lemma:decompose_first_second_order} and~\ref{lemma:first_order}
with $\wh T'$ set to $\wh T$ and $\wh\Sigma'$ set to $\wh\Sigma$, and
inequality (\ref{eq:second_order_raw}) in Lemma~\ref
{lemma:second_order}, which holds with probability at least $1-\frac
{1}{4}\alpha^2$. 

\subsection{Proof of Theorem \texorpdfstring{\protect\ref{thmm:T_Sigma}}{2.3}}
\label{app:T_Sigma}

We let the $\arcsin$ function have the series expansion $\arcsin(x) =
\sum_{k=0}^{\infty}g(k) x^k$ for $|x|\le1$. The exact form of the
$g(k)$'s for all $k$ is not important; we only need $g(0)=0$, $g(1)=1$,
all the $g(k)$'s are nonnegative, and $\sum_{k=0}^{\infty}g(k) = \pi
/2$. With the $\arcsin$ function acting component-wise, and with
$\Sigma
\circ_{k}\Sigma$ denoting the Hadamard product of $k$ $\Sigma$'s, we have
\[
T = \frac{2}{\pi}\arcsin(\Sigma) = \frac{2}{\pi}\sum
_{k=0}^{\infty} g(k) \Sigma\circ_{k}\Sigma.
\]
Because $\Sigma$ is positive semidefinite, by the Schur product
theorem, $\Sigma\circ_{k}\Sigma$, and thus $g(k) \Sigma\circ
_{k}\Sigma
$, are positive semidefinite for all $k\ge0$. In addition, $T$ is
positive semidefinite. Hence, by Weyl's inequality and the triangle inequality,
%
\begin{equation}
\frac{2}{\pi} g(1)\|\Sigma\|_2 \le\|T\|_2 \le
\frac{2}{\pi} \sum_{k=0}^{\infty} g(k)\|
\Sigma\circ_{k}\Sigma\|_2. \label{eq:T_Sigma_1}
\end{equation}
The first half of inequality (\ref{eq:T_Sigma_1}) yields the first half
of inequality (\ref{eq:T_Sigma}). Next, note that the $\Sigma\circ
_{k}\Sigma$'s, in addition to being positive semidefinite, all have
diagonal elements identically equal to one. Then, by \cite{Horn91},
Theorem~5.5.18, we have for all $k\ge2$, $\|\Sigma\circ
_{k}\Sigma\|_2 = \|(\Sigma\circ_{k-1}\Sigma)\circ\Sigma\|_2 \le\|
\Sigma
\|_2$. Therefore, the second half of inequality (\ref{eq:T_Sigma_1}) yields
\[
\|T\|_2 \le\frac{2}{\pi} \sum_{k=0}^{\infty}
g(k)\|\Sigma\|_2 = \| \Sigma\|_2,
\]
which is the second half of inequality (\ref{eq:T_Sigma}). \qed

\subsection{Proof of Theorem \texorpdfstring{\protect\ref{thmm:Sigma_psd}}{2.4}}
Because $\Sigma$ belongs to the feasible region $\calF$, and $\wh
\Sigma
^{\mathrm{generic}+}$ minimizes $\|\Sigma' - \wh\Sigma^{\mathrm
{generic}}\|$ over $\Sigma'\in\calF$ by (\ref{eq:Sigma_psd}), we
conclude that
%
\begin{equation}
\bigl\|\wh\Sigma^{\mathrm{generic}+} - \wh\Sigma^{\mathrm{generic}}\bigr\| \le\bigl\| \Sigma- \wh
\Sigma^{\mathrm{generic}}\bigr\|. \label{eq:Sigma_psd_1}
\end{equation}
Then, plugging inequality (\ref{eq:Sigma_psd_1}) into the triangle inequality
\[
\bigl\|\wh\Sigma^{\mathrm{generic}+}-\Sigma\bigr\| \le\bigl\|\wh\Sigma^{\mathrm
{generic}+}-\wh
\Sigma^{\mathrm{generic}}\bigr\| +\bigl \|\wh\Sigma^{\mathrm
{generic}}-\Sigma\bigr\|
\]
yields the conclusion of the theorem. 

\subsection{Proof of Corollary \texorpdfstring{\protect\ref{cor:Sigma_psd}}{2.5}}
First, with the choice $\calF=\calS^d_+$, we clearly have $\Sigma\in
\calF$. Then inequality (\ref{eq:Sigma_hat_p_operator_norm}) follows
straightforwardly from Theorem~\ref{thmm:Sigma_psd}. Next, we consider
the choice of $\calF$ as in (\ref{eq:calF_Sigma_hat_infinity}). With
argument similar to that used in the proof of Lemma~\ref
{lemma:second_order}, we conclude that there exists an event $A$ with
probability at least $1-\frac{1}{4}\alpha^2$ such that $\wh T$ satisfies
%
\begin{equation}
\|\wh T-T\|_{\infty}\le\sqrt{\tfrac{3}{2}} d^{-1/2} f(n,d,
\alpha) \label{eq:T_hat_infty}
\end{equation}
on the event $A$. For the rest of the proof, we concentrate on the
event $A$. By (\ref{R_tau}), (\ref{R_hattau}), (\ref{eq:T_hat_infty})
and the Lipschitz property of the sine function, we have
%
\begin{equation}
\|\wh\Sigma-\Sigma\|_{\infty}\le\frac{\pi}{2}\sqrt{\frac{3}{2}}
d^{-1/2} f(n,d,\alpha) = C_3 d^{-1/2} f(n,d,\alpha),
\label{eq:Sigma_hat_infty}
\end{equation}
which further implies that $\Sigma\in\calF$. Then inequality (\ref
{eq:Sigma_hat_p_operator_norm}) again follows from Theorem~\ref
{thmm:Sigma_psd}. Finally, inequality (\ref{eq:Sigma_hat_p_infinity})
follows because $\|\wh\Sigma^+-\wh\Sigma\|_{\infty}\le C_3 d^{-1/2}
f(n,d,\alpha)$ by the choice (\ref{eq:calF_Sigma_hat_infinity}) of
$\calF$, inequality (\ref{eq:Sigma_hat_infty}), and the triangle
inequality $\|\wh\Sigma^+-\Sigma\|_{\infty}\le\|\wh\Sigma^+-\wh
\Sigma\|
_{\infty} + \|\wh\Sigma-\Sigma\|_{\infty}$. 


\section{Proof of Theorem \texorpdfstring{\protect\ref{thmm:elementary}}{3.1}}
\label{sec:proof_elementary_model}

We first establish a proposition, which serves as the main ingredient
for the proof of Theorem~\ref{thmm:elementary}. For brevity of
presentation, we denote
\[
E=\wh\Sigma-\Sigma.
\]

\begin{proposition}\label{prop:elementary}
Assume that $\Theta^*$ satisfies $0<r<d$ and $\lambda_r(\Theta^*)
\ge2
{\mu}$. On the event $ \{ 2\| E \|_2 < \mu \}$, we have
%
\begin{eqnarray}
\label{eq:r_hat_equal_r_prop} \wh r & = & r,
\\
\label{T} \bigl\|\widehat\Theta- \Theta^*\bigr\|_F ^2 &\le&
8 r \| E \|_2^2,
\\
\label{S}\bigl | \widehat\sigma^2 - \sigma^2 \bigr| &\le& \|
E \|_2.
\end{eqnarray}
\end{proposition}

\begin{pf}
Let $\lambda_1(M)\ge\cdots\ge\lambda_d(M)$ be the ordered
eigenvalues of a generic symmetric matrix $M\in\RR^{d\times d}$. Note that
%
\begin{eqnarray}
\widehat r &>& r \quad\Longleftrightarrow\quad\widehat\lambda_{r+1} -\widehat
\lambda_d \ge\mu, \label{eq:elementary_r_hat_les_r}
\\
\widehat r &<& r \quad\Longleftrightarrow\quad\widehat\lambda_{r} -\widehat
\lambda_d < \mu. \label{eq:elementary_r_hat_gtr_r}
\end{eqnarray}
We obtain, using Weyl's inequality,
%
\begin{eqnarray}\label{eq:elementary_l_rp1_d}
\widehat\lambda_{r+1} -\widehat\lambda_d &=&
\lambda_{r+1}(\Sigma+E) - \lambda_d(\Sigma+E) \le
\lambda_{r+1} (\Sigma) +2 \| E\|_2 - \lambda _d(
\Sigma)
= 2 \|E\|_2,
\\
\label{eq:elementary_l_r_d}
\widehat\lambda_{r} -\widehat\lambda_d &=&
\lambda_{r}(\Sigma+E) - \lambda_d(\Sigma+E) \ge
\lambda_{r} (\Sigma) - 2\| E\|_2 - \lambda _d(
\Sigma)
\nonumber
\\[-8pt]
\\[-8pt]
\nonumber
&=& \lambda_r\bigl(\Theta^*\bigr) - 2 \|E\|_2.
\end{eqnarray}
Together, (\ref{eq:elementary_r_hat_les_r}), (\ref
{eq:elementary_r_hat_gtr_r}), (\ref{eq:elementary_l_rp1_d}), (\ref
{eq:elementary_l_r_d}) and the condition $\lambda_r(\Theta^*)\ge2\mu$
lead to
%
\begin{equation}
\{ \widehat r\ne r\} \subseteq \bigl\{ 2 \| E \|_2 \ge\min \bigl(\mu
, \lambda_{r}\bigl(\Theta^*\bigr) - \mu \bigr) \bigr\} \subseteq \bigl\{ 2
\| E \|_2 \ge\mu \bigr\}.
\end{equation}
A similar reasoning is used in the proof of \cite{BSW2011}, Theorem~2.
Consequently, equation (\ref{eq:r_hat_equal_r_prop}), that is, $\wh
r=r$, holds on the event $\{ 2\| E\|_2 < \mu\}$, and for the rest of
the proof we concentrate on this event. Then we have
%
\begin{eqnarray}\label{eq:elementary}
\bigl\|\widehat\Theta- \Theta^*\bigr\|_F &\le&\sqrt{2r} \bigl\| \widehat\Theta-
\Theta^*\bigr\|_2 = \sqrt{2r} \Biggl\llVert \sum
_{k=1}^r \bigl( \widehat\lambda_k -
\wh\sigma^2 \bigr) \widehat u_k \widehat
u_k^T -\Theta^* \Biggr\rrVert _2
\nonumber
\\
&= &\sqrt{2r} \Biggl\llVert \sum_{k=1}^d
\widehat\lambda_k \widehat u_k \widehat
u_k^T - \sum_{k=r+1}^d
\widehat\lambda_k \widehat u_k \widehat
u_k^T -\sum_{k=1}^r
\wh\sigma^2 \widehat u_k \widehat u_k^T
-\Theta^* \Biggr\rrVert _2
\nonumber
\\
&=& \sqrt{2r} \Biggl\llVert \wh\Sigma-\Sigma+ \sigma^2
I_d -\sum_{k=1}^r \wh
\sigma^2 \widehat u_k \widehat u_k^T
- \sum_{k=r+1}^d \widehat
\lambda_k \widehat u_k \widehat u_k^T
\Biggr\rrVert _2
\\
&=& \sqrt{2r} \Biggl\llVert E + \sum_{k=1}^d
\bigl(\sigma^2 - \widetilde \lambda _k\bigr) \widehat
u_k \widehat u_k^T \Biggr\rrVert
_2
\nonumber
\\
&\le&\sqrt{2r} \Bigl[ \| E\| _2 + \max_{1\le k\le d} \bigl|
\widetilde \lambda _k - \sigma^2\bigr| \Bigr]. \nonumber
\end{eqnarray}
Here, we have denoted
\[
\widetilde\lambda_k = %
\cases{\wh\sigma^2, &\quad
$\mbox{if } k\le r$, \vspace*{2pt}
\cr
\wh\lambda_k, &\quad $\mbox{if } k
\ge r+1 $.} %
\]
We use Weyl's inequality again to observe that
%
\begin{eqnarray}\label{eq:elementary_2}
\max_{1\le k\le d}\bigl|\widetilde\lambda_k -
\sigma^2\bigr| &=& \max\bigl(\bigl |\lambda _{r+1}(\wh\Sigma)-
\sigma^2\bigr|, \ldots, \bigl|\lambda_{d}(\wh\Sigma )-
\sigma^2\bigr|, \bigl| \wh\sigma^2 -\sigma^2\bigr| \bigr)
\nonumber
\\
&=& \max \bigl( \bigl|\lambda_{r+1}(\wh\Sigma)-\lambda_{r+1}(
\Sigma)\bigr|, \ldots , \bigl|\lambda_{d}(\wh\Sigma)-\lambda_{d}(
\Sigma)\bigr| \bigr)
\\
&\le& \| E\|_2, \nonumber
\end{eqnarray}
which implies inequality (\ref{S}). Finally, inequalities (\ref
{eq:elementary}) and (\ref{eq:elementary_2}) together imply
inequality~(\ref{T}).
\end{pf}

Note that the regularization parameter $\mu$ should both be large
enough such that the event $\{2 \| E\|_2< \mu\}$ has high probability,
and be small enough such that the condition $\lambda_r(\Theta^*) \ge
2{\mu}$ is not too stringent. However, these requirements cannot always
be met at the same time, as we demonstrate next. For brevity, we set
$f=f(n,d, \alpha)$.

First, on the one hand, it is clear from Theorem~\ref
{thmm:main_operator_norm_bound} that we should choose, for some absolute
constants $c_1$, $c_2$ and $\alpha<1/2$,
%
\begin{equation}
\label{mukeus} \mu\approx c_1 \sqrt{ \| T\|_2 } f+
c_2 f^2,
\end{equation}
to guarantee that the event $\{2\| E \|_2 < \mu\}$ has probability
larger than $1-2\alpha$. (In practice, we need a procedure that
determines $\mu$ based on $\|\wh T\|_2$ instead of $\|T\|_2$, and at
the same time guarantees the convergence rates in (\ref{T}) and (\ref
{S}) in terms of $\|T\|_2$. Theorem~\ref{thmm:elementary} describes such
a procedure in detail, using the results from Theorem~\ref
{thmm:main_operator_norm_bound}.) On the other hand, by Theorem~\ref
{thmm:T_Sigma} and the condition $\lambda_r(\Theta^*) \ge2 \mu$, the
following string of inequalities
%
\begin{equation}
\label{oeps} \frac{\pi}{2} \|T\|_2 \ge\|\Sigma\|_2
\ge\lambda_{\max}\bigl(\Theta ^*\bigr)\ge \lambda_r\bigl(
\Theta^*\bigr)\ge2 \mu
\end{equation}
hold. Now, if $\| T\|_2 \ll f^2$, then $\mu\ll f^2$ as well by (\ref
{oeps}), contradicting (\ref{mukeus}). Therefore, the interesting case
is (roughly) when inequality (\ref{eq:elementary_n_condition}) holds.

\begin{pf*}{Proof of Theorem~\ref{thmm:elementary}}
Let
%
\begin{equation}
\bar\mu' = 2 \bigl\{ C_1 \max \bigl[ \sqrt{ \| T
\|_2} f(n,d,\alpha ) , f^2(n,d,\alpha) \bigr] +
(C_1 + C_2 ) f^2(n,d,\alpha) \bigr\}.
\label{eq:mu_bar_p}
\end{equation}
Then Theorem~\ref{thmm:main_operator_norm_bound} guarantees that $\PP
\{
2\| E\|_2 < \mu< \bar\mu' \}\ge1-\alpha-\alpha^2/4 > 1-2\alpha$ with
the choices (\ref{MU}) and (\ref{eq:mu_bar_p}) of $\mu$ and $\bar
\mu
'$, and for the rest of the proof we concentrate on this event. Assume
that $\Theta^*$ satisfies $0<r<d$ and $\lambda_r(\Theta^*)\ge2\bar
\mu
$, and $n$ is large enough such that condition (\ref
{eq:elementary_n_condition}), which is in place for the reasons
discussed in the remarks following Proposition~\ref{prop:elementary},
holds. Because condition (\ref{eq:elementary_n_condition}) also ensures
that equation (\ref{eq:T_hat_bound_first_term_dominate}) holds, we have
$\bar\mu'=\bar\mu$. Hence, the assumption $\lambda_r(\Theta^*)\ge
2\bar
\mu$ further implies that $\lambda_r(\Theta^*)\ge2\bar\mu' > 2\mu$.
Then Proposition~\ref{prop:elementary} states that equation (\ref
{eq:r_hat_equal_r}) and inequalities (\ref{T}), (\ref{S}) hold. Next,
we can replace $\|E\|_2$ in inequalities (\ref{T}) and (\ref{S}) by
$\bar\mu'/2$ using the bound $\| E\|_2<\bar\mu'/2$, and further replace
$\bar\mu'$ by $\bar\mu$. Inequality (\ref{GRENS1}) and the second half
of inequality (\ref{GRENS}) then follow. The first half of
inequality (\ref{GRENS}) follows because by (\ref{eq:wt_Sigma_e}), we have
\[
\bigl\|\wt\Sigma^e-\Sigma\bigr\|_F^2 = \bigl\|\wh
\Theta_o-\Theta^*_o\bigr\|_F^2 \le\bigl\|
\wh \Theta-\Theta^*\bigr\|_F^2.
\]
It remains to establish the last statement of the theorem. We let
$\operatorname
{diag}(\Theta^*)$ be the common value of the diagonal elements of
$\Theta^*$. We assume that $\operatorname{diag}(\Theta^*)\le1-\sqrt{2r\bar
\mu
^2}$ as in the statement of the theorem, and show that $\wt\Sigma^e$ is
positive semidefinite. Inequality (\ref{GRENS}) implies that $\|\wh
\Theta-\Theta^*\|_{\infty}\le\sqrt{2r\bar\mu^2}$. Thus, the
values of
the diagonal elements of $\wh\Theta$ cannot exceed $\operatorname
{diag}(\Theta
^*)+\sqrt{2r\bar\mu^2}\le1$. Hence, in this case, by (\ref
{eq:wt_Sigma_e}), $\wt\Sigma^e$ is obtained by adding to $\wh\Theta
$ a
diagonal matrix with nonnegative diagonal entries. Because $\wh\Theta$
is positive semidefinite by construction, we conclude that $\wt\Sigma
^e$ is positive semidefinite as well.
\end{pf*}


\section{Proof of Theorem \texorpdfstring{\protect\ref{thmm:recovery_bound_optimize}}{3.2}}
\label{sec:proof_refined_estimator}

\subsection{Preliminaries}

We let $M\in\mathbb{R}^{d\times d}$ be an arbitrary matrix of rank $r$,
with the (reduced) singular value decomposition $M=U\Lambda V^T$. Here,
$U,V\in\RR^{d\times r}$ are, respectively, matrix of the left and right
orthonormal singular vectors of $M$ corresponding to the nonzero
singular values that are the diagonal elements of $\Lambda\in\RR
^{r\times r}$. Following the exposition in \cite{Chandrasekaran11}, the
tangent space $T(M)\subset\RR^{d\times d}$ at $M$ with respect to the
algebraic variety of matrices with rank at most $r=\operatorname
{rank}(M)$, or the tangent space $T(M)$ for short, is given by
\[
T(M) = \bigl\{ U X^T + Y V^T | X, Y\in
\mathbb{R}^{d\times r} \bigr\}.
\]
We denote the orthogonal complement of $T(M)$ by $T(M)^{\perp}$. In
addition, we denote the projector onto the tangent space $T(M)$ by
$\calP_{T(M)}$, and the projector onto $T(M)^{\perp}$ by $\calP
_{T(M)^{\perp}}$. Then, for an arbitrary matrix $N\in\mathbb
{R}^{d\times d}$, the explicit forms of $\calP_{T(M)}$ and $\calP
_{T(M)^{\perp}}$ are given by
\begin{eqnarray*}
\calP_{T(M)}(N) &=& U U^T N + N V V^T - U
U^T N V V^T,
\\
\calP_{T(M)^{\perp}}(N)& =& \bigl(I_d - U U^T\bigr) N
\bigl(I_d- V V^T\bigr),
\end{eqnarray*}
respectively. One basic fact involving the projectors $\calP_{T(M)}$
and $\calP_{T(M)^{\perp}}$ is
\[
\bigl\|\calP_{T(M)}(N)\bigr\|_2\le2\|N\|_2 \quad\mbox{and}\quad
\bigl\|\calP _{T(M)^{\perp
}}(N)\bigr\|_2\le\|N\|_2.
\]
We denote the set of $d\times d$ diagonal matrices by $\Omega$. We let
the projector onto $\Omega$ be denoted by $\calP_{\Omega}$. Recall that
$\circ$ denotes the Hadamard product. Then, for an arbitrary matrix
$N\in\RR^{d\times d}$, the explicit form of $\calP_{\Omega}$ is
given by
\[
\calP_{\Omega}(N) = I_d\circ N.
\]

We also prove a simple lemma.

\begin{lemma}
\label{lemma:AMB_infty}
Let $A,B,C\in\RR^{d\times d}$ be arbitrary matrices. Then
\[
\| A C B \|_{\infty} \le\sqrt{\bigl\|A A^T\bigr\|_{\infty}
\bigl\|B^T B\bigr\|_{\infty}} \|C\|_2.
\]
\end{lemma}
\begin{pf}
The proof can be found in Appendix \ref{sec:aux_proof_refined_estimator}.
\end{pf}

\subsection{Recovery bound with primal-dual certificate}
\label{sec:rec_bnd_with_certificate}

We let $\bar\Theta, Q \in\RR^{d\times d}$ but otherwise be arbitrary
at this stage. Eventually, we will set $\bar\Theta$ to be some low-rank
approximation to $\Theta^*$, and set $Q$ to be a \textit{primal-dual
certificate} \cite{Zhang12}, or certificate for short, in the sense
defined in equation (\ref{eq:Q_def}) below. For notational brevity, we denote
\[
\T= T(\bar\Theta) \quad\mbox{and}\quad \Tp=T(\bar\Theta)^{\perp}
\]
for the tangent space $T(\bar\Theta)$ and its orthogonal complement
$T(\bar\Theta)^{\perp}$, respectively.

We now state two lemmas toward the general recovery bound for the
refined estimator $\wt\Sigma$ in terms of $\bar\Theta$ and the
(soon-to-be) certificate $Q$.

\begin{lemma}
\label{lemma_equality}
We have
%
\begin{eqnarray}
\label{eq:recoverty_bound_equality} && \tfrac{1}2 \bigl\| \wt\Theta_o -
\Theta_o^* \bigr\|_F^2 + \tfrac{1}2\| \wt
\Theta_o - Q_o\|_F^2 + \bigl
\langle -Q_o + \bar\Theta_o - \Theta_o^* +
\wt\Theta _o , \bar\Theta_o -\wt\Theta_o
\bigr\rangle
\nonumber
\\[-8pt]
\\[-8pt]
\nonumber
&&\quad= \tfrac{1}2 \bigl\| \bar\Theta_o - \Theta_o^*
\bigr\|_F^2 + \tfrac{1}2\| \bar \Theta _o
- Q_o\|_F^2.
\end{eqnarray}
\end{lemma}
\begin{pf}
The identity follows from straightforward algebra, and can also be
obtained from the proof for \cite{Zhang12}, Theorem~3.2.
\end{pf}

We define, for any constant $c\ge1$,
%
\begin{equation}
G_c = \bigl\{\Phi\in\RR^{d\times d} \dvt \Phi\in\mu\partial\|
\bar \Theta\|_* \mbox{ and } \|\calP_{\Tp}\Phi\|_2\le\mu/c
\bigr\}. \label{eq:G_c_definition}
\end{equation}
Here, $\partial\|A\|_*$ denotes the subdifferential with respect to the
nuclear norm at the matrix $A$; we refer to \cite{Watson92} for its
explicit form. Note that $G_c$ is a subset of the subdifferential $\mu
\partial\|\bar\Theta\|_*$, and coincides with the latter when $c=1$.

\begin{lemma}
\label{lemma_inequality}
Assume that
%
\begin{equation}
-Q_o+\bar\Theta_o-\Theta_o^*+\wh
\Sigma_o\in G_c. \label{eq:condition_in_Gc}
\end{equation}
Then
%
\begin{equation}
\bigl\langle -Q_o + \bar\Theta_o -
\Theta_o^* + \wt\Theta_o , \bar \Theta_o -
\wt\Theta_o \bigr\rangle \ge(1-1/c) \mu\| \calP_{\Tp}\wt
\Theta\|_*. \label{eq:Zhang_lemma_6_current_context}
\end{equation}
\end{lemma}

\begin{pf}
We follow the proof of \cite{Zhang12}, Proposition~3.2. Let $\Psi,\Xi
\in
\RR^{d\times d}$ satisfy $\Psi\in\mu\partial\|\wt\Theta\|_*$,
$\Xi_o\in
\mu\partial\|\bar\Theta\|_*$ but otherwise be arbitrary at this stage.
By the definition of subgradient, we have
%
\begin{equation}
\langle\Xi_o,\bar\Theta-\wt\Theta \rangle \ge\mu\|\bar\Theta\| _* -
\mu\| \wt\Theta\|_* \ge \langle\Psi,\bar\Theta-\wt\Theta \rangle. \label{eq:subgrad_Theta_b_h}
\end{equation}
Now we impose on $\Xi$ the stronger condition that $\Xi_o\in G_c$. Then
the first half of inequality (\ref{eq:subgrad_Theta_b_h}) can be
strengthened by \cite{Hsu11}, Lemma~6, to
%
\begin{equation}
\langle \Xi_o, \bar\Theta-\wt\Theta \rangle \ge(1-1/c) \mu\| \calP
_{\Tp}\wt\Theta\|_* + \mu\|\bar\Theta\|_* - \mu\|\wt\Theta\|_*.
\label{eq:Zhang_lemma_6}
\end{equation}
Next, combining inequality (\ref{eq:Zhang_lemma_6}) and the second half
of inequality (\ref{eq:subgrad_Theta_b_h}) yields
%
\begin{equation}
\langle\Xi_o,\bar\Theta-\wt\Theta \rangle \ge \langle\Psi,\bar
\Theta-\wt \Theta \rangle + (1-1/c) \mu\| \calP_{\Tp}\wt\Theta\|_*.
\label{eq:subgrad_Theta_b_h_s}
\end{equation}
Let $L(\Theta)=\frac{1}{2}\|\Theta_o-\wh\Sigma_o\|_F^2$ denote the loss
function in the convex program (\ref{eq:convex_program}) and $\nabla
L(\Theta)=\Theta_o-\wh\Sigma_o$ denote its gradient. Then, adding
$
\langle\nabla L(\wt\Theta),\bar\Theta-\wt\Theta \rangle$ to both sides of
inequality (\ref{eq:subgrad_Theta_b_h_s}) yields
%
\begin{equation}
\bigl\langle\Xi_o+\nabla L(\wt\Theta),\bar\Theta-\wt\Theta \bigr
\rangle \ge \bigl\langle\Psi+\nabla L(\wt\Theta),\bar\Theta-\wt\Theta \bigr\rangle
+ (1-1/c) \mu\| \calP_{\Tp}\wt\Theta\|_*. \label{eq:subgrad_Theta_b_h_s1}
\end{equation}

We now fix our choices of $\Psi$ and $\Xi$. First, by the optimality of
$\wt\Theta$ for the convex program~(\ref{eq:convex_program}), we have
$0\in\nabla L(\wt\Theta)+\mu\partial\|\wt\Theta\|_*$. Hence, we can
fix $\Psi\in\mu\partial\|\wt\Theta\|_*$ such that
%
\begin{equation}
\label{eq:Theta_hat_optimality} \nabla L(\wt\Theta)+\Psi= 0.
\end{equation}
Then, plugging equation (\ref{eq:Theta_hat_optimality}) into
inequality (\ref{eq:subgrad_Theta_b_h_s1}) yields
%
\begin{equation}
\bigl\langle\Xi_o+\nabla L(\wt\Theta),\bar\Theta-\wt\Theta \bigr
\rangle \ge (1-1/c) \mu\| \calP_{\Tp}\wt\Theta\|_*. \label{eq:subgrad_Theta_b_h_s2}
\end{equation}
Next, we set $\Xi=-Q+\bar\Theta-\Theta^*+\wh\Sigma$, so $\Xi
_o\in G_c$
by assumption. We also use $\nabla L(\wt\Theta) = \wt\Theta_o-\wh
\Sigma_o$. Then inequality (\ref{eq:subgrad_Theta_b_h_s2}) becomes
%
\begin{equation}
\bigl\langle-Q_o+\bar\Theta_o-\Theta^*_o+
\wt\Theta_o,\bar\Theta-\wt \Theta \bigr\rangle \ge(1-1/c) \mu\|
\calP_{\Tp}\wt\Theta\|_*. \label{eq:subgrad_Theta_b_h_s3}
\end{equation}
Finally, observe that, for arbitrary commensurate matrices $A$ and $B$,
we have $ \langle A_o,B \rangle=\tr(A^T_o B)=\tr(A^T_o B_o)=
\langle A_o,B_o \rangle$. Hence, we are free to replace the term $\bar\Theta
-\wt
\Theta$ in the angle bracket on the left-hand side of inequality (\ref
{eq:subgrad_Theta_b_h_s3}) by $\bar\Theta_o-\wt\Theta_o$. The corollary
then follows.
\end{pf}

We are now ready to derive the general recovery bound for the refined
estimator $\widetilde{\Sigma}$ in terms of $\bar\Theta$ and the
certificate $Q$. We denote $E=\wh\Sigma-\Sigma$ again, and note that $E_o=E$.

\begin{theorem}
\label{thmm:recovery_bound_raw}
If
%
\begin{equation}
-Q_o+\bar\Theta_o+E\in G_c,
\label{eq:Q_def}
\end{equation}
then
%
\begin{equation}
\tfrac{1}2 \| \widetilde{\Sigma} - \Sigma\|_F^2
+ (1-1/c) \mu\| \calP _{\Tp}\wt\Theta\|_* \le\tfrac{1}2 \bigl\| \bar
\Theta_o - \Theta_o^* \bigr\| _F^2 +
\tfrac{1}2\| \bar\Theta_o - Q_o
\|_F^2. \label{eq:thmm:recovery_bound_raw}
\end{equation}
\end{theorem}

\begin{pf}
We start from Lemma~\ref{lemma_equality}. By the construction of
$\widetilde{\Sigma}$ as in (\ref{eq:Sigma_refined}), the off-diagonal
elements of $\wt\Theta$ and $\widetilde{\Sigma}$ agree, that is,
$\wt
\Theta_o=\widetilde{\Sigma}_o$. In addition, $\Theta^*_o=\Sigma_o$.
Hence, $\wt\Theta_o - \Theta^*_o = \widetilde{\Sigma}_o-\Sigma_o =
\widetilde{\Sigma}-\Sigma$. Thus, after discarding the term $\frac
{1}{2}\|\wt\Theta_o-Q_o\|^2$, equation (\ref
{eq:recoverty_bound_equality}) becomes
%
\begin{eqnarray}\label{eq:recovery_bound_intermediate}
&&\tfrac{1}2 \| \widetilde{\Sigma}-\Sigma\|_F^2
+ \bigl\langle -Q_o + \bar \Theta _o -
\Theta_o^* + \wt\Theta_o , \bar\Theta_o -
\wt\Theta_o \bigr\rangle
\nonumber
\\[-8pt]
\\[-8pt]
\nonumber
&&\quad\le\tfrac{1}2 \bigl\| \bar\Theta_o - \Theta_o^*
\bigr\|_F^2 + \tfrac{1}2\| \bar \Theta_o
- Q_o\|_F^2.
\end{eqnarray}

Next we invoke Lemma~\ref{lemma_inequality}. Because $-\Theta_o^*+\wh
\Sigma_o = -\Sigma_o+\wh\Sigma_o = E_o=E$, condition (\ref{eq:Q_def})
translates into condition (\ref{eq:condition_in_Gc}), and hence
inequality (\ref{eq:Zhang_lemma_6_current_context}) holds. Finally,
plugging inequality (\ref{eq:Zhang_lemma_6_current_context}) into
inequality (\ref{eq:recovery_bound_intermediate}) yields the theorem.
\end{pf}

\subsection{Certificate construction}
\label{sec:certificate_construction}

From Theorem~\ref{thmm:recovery_bound_raw}, it is clear that the
recovery bounds on $\| \widetilde{\Sigma} - \Sigma\|_F^2$ and $\|
\calP
_{\Tp}\wt\Theta\|_*$ depend crucially on an appropriate certificate $Q$
such that $\|Q_o-\bar\Theta_o\|_F^2$ can be tightly bounded. This
section is dedicated to the construction of such a certificate.

Recall that $\bar\Theta\in\RR^{d\times d}$, which is intended to be
some low-rank approximation to $\Theta^*$, has been left unspecified so
far. Now we restrict $\bar\Theta$ to be a positive semidefinite matrix
of rank $r$, with the eigen-decomposition
%
\begin{equation}
\bar\Theta=\bar U \bar\Lambda{\bar U}^T. \label{eq:bar_Theta_eigendecompose}
\end{equation}
Here, $\bar U\in\RR^{d\times r}$ is the matrix of the orthonormal
eigenvectors of $\bar\Theta$ corresponding to the positive eigenvalues
that are the diagonal elements of $\bar\Lambda\in\RR^{r\times r}$.
Recall from Section~\ref{sec:rec_bnd_with_certificate} that $\T$
denotes the tangent space $T(\bar\Theta)$, and $\Tp$ denotes its
orthogonal complement $T(\bar\Theta)^{\perp}$. Then, with our specific
choice of $\bar\Theta$, the projectors $\calP_{\T}$ and $\calP
_{\Tp}$
are given by
\begin{subequations}
%
\begin{eqnarray}
\label{eq:P_Tbar}  \calP_{\T}(N) &=& \bar U {\bar U}^T N + N
\bar U {\bar U}^T - \bar U {\bar U}^T N \bar U {\bar
U}^T,
\\
\label{eq:P_Tbar_p} \calP_{\Tp}(N) &= &\bigl(I_d - \bar U {\bar
U}^T\bigr) N \bigl(I_d- \bar U {\bar U}^T
\bigr)
\end{eqnarray}
\end{subequations}
for arbitrary $N\in\mathbb{R}^{d\times d}$. For notational brevity,
from now on we will omit the parentheses surrounding the argument when
applying the projectors. Again with our specific choice of $\bar\Theta
$, we can give a more explicit characterization of $G_c$, defined
earlier in (\ref{eq:G_c_definition}), as
%
\begin{equation}
G_c = \bigl\{\Phi\in\RR^{d\times d} \dvt \calP_{\bar T}
\Phi= \mu\bar U \bar U^T \mbox{ and } \|\calP_{\Tp}\Phi
\|_2\le\mu/c\bigr\}. \label{eq:G_c_def_exp}
\end{equation}
We also define
%
\begin{equation}
\gamma= \bigl\|\bar U {\bar U}^T\bigr\|_{\infty}=\max
_{1\le i\le d} \bigl[\bar U {\bar U}^T\bigr]_{ii}
\le1. \label{eq:def_gamma}
\end{equation}
The second equality in (\ref{eq:def_gamma}) is due to the fact that
$\bar U{\bar U}^T$ is positive semidefinite, while the inequality
follows since $\bar U$ is a matrix of orthonormal eigenvectors.

Next, we obtain some technical results stating that, under certain
conditions, the operators $\calP_{\T}$ and $\calP_{\Omega}\calP
_{\T}$
are contractions under certain matrix norms (Lemma~\ref
{lemma:contraction}), and the operator $\calId-\calP_{\T}\calP
_{\Omega
}$, with $\calId$ the identity operator in $\RR^{d\times d}$, is
invertible (Lemma~\ref{lemma:invertibility}). These results essentially
follow from \cite{Hsu11} (e.g., their Lemmas 4, 8 and 10),
but we offer tighter bounds specialized to our study.

\begin{lemma}
\label{lemma:contraction}
For any diagonal matrix $D\in\RR^{d\times d}$, we have
%
\begin{equation}
\label{eq:calP_T_inf_to_inf} \|\calP_{\T} D\|_{\infty} \le3\gamma\|D
\|_{\infty}.
\end{equation}
For any matrix $M\in\RR^{d\times d}$, we have
%
\begin{equation}
\|\calP_{\T} M\|_{\infty} \le2\sqrt{\gamma} \|M
\|_{2} \label{eq:calP_T_inf_to_inf_a}
\end{equation}
and
%
\begin{equation}
\|\calP_{\Omega}\calP_{\T} M\|_{1} \le3\gamma\|M
\|_{1}. \label{eq:calP_T_1_to_1}
\end{equation}
\end{lemma}

\begin{pf}
The proof can be found in Appendix \ref{sec:aux_proof_refined_estimator}.
\end{pf}

\begin{lemma}
\label{lemma:invertibility}
Assume that $\gamma<1/3$. Then the operator $\calId-\calP_{\T}\calP
_{\Omega}\dvtx\RR^{d\times d}\rightarrow\RR^{d\times d}$ is a bijection,
and hence is invertible. Moreover, $\calId-\calP_{\T}\calP_{\Omega}$
satisfies, for any matrix $M\in\RR^{d\times d}$,
%
\begin{equation}
\bigl\|(\calId-\calP_{\T}\calP_{\Omega})^{-1}M
\bigr\|_{\infty}\le\frac
{1}{1-3\gamma}\|M\|_{\infty}. \label{eq:I_calPT_calPO_inversion}
\end{equation}
\end{lemma}

\begin{pf}
The proof can be found in Appendix \ref{sec:aux_proof_refined_estimator}.
\end{pf}

We demonstrate in Theorem~\ref{thmm:certificate} that, under appropriate
conditions, we can solve for $Q_o-\bar\Theta_o$ in an equation of the
form (\ref{eq:Q_def}), such that $Q-\bar\Theta$ has low rank and $\|
Q-\bar\Theta\|_2$ is small, which further implies that $\|Q_o-\bar
\Theta
_o\|_F^2$ is tightly bounded, as is desired. The techniques we use are
based on the proofs of \cite{Chandrasekaran12a}, Proposition~5.2 and
\cite{Hsu11}, Theorem~5.

\begin{theorem}
\label{thmm:certificate}
Assume that $\bar\Theta$ is positive semidefinite and has the
eigen-decomposition (\ref{eq:bar_Theta_eigendecompose}). Let $\bar
T=T(\bar\Theta)$. Let $G_c$ and $\gamma$ be defined as in (\ref
{eq:G_c_def_exp}) and (\ref{eq:def_gamma}), respectively. Suppose that
$\gamma$ satisfies
%
\begin{equation}
\label{eq:gamma_condition} \gamma< \frac{1}{c+3}.
\end{equation}
Let $A$ be the event on which
%
\begin{equation}
\label{eq:mu_condition} \mu\ge \biggl(\frac{1}{c}-\frac{\gamma
}{1-3\gamma}
\biggr)^{-1} \biggl(\frac{2\sqrt{\gamma}}{1-3\gamma} + 1 \biggr) \|E\|_2
\end{equation}
holds. Then, on the event $A$, there exists some $\Phi\in\bar T$ such that
%
\begin{equation}
-\Phi_o+E\in G_c \label{eq:Phi_in_Gc_condition}
\end{equation}
and
%
\begin{equation}
\|\Phi\|_2 \le \biggl( \frac{2}{c} +1 \biggr) \mu.
\label{eq:certificate_operator_norm}
\end{equation}
\end{theorem}

\textit{Remark}.
Note that inequality (\ref{eq:gamma_condition}) ensures that the
multiplicative factor $ (\frac{1}{c}-\frac{\gamma}{1-3\gamma
} )^{-1}$ in inequality (\ref{eq:mu_condition}) is positive.

\begin{pf*}{Proof of Theorem~\ref{thmm:certificate}}
We focus on the event $A$. Note that assumption (\ref
{eq:gamma_condition}) entails that $\gamma<1/4$ since $c\ge1$. As a
result, we can apply Lemma~\ref{lemma:invertibility} to conclude that
$\calId-\calP_{\T}\calP_{\Omega}$ is invertible, and that
inequality (\ref{eq:I_calPT_calPO_inversion}) holds. Then we can set
%
\begin{equation}
\Phi= (\calId-\calP_{\T}\calP_{\Omega})^{-1} \bigl(
\calP_{\T
}E- \mu \bar U{\bar U}^T \bigr). \label{eq:certificate_sol}
\end{equation}
We show that $\Phi$ has all the desired properties.

First, we apply the operator $\calId-\calP_{\T}\calP_{\Omega}$ on both
sides of equation (\ref{eq:certificate_sol}), and obtain
%
\begin{equation}
\Phi= \calP_{\T}\calP_{\Omega}\Phi+\calP_{\T}E-\mu
\bar U{\bar U}^T, \label{eq:certificate_eq_1}
\end{equation}
from which it is clear that $\Phi\in\T$.

Relationship (\ref{eq:Phi_in_Gc_condition}) is equivalent to
%
\begin{equation}
-(\Phi-\calP_{\Omega}\Phi)+E\in G_c, \label{eq:certificate_condition}
\end{equation}
which is further equivalent to the following two conditions by the
characterization (\ref{eq:G_c_def_exp}) of $G_c$. The first condition
is obtained by applying the operator $\calP_{\T}$ and the second one is
obtained by applying the operator $\calP_{\Tp}$ on both sides of
(\ref
{eq:certificate_condition}):
\begin{subequations}
%
\begin{eqnarray}
\label{eq:certificate_eq} -(\calId-\calP_{\T}\calP_{\Omega})\Phi+
\calP_{\T}E & = &\mu\bar U{\bar U}^T,
\\
\label{eq:certificate_ineq} \bigl\|\calP_{\Tp}(\Phi-\calP_{\Omega}\Phi- E)
\bigr\|_{2} & \le&\mu/c.
\end{eqnarray}
\end{subequations}

Equation (\ref{eq:certificate_eq}) is equivalent to equation (\ref
{eq:certificate_eq_1}), and hence is satisfied. Next, we check that
inequality (\ref{eq:certificate_ineq}) holds. By equation (\ref
{eq:certificate_sol}), inequalities (\ref{eq:I_calPT_calPO_inversion})
and (\ref{eq:calP_T_inf_to_inf_a}), we have
%
\begin{eqnarray}\label{eq:Phi_infty_inequality}
\|\Phi\|_{\infty} &\le&\frac{1}{ 1 - 3\gamma} \bigl\| \calP_{\T}E - \mu
\bar U{\bar U}^T \bigr\|_{\infty} \le\frac{1}{ 1 - 3\gamma} \bigl( \|
\calP _{\T}E \|_{\infty} + \bigl\| \mu\bar U{\bar U}^T
\bigr\|_{\infty} \bigr)
\nonumber
\\[-8pt]
\\[-8pt]
\nonumber
&\le&\frac{1}{ 1 - 3\gamma} \bigl( 2\sqrt{\gamma} \| E \|_2 + \gamma \mu \bigr).
\end{eqnarray}
Using inequality (\ref{eq:Phi_infty_inequality}) and $\|\calP_{\Tp
}\calP
_{\Omega}\Phi\|_2 \le\|\calP_{\Omega}\Phi\|_2 = \|\calP_{\Omega
}\Phi\|
_{\infty}\le\|\Phi\|_{\infty}$, we have
%
\begin{eqnarray}\label{eq:nu_orthogonal_condition}
\bigl\|\calP_{\Tp}(\Phi-\calP_{\Omega}\Phi- E)\bigr\|_{2}
&\le&\|\calP _{\Tp}\Phi \|_2+\|\calP_{\Tp}
\calP_{\Omega}\Phi\|_2 + \|\calP_{\Tp}E
\|_2
\nonumber
\\[-8pt]
\\[-8pt]
\nonumber
&\le& 0+\|\Phi\|_{\infty} + \|E\|_{2} \le \biggl(
\frac{2\sqrt
{\gamma}
}{ 1 - 3\gamma} +1 \biggr) \|E\|_{2} + \frac{\gamma}{1-3\gamma} \mu.
\end{eqnarray}
Then it is easy to see that inequality (\ref
{eq:nu_orthogonal_condition}), assumptions (\ref{eq:gamma_condition})
and (\ref{eq:mu_condition}) together imply inequality (\ref
{eq:certificate_ineq}). Hence, we have verified (\ref{eq:Phi_in_Gc_condition}).

Finally, starting from equation (\ref{eq:certificate_eq_1}), we have
\begin{eqnarray*}
\|\Phi\|_2 &\le&\|\calP_{\T}\calP_{\Omega}\Phi
\|_2 + \|\calP_{\T
}E\|_2 +\bigl \|\mu\bar U{\bar
U}^T\bigr\|_2 \le2\|\Phi\|_{\infty} + 2\|E
\|_2 + \mu\bigl \| \bar U{\bar U}^T\bigr\|_2
\\
&\le&\frac{2}{ 1 - 3\gamma} \bigl( 2\sqrt{\gamma} \| E \|_2 + \gamma \mu \bigr) + 2
\|E\|_2 + \mu=2 \biggl( \frac{ 2\sqrt{\gamma}
}{1-3\gamma
} + 1 \biggr) \|E
\|_2 + \biggl(\frac{2\gamma}{1-3\gamma}+1 \biggr)\mu
\\
&\le& 2 \biggl( \frac{1}{c} - \frac{\gamma}{1-3\gamma} \biggr) \mu+ \biggl(
\frac{2\gamma}{1-3\gamma}+1 \biggr)\mu= \biggl( \frac{2}{c} +1 \biggr)\mu.
\end{eqnarray*}
Here, the second inequality follows from the fact that $\|\calP_{\T
}\calP_{\Omega}\Phi\|_2\le2\|\calP_{\Omega}\Phi\|_2\le2\|\Phi\|
_{\infty
}$, the third inequality follows from inequality (\ref
{eq:Phi_infty_inequality}), and the fourth inequality follows by
assumption~(\ref{eq:mu_condition}). Hence, inequality (\ref
{eq:certificate_operator_norm}) is established.
\end{pf*}

\subsection{Recovery bound for the refined estimator \texorpdfstring{$\widetilde
{\Sigma}$}{widetilde{Sigma}}}
\label{sec:rec_bnd_refined_estimator}

In this section, we state in Corollary~\ref{cor:recovery_bound} the
main recovery bound that will lead to the oracle inequality for the
refined estimator $\widetilde{\Sigma}$. We recall $U_r^*$, $\gamma_r$
and $\Theta_r^*$ as introduced in equations (\ref{eq:U_star_r}),
(\ref
{eq:gamma_r}) and (\ref{eq:Theta_star_r}).

\begin{corollary}
\label{cor:recovery_bound}
Let $r$ be such that $0\le r\le r^*$ and
%
\begin{equation}
\gamma_r<\frac{1}{c+3}. \label{eq:gamma_r_condition}
\end{equation}
Let $A$ be the event on which the regularization parameter $\mu$ satisfies
%
\begin{equation}
\label{eq:mu_gamma_r_condition} \mu\ge \biggl(\frac{1}{c}-\frac{\gamma_r}{1-3\gamma_r}
\biggr)^{-1} \biggl(\frac{2\sqrt{\gamma_r}}{1-3\gamma_r} + 1 \biggr) \|E\|_2.
\end{equation}
(Note that inequalities (\ref{eq:gamma_r_condition}) and (\ref
{eq:mu_gamma_r_condition}) are just inequalities (\ref
{eq:gamma_condition}) and (\ref{eq:mu_condition}) with the substitution
of $\gamma$ by $\gamma_r$.) Then, on the event $A$ we have
%
\begin{equation}
\| \widetilde{\Sigma} - \Sigma\|_F^2 + (2-2/c) \mu\|
\calP _{T(\Theta
^*_r)^{\perp}}\wt\Theta\|_* \le\sum_{ j\dvt r<j\le r^* }
\lambda _j^2\bigl(\Theta ^*\bigr) + 2 (1+2/c)^2
r \mu^2. \label{eq:recovery_bound_2}
\end{equation}
\end{corollary}

\textit{Remark}.
We can now see that the choice $c=1$ in $G_c$ is sufficient for proving
a bound on $\|\widetilde{\Sigma}-\Sigma\|_F^2$. With this choice of
$c$, inequality (\ref{eq:gamma_r_condition}) states that $U_r^*$, the
truncated matrix of the orthonormal eigenvectors of $\Theta^*$
corresponding to the $r$ largest eigenvalues, should satisfy the mild
condition $\|U_r^* U_r^{*T}\|_{\infty}<1/4$. On the other hand, the
choice $c>1$ leads to a bound on $\|\calP_{\Omega}(\wt\Theta-\Theta
^*)\|
_1$ as we will see in Appendix \ref{sec:diagonal_deviation}.

\begin{pf*}{Proof of Corollary \ref{cor:recovery_bound}}
We start with the general recovery bound, Theorem~\ref
{thmm:recovery_bound_raw}. In the context of Theorem~\ref
{thmm:recovery_bound_raw}, $\bar\Theta$ and $Q$ should satisfy
relationship (\ref{eq:Q_def}) but are otherwise completely arbitrary.

We now set $\bar\Theta=\Theta^*_r$, so $\bar\Theta$ is positive
semidefinite. We also concentrate on the event $A$. Then, by
assumptions (\ref{eq:gamma_r_condition}) and (\ref
{eq:mu_gamma_r_condition}), inequalities (\ref{eq:gamma_condition}) and
(\ref{eq:mu_condition}) hold with the substitution of $\gamma$ by
$\gamma_r$. Hence, Theorem~\ref{thmm:certificate} applies. We let
$\Phi$
be constructed according to Theorem~\ref{thmm:certificate} for the
chosen $\bar\Theta=\Theta_r^*$, so that $\Phi\in\T=T(\Theta_r^*)$,
$-\Phi_o+E\in G_c$, and $\|\Phi\|_2\le(1+2/c)\mu$. We set $Q=\bar
\Theta
+\Phi$ so $Q - \bar\Theta= \Phi$. Then relationship (\ref{eq:Q_def})
is satisfied, and Theorem~\ref{thmm:recovery_bound_raw} further states
that inequality (\ref{eq:thmm:recovery_bound_raw}) holds. We proceed to
bound the two terms on the right-hand side of inequality (\ref
{eq:thmm:recovery_bound_raw}) separately.

First, we consider the term $\| \bar\Theta_o - \Theta_o^* \|_F^2$.
Here and below, for brevity, we sometimes abbreviate the summation
range $j\dvt r<j\le r^*$ by $j>r$. We have
\[
\bigl\| \bar\Theta_o - \Theta_o^* \bigr\|_F^2
\le\bigl\| \bar\Theta- \Theta^* \bigr\| _F^2 = \bigl\|
\Theta^*_r - \Theta^* \bigr\|_F^2 = \sum
_{ j>r } \lambda _j^2\bigl(\Theta ^*
\bigr).
\]

Next, we consider the term $\| \bar\Theta_o - Q_o\|_F^2$. Using the
fact that $\Phi\in T(\Theta_r^*)$ and so $\operatorname{rank}(\Phi)\le2r$, and
$\|\Phi\|_2\le(1+2/c)\mu$, we have
\[
\|\bar\Theta_o - Q_o\|_F^2 =
\|\Phi_o\|_F^2 \le\|\Phi
\|_F^2 \le 2r\| \Phi\|_2^2
\le2(1+2/c)^2 r\mu^2.
\]

Combining both displays, we conclude that inequality (\ref
{eq:recovery_bound_2}) holds.
\end{pf*}

The bound on $\| \widetilde{\Sigma} - \Sigma\|_F$ obtained in
Corollary~\ref{cor:recovery_bound} can be further refined by optimizing
the balance between the approximation error and the estimation error.
We can also fix our choice of the regularization parameter $\mu$
according to inequality (\ref{eq:mu_gamma_r_condition}). These
considerations finally lead to our proof of Theorem~\ref
{thmm:recovery_bound_optimize}.

\begin{pf*}{Proof of Theorem~\ref{thmm:recovery_bound_optimize}}
We fix $c=2$, and $\gamma'=1/9$. Then inequality (\ref
{eq:gamma_r_condition}) holds with the substitution of $\gamma_r$ by
$\gamma'$. Let $A$ be the event
%
\begin{equation}
A = \biggl\{ \biggl(\frac{1}{c}-\frac{\gamma'}{1-3\gamma'} \biggr)^{-1}
\biggl(\frac{2\sqrt{\gamma'}}{1-3\gamma'} + 1 \biggr) \|E\|_2 \le \mu \le\bar\mu
\biggr\}. \label{eq:event_A_gamma_p_1}
\end{equation}
That is, $A$ is the event on which both $\mu\le\bar\mu$ and
inequality (\ref{eq:mu_gamma_r_condition}) with the substitution of
$\gamma_r$ by $\gamma'$ hold. Note that the multiplicative factor in
front of $\|E\|_2$ on the right-hand side of (\ref
{eq:event_A_gamma_p_1}) exactly equals $C=6$ with our choices of $c$
and $\gamma'$. Then, by Theorem~\ref{thmm:main_operator_norm_bound} and
our choices (\ref{MU_2}) and (\ref{BARMU_2}) of $\mu$ and $\bar\mu
$, we
conclude that $\PP(A)\ge1-\alpha-\alpha^2/4>1-2\alpha$, and for the
rest of the proof we concentrate on the event $A$.

We let $R$ be chosen according to (\ref{eq:gamma_R_condition}), so in
particular $\gamma_R\le1/9=\gamma'$. Because $\gamma_r$ is
nondecreasing in $r$, and inequalities (\ref{eq:gamma_r_condition}) and
(\ref{eq:mu_gamma_r_condition}) hold with the substitution of $\gamma
_r$ by $\gamma'$, it is straightforward to conclude that
inequalities (\ref{eq:gamma_r_condition}) and (\ref
{eq:mu_gamma_r_condition}) hold in terms of $\gamma_r$ for all $0\le
r\le R$. Hence, by Corollary~\ref{cor:recovery_bound},
inequality (\ref
{eq:recovery_bound_2}) holds for all $0\le r\le R$. Then, after
discarding the term $(2-2/c) \mu\| \calP_{T(\Theta^*_r)^{\perp}}\wt
\Theta\|_*$ on the left-hand side of inequality (\ref
{eq:recovery_bound_2}), we obtain, for all $0\le r\le R$, that
%
\begin{equation}
\| \widetilde{\Sigma} - \Sigma\|_F^2 \le\sum
_{j>r}\lambda _j^2\bigl(\Theta ^*
\bigr) + 2(1+2/c)^2 r\mu^2 \le\sum
_{j>r}\lambda_j^2\bigl(\Theta^*\bigr)
+ 8 r\bar\mu^2. \label{eq:recovery_bound_3}
\end{equation}
Here, the second inequality in (\ref{eq:recovery_bound_3}) follows
because $c=2$ and $\mu\le\bar\mu$. Finally, the theorem follows by
taking the minimum of inequality (\ref{eq:recovery_bound_3}) over
$0\le
r\le R$.
\end{pf*}


\begin{appendix}\label{app}
\section{Discussion of some basic concepts}
\label{sec:def_elliptical_distribution}

In this section, we present formal definitions of some basic concepts
in this paper and then discuss the characterization of the
semi-parametric elliptical copula model. We first present the
definition of an elliptical distribution; see, for instance, \cite{Cambani81}.

\begin{definition}
\label{def:elliptical_distribution}
A random vector $Y=(Y_1,\ldots,Y_d)^T\in\RR^d$ has an elliptical
distribution if for some $\mu\in\RR^d$ and some positive semidefinite
matrix $\overline{\Sigma}\in\RR^{d\times d}$, the characteristic
function $\varphi_{Y-\mu}(t)$ of $Y-\mu$ is a function of the quadratic
form $t^T\overline{\Sigma}t$, that is, $\varphi_{Y-\mu}(t)=\phi
(t^T\overline{\Sigma}t)$ for some function $\phi$. We write $Y\sim
\mathcal{E}_d(\mu,\overline\Sigma,\phi)$, and call $\phi$ the
characteristic generator.
\end{definition}

Next, we present the definition of a copula \cite{Sklar96}; see, for
instance, \cite{Embrechts03}, Theorem~2.2.

\begin{definition}
The copula $C\dvtx [0,1]^d\rightarrow[0,1]$ of a continuous random vector
$Y=(Y_1,\ldots, Y_d)^T \in\RR^d$ is the joint distribution function of
the transformed random vector $U=(F_1(Y_1),\ldots, F_d(Y_d) )^T \in\RR
^d$ on the unit cube $[0,1]^d$, using the marginal distribution
functions $F_j(y)=\PP\{ Y_j \le y\}$ for $1\le j\le d$.
\end{definition}

We recall the basic property that copulas are invariant under strictly
increasing transformations of the individual vector components of the
underlying distribution; see, for instance, \cite{Embrechts03},
Theorem~2.6. It follows from this invariance property
that, if the random vector $X\in\mathbb{R}^d$ follows a distribution
from the semi-parametric elliptical copula model, and if $X$ has the
same copula with an elliptically distributed random vector $Y\in
\mathbb
{R}^d$ such that $Y\sim\mathcal{E}_d(\mu,\overline\Sigma,\phi)$, then
the copula of $X$ is uniquely characterized by the same characteristic
generator $\phi$ and a copula correlation matrix $\Sigma$, defined as
$[\Sigma]_{k\ell}= [\overline{\Sigma}]_{k\ell} / ([\overline
{\Sigma
}]_{kk}[\overline{\Sigma}]_{\ell\ell} )^{1/2}$ for all $1\le k,\ell
\le d$.

\section{Auxiliary proofs for Section \texorpdfstring{\protect\ref
{sec:proof_refined_estimator}}{6}}
\label{sec:aux_proof_refined_estimator}

This section contains the proofs of some auxiliary lemmas in
Section~\ref{sec:proof_refined_estimator}.

\begin{pf*}{Proof of Lemma~\ref{lemma:AMB_infty}}
We let $e_i\in\RR^{d}$ denote the vector with one at the $i$th position
and zeros elsewhere, and $\|\cdot\|$ denote the Euclidean norm for
vectors. Then we have
\begin{eqnarray*}
\| A C B \|_{\infty} &=& \max_{i,j} \bigl\llvert
e_i^T A C B e_j \bigr\rrvert \le \max
_{i,j} \bigl\| e_i^T A \bigr\| \| C B
e_j \| \le\max_{i,j} \bigl\| e_i^T
A \bigr\| \| C\|_2 \| B e_j \|
\\
&=& \max_{i,j} \sqrt{ e_i^T A
A^T e_i } \| C\|_2 \sqrt{
e_j^T B^T B e_j } \le\sqrt{ \bigl\| A
A^T\bigr\|_{\infty} } \sqrt{ \bigl\| B^T B\bigr\|_{\infty} }
\| C\| _2.
\end{eqnarray*}
Here, the first equality follows from an observation in the proof of
\cite{Chandrasekaran11}, Proposition~4, and the first inequality
follows by the Cauchy--Schwarz inequality. The lemma follows.
\end{pf*}

\begin{pf*}{Proof of Lemma~\ref{lemma:contraction}}
Let $D\in\RR^{d\times d}$ be an arbitrary diagonal matrix, and $M\in
\RR
^{d\times d}$ an arbitrary matrix. We first prove inequality (\ref
{eq:calP_T_inf_to_inf}). Using equation (\ref{eq:P_Tbar}), we have
%
\begin{equation}
\|\calP_{\T} D\|_{\infty} \le\bigl\llVert \bigl(\bar U{\bar
U}^T \bigr) D \bigr\rrVert _{\infty} + \bigl\llVert D \bigl(
\bar U{\bar U}^T \bigr) \bigr\rrVert _{\infty} + \bigl\llVert
\bigl(\bar U{\bar U}^T \bigr) D \bigl(\bar U{\bar U}^T
\bigr) \bigr\rrVert _{\infty}. \label{eq:contraction_1}
\end{equation}
We bound the terms on the right-hand side
of inequality (\ref{eq:contraction_1}) separately. Note that, although
$\|\cdot\|_{\infty}$, the element-wise $\ell_{\infty}$ norm, is not
submultiplicative, it is easy to see that the inequality $\|A B\|
_{\infty}\le\|A\|_{\infty}\|B\|_{\infty}$ holds when at least one of
$A,B$ is a diagonal matrix. Hence, we have
%
\begin{equation}
\max \bigl\{ \bigl\llVert \bigl(\bar U{\bar U}^T \bigr) D \bigr
\rrVert _{\infty}, \bigl\llVert D \bigl(\bar U{\bar U}^T \bigr)
\bigr\rrVert _{\infty} \bigr\} \le \|\bar U {\bar U}^T
\|_{\infty} \|D\|_{\infty} = \gamma\|D\|_{\infty}.
\label{eq:contraction_2}
\end{equation}
Next, setting $A=B= \bar U{\bar U}^T$ and $C=D$ in Lemma~\ref
{lemma:AMB_infty} yields
%
\begin{eqnarray}\label{eq:contraction_3}
\bigl\| \bigl(\bar U{\bar U}^T \bigr) D \bigl(\bar U{\bar
U}^T \bigr) \bigr\| _{\infty} &\le&\sqrt{ \bigl\| \bar U {\bar
U}^T \bar U {\bar U}^T \bigr\| _{\infty
} \bigl\| \bar U {
\bar U}^T \bar U {\bar U}^T \bigr\|_{\infty} } \|D
\|_2
\nonumber
\\[-8pt]
\\[-8pt]
\nonumber
&=& \sqrt{\bigl \| \bar U {\bar U}^T\bigr \|_{\infty} \bigl\| \bar U {\bar
U}^T \bigr\| _{\infty} } \|D\|_2 = \gamma\|D
\|_{\infty}.
\end{eqnarray}
Here, the final equality follows because $D$ is diagonal and so $\|D\|
_2=\|D\|_{\infty}$. Finally, plugging inequalities (\ref
{eq:contraction_2}) and (\ref{eq:contraction_3}) into inequality (\ref
{eq:contraction_1}) yields inequality (\ref{eq:calP_T_inf_to_inf}).

To prove inequality (\ref{eq:calP_T_inf_to_inf_a}), note that, again by
equation (\ref{eq:P_Tbar}), we have
%
\begin{equation}
\|\calP_{\T} M\|_{\infty} \le\bigl\llVert \bigl(\bar U{\bar
U}^T \bigr) M \bigr\rrVert _{\infty} + \bigl\llVert \bigl(
I_d - \bar U{\bar U}^T \bigr) M \bigl(\bar U{\bar
U}^T \bigr) \bigr\rrVert _{\infty}. \label{eq:contraction_4}
\end{equation}
Setting $A=UU^T$, $B=I_d$ and $C=M$ in Lemma~\ref{lemma:AMB_infty} yields
%
\begin{equation}
\bigl\llVert \bigl(\bar U{\bar U}^T \bigr) M \bigr\rrVert
_{\infty} \le \sqrt {\gamma}\|M\|_2, \label{eq:contraction_5}
\end{equation}
while setting $A=  (I_d - \bar U{\bar U}^T )$, $B= U U^T$ and
$C=M$ in Lemma~\ref{lemma:AMB_infty} yields
%
\begin{equation}
\bigl\llVert \bigl( I_d - \bar U{\bar U}^T \bigr) M
\bigl( \bar U{\bar U}^T \bigr) \bigr\rrVert _{\infty} \le\sqrt{
\gamma}\|M\|_2. \label{eq:contraction_6}
\end{equation}
Inequality (\ref{eq:calP_T_inf_to_inf_a}) then follows from
inequalities (\ref{eq:contraction_4}), (\ref{eq:contraction_5}) and
(\ref{eq:contraction_6}).

Finally, we prove inequality (\ref{eq:calP_T_1_to_1}). Note that $\|
\cdot\|_{\infty}$ and $\|\cdot\|_{1}$ are dual norms. Then
\begin{eqnarray*}
\|\calP_{\Omega} \calP_{\T} M\|_{1} &=& \sup
_{N\dvt \|N\|_{\infty}\le1} \langle \calP_{\Omega}\calP_{\T} M, N
\rangle = \sup_{N\dvt \|N\|
_{\infty
}\le1} \langle \calP_{\T} M,
\calP_{\Omega} N \rangle\\
& = &\sup_{N\dvt \|N\|
_{\infty}\le1} \langle M,
\calP_{\T} \calP_{\Omega} N \rangle
\nonumber
\\
&\le&\sup_{N\dvt \|N\|_{\infty}\le1} \|M\|_{1} \|\calP_{\T}
\calP _{\Omega} N\|_{\infty} \le3\gamma\sup_{N\dvt \|N\|_{\infty}\le1}
\|M\|_{1} \| \calP _{\Omega} N\|_{\infty}
\\
&\le&3\gamma\sup_{N\dvt \|N\|_{\infty}\le1} \|M\|_{1} \| N
\|_{\infty} \le 3\gamma\|M\|_{1},
\end{eqnarray*}
using first H\"{o}lder's inequality and then inequality (\ref
{eq:calP_T_inf_to_inf}) on the diagonal matrix $\calP_{\Omega} N$.
\end{pf*}

\begin{pf*}{Proof of Lemma~\ref{lemma:invertibility}}
We assume that $\gamma<1/3$. Let $M\in\RR^{d\times d}$ be an arbitrary
matrix. Applying inequality (\ref{eq:calP_T_inf_to_inf}) in Lemma~\ref
{lemma:contraction} on the diagonal matrix $\calP_{\Omega} M$, we obtain
\[
\|\calP_{\T} \calP_{\Omega} M\|_{\infty} \le3\gamma\|
\calP _{\Omega} M\| _{\infty} \le3\gamma\|M\|_{\infty}.
\]
Then, by the triangle inequality,
\[
\bigl\|(\calId-\calP_{\T}\calP_{\Omega})M\bigr\|_{\infty} \ge\|M
\|_{\infty
} - \| \calP_{\T}\calP_{\Omega} M
\|_{\infty} \ge(1-3\gamma) \|M\| _{\infty}.
\]
Because $\gamma<1/3$, $\|(\calId-\calP_{\T}\calP_{\Omega})M\|
_{\infty
}=0$ if and only if $\|M\|_{\infty}=0$, or equivalently $M=0$. Thus,
the null space of the operator $\calId-\calP_{\T}\calP_{\Omega}$
is the
zero matrix. Hence, $\calId-\calP_{\T}\calP_{\Omega}$ is a bijection,
and thus invertible.

Next, we prove inequality (\ref{eq:I_calPT_calPO_inversion}). Let
$(\calId-\calP_{\T}\calP_{\Omega})^{-1}M = M'$, or equivalently
$M=(\calId-\calP_{\T}\calP_{\Omega})M'$. Then, analogues to the
derivation above, we have
\[
\|M\|_{\infty} = \bigl\|(\calId-\calP_{\T}\calP_{\Omega})M'
\bigr\|_{\infty
} \ge (1-3\gamma)\bigl\|M'\bigr\|_{\infty} = (1-3
\gamma)\bigl\|(\calId-\calP_{\T}\calP _{\Omega})^{-1}M
\bigr\|_{\infty},
\]
which is inequality (\ref{eq:I_calPT_calPO_inversion}).
\end{pf*}


\section{Bounding the diagonal deviation of the low-rank matrix estimator}
\label{sec:diagonal_deviation}

We commented in the remark following Corollary~\ref{cor:recovery_bound}
that the choice $c=1$ in $G_c$ is sufficient for proving a bound on $\|
\widetilde{\Sigma}-\Sigma\|_F^2$. On the other hand, exactly as
commented in \cite{Zhang12}, and as is apparent from Theorem~\ref
{thmm:recovery_bound_raw},\vspace*{1pt} choosing $c>1$ leads to a bound for $\calP
_{\Tp}\wt\Theta$, that is, the portion of $\wt\Theta$ orthogonal
to the
tangent space $\T$. As in \cite{Hsu11}, such a bound can be further
exploited to control $\calP_{\Omega}(\wt\Theta-\Theta^*)$, which
in our
case is the deviation of $\wt\Theta$ from $\Theta^*$ on the diagonal.
We first present a lemma toward the bound for $\calP_{\Omega}(\wt
\Theta
-\Theta^*)$. The proof of the lemma is a straightforward modification
of the proof of \cite{Hsu11}, Theorem~7; for completeness, we include
it here. We employ the same notation as in Section~\ref
{sec:certificate_construction}, and we denote $E=\wh\Sigma-\Sigma$ again.

\begin{lemma}
\label{lemma:Theta_hat_diagonal}
Let $r=\operatorname{rank}(\bar\Theta)$. We have
%
\begin{equation}
(1-3\gamma) \bigl\|\calP_{\Omega}\bigl(\wt\Theta-\Theta^*\bigr)
\bigr\|_1 \le\bigl\|\calP _{\Tp
}\bigl(\wt\Theta-\Theta^*\bigr)\bigr\|_* +
4r \bigl(\|E\|_2+\mu \bigr). \label{eq:Theta_hat_diagonal_lemma}
\end{equation}
\end{lemma}

\begin{pf}
Let $\wt\Delta_{\Theta} = \wt\Theta-\Theta^*$. The optimality of
$\wt
\Theta$ for the convex program (\ref{eq:convex_program}) implies that
we can fix $\Psi\in\mu\partial\|\wt\Theta\|_*$ such that
equation (\ref{eq:Theta_hat_optimality}) holds. Using $\nabla L(\wt
\Theta) = \wt\Theta_o-\wh\Sigma_o$, equation (\ref
{eq:Theta_hat_optimality}) is equivalent to
%
\begin{equation}
\wt\Delta_{\Theta} = \calP_{\Omega}\wt\Delta_{\Theta} + E -
\Psi. \label{eq:optimality_2}
\end{equation}
Applying $\calP_{\Omega}\calP_{\T}$ on both sides of equation (\ref
{eq:optimality_2}) gives
%
\begin{equation}
\calP_{\Omega}\calP_{\T} \wt\Delta_{\Theta} =
\calP_{\Omega
}\calP_{\T} \calP_{\Omega}\wt
\Delta_{\Theta} + \calP_{\Omega}\calP_{\T} E - \calP
_{\Omega}\calP_{\T} \Psi. \label{eq:10}
\end{equation}
Then, using equation (\ref{eq:10}), we have
%
\begin{eqnarray}\label{eq:DTheta_diagonal_expansion}
\calP_{\Omega}\wt\Delta_{\Theta} &=& \calP_{\Omega}
\calP_{\Tp
}\wt\Delta _{\Theta} + \calP_{\Omega}
\calP_{\T}\wt\Delta_{\Theta}
\nonumber
\\[-8pt]
\\[-8pt]
\nonumber
&=& \calP_{\Omega}\calP_{\Tp}\wt\Delta_{\Theta} +
\calP_{\Omega
}\calP_{\T
} \calP_{\Omega}\wt
\Delta_{\Theta} + \calP_{\Omega}\calP_{\T} E - \calP
_{\Omega}\calP_{\T} \Psi.
\end{eqnarray}
We apply $\|\cdot\|_1$ on both sides of equation (\ref
{eq:DTheta_diagonal_expansion}). Note that, for any matrix $M\in\RR
^{d\times d}$, $\|\calP_{\Omega}M\|_1 = \|\calP_{\Omega}M\|_*$. In
addition, inequality (\ref{eq:calP_T_1_to_1}) implies that $\|\calP
_{\Omega}\calP_{\T} \calP_{\Omega}\wt\Delta_{\Theta}\|_1\le
3\gamma\|
\calP_{\Omega}\wt\Delta_{\Theta}\|_1$. Hence, we have
%
\begin{eqnarray}\label{eq:DTheta_diagonal_expansion_2}
\|\calP_{\Omega}\wt\Delta_{\Theta}\|_1 &\le&\|
\calP_{\Omega
}\calP_{\Tp
}\wt\Delta_{\Theta}
\|_1 + \|\calP_{\Omega}\calP_{\T} \calP
_{\Omega}\wt \Delta_{\Theta}\|_1 + \|
\calP_{\Omega}\calP_{\T} E\|_1 + \|\calP
_{\Omega
}\calP_{\T} \Psi\|_1
\nonumber
\\[-8pt]
\\[-8pt]
\nonumber
&\le&\|\calP_{\Omega}\calP_{\Tp}\wt\Delta_{\Theta}\|_* +
3\gamma \|\calP _{\Omega}\wt\Delta_{\Theta}\|_1 + \|
\calP_{\Omega}\calP_{\T} E\| _* + \| \calP_{\Omega}
\calP_{\T} \Psi\|_*.
\end{eqnarray}
Note that, for any matrix $M\in\RR^{d\times d}$, we have $\calP
_{\Omega
}M = I_d\circ M$. By \cite{Horn91}, Theorem~5.5.19, $\|I_d\circ M\|
_*\le\|M\|_*$. In addition, $\operatorname{rank}(\calP_{\T}M)\le2r$,
and so $\|\calP_{\T} M\|_*\le2r \|\calP_{\T} M\|_2\le4r\|M\|_2$.
Hence, from inequality (\ref{eq:DTheta_diagonal_expansion_2}), we
further deduce
\[
(1-3\gamma) \|\calP_{\Omega} \wt\Delta_{\Theta}\|_1
\le\|\calP _{\Tp} \wt\Delta_{\Theta}\|_* + \|
\calP_{\T} E\|_* + \|\calP_{\T} \Psi \|_* \le \|
\calP_{\Tp} \wt\Delta_{\Theta}\|_* + 4r\|E\|_2 + 4r
\|\Psi\|_2.
\]
The corollary then follows by noting that $\|\Psi\|_2\le\mu$.
\end{pf}

We now state a concrete bound for $\calP_{\Omega}(\wt\Theta-\Theta^*)$.

\begin{theorem}
\label{thmm:Theta_hat_diagonal}
Let $\mu$ and $\bar\mu$ be as in (\ref{MU_2}) and (\ref{BARMU_2}),
respectively, and let
%
\begin{equation}
\mu' = C \bigl\{ C_1 \max \bigl[ \sqrt{ \| T
\|_2} f(n,d,\alpha), f^2(n,d,\alpha) \bigr] +
C_2 f^2(n,d,\alpha) \bigr\}, \label{MU_prime_2}
\end{equation}
all with $0<\alpha<1/2$, $C_1=\pi$, $C_2=3\pi^2/16 < 1.86$, and $C=6$.
We recall $R$ as defined in (\ref{eq:gamma_R_condition}). Then, with
probability exceeding $1-2\alpha$, we have
%
\begin{eqnarray} \label{eq:Theta_hat_diagonal_thmm}
&&\bigl\|\calP_{\Omega}\bigl(\wt\Theta-\Theta^*\bigr)\bigr\|_1
\nonumber
\\[-8pt]
\\[-8pt]
\nonumber
&&\quad\le\min
_{0\le r\le R} \biggl\{ \frac{3}{2\mu'}\sum
_{j\dvt r<j\le r^*}\lambda_j^2\bigl(\Theta^*\bigr)
+ \frac
{3}{2}\sum_{j\dvt r<j\le r^*}\lambda_j
\bigl(\Theta^*\bigr) + 19 r\bar\mu \biggr\}.
\end{eqnarray}
\end{theorem}

\begin{pf}
We fix $c=2$, and $\gamma'=1/9$. Then inequality (\ref
{eq:gamma_r_condition}) holds with the substitution of $\gamma_r$ by~$\gamma'$. Let $A$ be the event
%
\begin{equation}
A = \biggl\{ \biggl(\frac{1}{c}-\frac{\gamma'}{1-3\gamma'} \biggr)^{-1}
\biggl(\frac{2\sqrt{\gamma'}}{1-3\gamma'} + 1 \biggr) \|E\|_2 \le \mu'
\le\mu\le\bar\mu \biggr\}. \label{eq:event_A_gamma_p_2}
\end{equation}
Hence, on the event $A$, both $\mu'\le\mu\le\bar\mu$, and
inequality (\ref{eq:mu_gamma_r_condition}) with the substitution of
$\gamma_r$ by $\gamma'$, hold. Note that the multiplicative factor in
front of $\|E\|_2$ on the right-hand side of (\ref
{eq:event_A_gamma_p_2}) exactly equals $C=6$ with our choices of $c$
and $\gamma'$. Then, by Theorem~\ref{thmm:main_operator_norm_bound} and
our choices (\ref{MU_prime_2}), (\ref{MU_2}) and (\ref{BARMU_2}) of
$\mu
'$, $\mu$ and $\bar\mu$, we conclude that $\PP(A)\ge1-\alpha
-\alpha
^2/4>1-2\alpha$, and for the rest of the proof we focus on the event $A$.

Note that Lemma~\ref{lemma:Theta_hat_diagonal} provides a bound on $\|
\calP_{\Omega}(\wt\Theta-\Theta^*)\|_1$ through the chosen $\bar
\Theta$
and the associated $\Tp$. We fix an arbitrary $0\le r\le R$, and choose
$\bar\Theta=\Theta^*_r$, which implies that $\gamma=\gamma_r$. Then
\begin{eqnarray*}
\calP_{\Tp} \bigl(\wt\Theta-\Theta^*\bigr)& =& \calP_{T(\Theta_r^*)^{\perp
}}\wt
\Theta- \calP_{T(\Theta_r^*)^{\perp}}\Theta^*\\
& = &\calP_{T(\Theta
_r^*)^{\perp}}\wt\Theta- \bigl(
\Theta^*-\Theta^*_r\bigr)
\end{eqnarray*}
and so
%
\begin{equation}
\bigl\|\calP_{\Tp} \bigl(\wt\Theta-\Theta^*\bigr)\bigr\|_* \le\|
\calP_{T(\Theta
_r^*)^{\perp
}}\wt\Theta\|_*+ \sum_{j>r}
\lambda_j\bigl(\Theta^*\bigr). \label{eq:Theta_hat_nuclear_star}
\end{equation}
Plugging inequality (\ref{eq:Theta_hat_nuclear_star}) into
inequality (\ref{eq:Theta_hat_diagonal_lemma}) with the substitution of
$\gamma$ by $\gamma_r$ yields
%
\begin{eqnarray}\label{eq:Theta_hat_diagonal_1}
&&\bigl\|\calP_{\Omega}\bigl(\wt\Theta-\Theta^*\bigr)\bigr\|_1
\nonumber
\\[-8pt]
\\[-8pt]
\nonumber
&&\quad\le
\biggl(\frac
{1}{1-3\gamma
_r} \biggr) \biggl[ \|\calP_{T(\Theta_r^*)^{\perp}}\wt\Theta\|_*+
\sum_{j>r}\lambda_j\bigl(\Theta^*\bigr) +
4r \bigl(\|E\|_2+\mu \bigr) \biggr].
\end{eqnarray}

As argued in the proof of Theorem~\ref{thmm:recovery_bound_optimize},
because inequalities (\ref{eq:gamma_r_condition}) and (\ref
{eq:mu_gamma_r_condition}) hold with the substitution of $\gamma_r$ by
$\gamma'$, we conclude that inequalities (\ref{eq:gamma_r_condition})
and (\ref{eq:mu_gamma_r_condition}) hold in terms of $\gamma_r$. Hence,
by Corollary~\ref{cor:recovery_bound}, inequality (\ref
{eq:recovery_bound_2}) applies, and we have
%
\begin{equation}
\| \calP_{T(\Theta_r^*)^{\perp}}\wt\Theta\|_* \le\frac{1}{\mu
} \biggl[ \sum
_{j>r}\lambda_j^2\bigl(
\Theta^*\bigr) + 8 r \mu^2 \biggr]. \label{eq:Theta_hat_Tp_nuclear}
\end{equation}
Plugging inequality (\ref{eq:Theta_hat_Tp_nuclear}) into
inequality (\ref{eq:Theta_hat_diagonal_1}), we have
%
\begin{eqnarray}\label{eq:Theta_hat_diagonal_1_final}
\bigl\|\calP_{\Omega}\bigl(\wt\Theta-\Theta^*\bigr)\bigr\|_1 &\le&
\biggl(\frac
{1}{1-3\gamma
_r} \biggr) \biggl\{ \frac{1}{\mu} \biggl[ \sum
_{j>r}\lambda _j^2\bigl(
\Theta ^*\bigr) + 8 r \mu^2 \biggr] + \sum
_{j>r}\lambda_j\bigl(\Theta^*\bigr) + 4r \bigl(\|E\|
_2+\mu \bigr) \biggr\}
\nonumber
\\
&\le&\frac{3}{2} \biggl\{ \frac{1}{\mu} \sum
_{j>r}\lambda _j^2\bigl(\Theta ^*
\bigr) + \sum_{j>r}\lambda_j\bigl(
\Theta^*\bigr) + \frac{38}{3} r \mu \biggr\}
\\
&\le&\frac{3}{2} \biggl\{ \frac{1}{\mu'} \sum
_{j>r}\lambda _j^2\bigl(\Theta ^*
\bigr) + \sum_{j>r}\lambda_j\bigl(
\Theta^*\bigr) + \frac{38}{3} r \bar\mu \biggr\}. \nonumber
\end{eqnarray}
Here, the second inequality follows because $\gamma_r\le1/9$ and $\|
E\|
_2\le\mu/6$, and the last inequality follows because $\mu' \le\mu
\le
\bar\mu$. Then inequality (\ref{eq:Theta_hat_diagonal_thmm}) is
obtained by minimizing inequality (\ref{eq:Theta_hat_diagonal_1_final})
over $0\le r\le R$.
\end{pf}
\end{appendix}

\section*{Acknowledgements}

We are grateful to Han Liu for helpful discussions and for providing
independent credit to our work. We thank the area Editor and the
referees for their very constructive comments. This research is
supported in part by NSF Grants DMS-10-07444 and DMS-13-10119.


\begin{thebibliography}{48}

\bibitem{Agarwal12}
%
\begin{barticle}[mr]
\bauthor{\bsnm{Agarwal},~\bfnm{Alekh}\binits{A.}},
\bauthor{\bsnm{Negahban},~\bfnm{Sahand}\binits{S.}} \AND
\bauthor{\bsnm{Wainwright},~\bfnm{Martin~J.}\binits{M.J.}}
(\byear{2012}).
\btitle{Noisy matrix decomposition via convex relaxation: Optimal
rates in high dimensions}.
\bjournal{Ann. Statist.}
\bvolume{40}
\bpages{1171--1197}.
\bid{doi={10.1214/12-AOS1000}, issn={0090-5364}, mr={2985947}}
\end{barticle}
%

\bptok{imsref}%
\endbibitem

\bibitem{Bickel08}
%
\begin{barticle}[mr]
\bauthor{\bsnm{Bickel},~\bfnm{Peter~J.}\binits{P.J.}} \AND
\bauthor{\bsnm{Levina},~\bfnm{Elizaveta}\binits{E.}}
(\byear{2008}).
\btitle{Covariance regularization by thresholding}.
\bjournal{Ann. Statist.}
\bvolume{36}
\bpages{2577--2604}.
\bid{doi={10.1214/08-AOS600}, issn={0090-5364}, mr={2485008}}
\end{barticle}
%

\bptok{imsref}%
\endbibitem

\bibitem{BSW2011}
%
\begin{barticle}[mr]
\bauthor{\bsnm{Bunea},~\bfnm{Florentina}\binits{F.}},
\bauthor{\bsnm{She},~\bfnm{Yiyuan}\binits{Y.}} \AND
\bauthor{\bsnm{Wegkamp},~\bfnm{Marten~H.}\binits{M.H.}}
(\byear{2011}).
\btitle{Optimal selection of reduced rank estimators of
high-dimensional matrices}.
\bjournal{Ann. Statist.}
\bvolume{39}
\bpages{1282--1309}.
\bid{doi={10.1214/11-AOS876}, issn={0090-5364}, mr={2816355}}
\end{barticle}
%

\bptok{imsref}%
\endbibitem

\bibitem{Bunea12}
%
\begin{barticle}[auto:parserefs-M02]
\bauthor{\bsnm{Bunea},~\bfnm{Florentina}\binits{F.}} \AND
\bauthor{\bsnm{Xiao},~\bfnm{Luo}\binits{L.}}
(\byear{2015}).
\btitle{On the sample covariance matrix estimator of reduced
effective rank population matrices, with applications to {fPCA}}.
\bjournal{Bernoulli}
\bvolume{21}
\bpages{1200--1230}.
\bid{mr={3338661}}
\end{barticle}
%

\bptok{imsref}%
\endbibitem

\bibitem{Cai10}
%
\begin{barticle}[mr]
\bauthor{\bsnm{Cai},~\bfnm{T.~Tony}\binits{T.T.}},
\bauthor{\bsnm{Zhang},~\bfnm{Cun-Hui}\binits{C.-H.}} \AND
\bauthor{\bsnm{Zhou},~\bfnm{Harrison~H.}\binits{H.H.}}
(\byear{2010}).
\btitle{Optimal rates of convergence for covariance matrix estimation}.
\bjournal{Ann. Statist.}
\bvolume{38}
\bpages{2118--2144}.
\bid{doi={10.1214/09-AOS752}, issn={0090-5364}, mr={2676885}}
\end{barticle}
%

\bptok{imsref}%
\endbibitem

\bibitem{Cai12}
%
\begin{barticle}[mr]
\bauthor{\bsnm{Cai},~\bfnm{T.~Tony}\binits{T.T.}} \AND
\bauthor{\bsnm{Zhou},~\bfnm{Harrison~H.}\binits{H.H.}}
(\byear{2012}).
\btitle{Optimal rates of convergence for sparse covariance matrix estimation}.
\bjournal{Ann. Statist.}
\bvolume{40}
\bpages{2389--2420}.
\bid{doi={10.1214/12-AOS998}, issn={0090-5364}, mr={3097607}}
\end{barticle}
%

\bptok{imsref}%
\endbibitem

\bibitem{Cambani81}
%
\begin{barticle}[mr]
\bauthor{\bsnm{Cambanis},~\bfnm{Stamatis}\binits{S.}},
\bauthor{\bsnm{Huang},~\bfnm{Steel}\binits{S.}} \AND
\bauthor{\bsnm{Simons},~\bfnm{Gordon}\binits{G.}}
(\byear{1981}).
\btitle{On the theory of elliptically contoured distributions}.
\bjournal{J. Multivariate Anal.}
\bvolume{11}
\bpages{368--385}.
\bid{doi={10.1016/0047-259X(81)90082-8}, issn={0047-259X}, mr={0629795}}
\end{barticle}
%

\bptok{imsref}%
\endbibitem

\bibitem{Candes2009}
%
\begin{barticle}[mr]
\bauthor{\bsnm{Cand{\`e}s},~\bfnm{Emmanuel~J.}\binits{E.J.}} \AND
\bauthor{\bsnm{Recht},~\bfnm{Benjamin}\binits{B.}}
(\byear{2009}).
\btitle{Exact matrix completion via convex optimization}.
\bjournal{Found. Comput. Math.}
\bvolume{9}
\bpages{717--772}.
\bid{doi={10.1007/s10208-009-9045-5}, issn={1615-3375}, mr={2565240}}
\end{barticle}
%

\bptok{imsref}%
\endbibitem

\bibitem{Chandrasekaran12a}
%
\begin{barticle}[mr]
\bauthor{\bsnm{Chandrasekaran},~\bfnm{Venkat}\binits{V.}},
\bauthor{\bsnm{Parrilo},~\bfnm{Pablo~A.}\binits{P.A.}} \AND
\bauthor{\bsnm{Willsky},~\bfnm{Alan~S.}\binits{A.S.}}
(\byear{2012}).
\btitle{Latent variable graphical model selection via convex optimization}.
\bjournal{Ann. Statist.}
\bvolume{40}
\bpages{1935--1967}.
\bid{doi={10.1214/11-AOS949}, issn={0090-5364}, mr={3059067}}
\end{barticle}
%

\bptok{imsref}%
\endbibitem

\bibitem{Chandrasekaran11}
%
\begin{barticle}[mr]
\bauthor{\bsnm{Chandrasekaran},~\bfnm{Venkat}\binits{V.}},
\bauthor{\bsnm{Sanghavi},~\bfnm{Sujay}\binits{S.}},
\bauthor{\bsnm{Parrilo},~\bfnm{Pablo~A.}\binits{P.A.}} \AND
\bauthor{\bsnm{Willsky},~\bfnm{Alan~S.}\binits{A.S.}}
(\byear{2011}).
\btitle{Rank-sparsity incoherence for matrix decomposition}.
\bjournal{SIAM J. Optim.}
\bvolume{21}
\bpages{572--596}.
\bid{doi={10.1137/090761793}, issn={1052-6234}, mr={2817479}}
\end{barticle}
%

\bptok{imsref}%
\endbibitem

\bibitem{Demarta05}
%
\begin{barticle}[auto:parserefs-M02]
\bauthor{\bsnm{Demarta},~\bfnm{Stefano}\binits{S.}} \AND
\bauthor{\bsnm{McNeil},~\bfnm{Alexander~J.}\binits{A.J.}}
(\byear{2005}).
\btitle{The $t$ copula and related copulas}.
\bjournal{Int. Stat. Rev.}
\bvolume{73}
\bpages{111--129}.
\end{barticle}
%

\bptok{imsref}%
\endbibitem

\bibitem{Eckart36}
%
\begin{barticle}[auto:parserefs-M02]
\bauthor{\bsnm{Eckart},~\bfnm{Carl}\binits{C.}} \AND
\bauthor{\bsnm{Young},~\bfnm{Gale}\binits{G.}}
(\byear{1936}).
\btitle{The approximation of one matrix by another of lower rank}.
\bjournal{Psychometrika}
\bvolume{1}
\bpages{211--218}.
\end{barticle}
%

\bptok{imsref}%
\endbibitem

\bibitem{Embrechts03}
%
\begin{bincollection}[auto:parserefs-M02]
\bauthor{\bsnm{Embrechts},~\bfnm{Paul}\binits{P.}},
\bauthor{\bsnm{Lindskog},~\bfnm{Filip}\binits{F.}} \AND
\bauthor{\bsnm{McNeil},~\bfnm{Alexander}\binits{A.}}
(\byear{2003}).
\btitle{Modelling dependence with copulas and applications to risk management}.
In \bbooktitle{Handbook of Heavy Tailed Distributions in Finance}
(\beditor{\bfnm{Svetlozar~T.}\binits{S.T.}~\bsnm{Rachev}}, ed.)
\bpages{329--384}.
\blocation{Amsterdam}:
\bpublisher{Elsevier}.
\end{bincollection}
%

\bptok{imsref}%
\endbibitem

\bibitem{Fang02}
%
\begin{barticle}[mr]
\bauthor{\bsnm{Fang},~\bfnm{Hong-Bin}\binits{H.-B.}},
\bauthor{\bsnm{Fang},~\bfnm{Kai-Tai}\binits{K.-T.}} \AND
\bauthor{\bsnm{Kotz},~\bfnm{Samuel}\binits{S.}}
(\byear{2002}).
\btitle{The meta-elliptical distributions with given marginals}.
\bjournal{J. Multivariate Anal.}
\bvolume{82}
\bpages{1--16}.
\bid{doi={10.1006/jmva.2001.2017}, issn={0047-259X}, mr={1918612}}
\end{barticle}
%

\bptok{imsref}%
\endbibitem

\bibitem{Fazel02}
%
\begin{bmisc}[auto:parserefs-M02]
\bauthor{\bsnm{Fazel},~\bfnm{Maryam}\binits{M.}}
(\byear{2002}).
\bhowpublished{Matrix rank minimization with applications.
Ph.D. thesis,
Stanford Univ.}
\end{bmisc}
%

\bptok{imsref}%
\endbibitem

\bibitem{Friedman08}
%
\begin{barticle}[pbm]
\bauthor{\bsnm{Friedman},~\bfnm{Jerome}\binits{J.}},
\bauthor{\bsnm{Hastie},~\bfnm{Trevor}\binits{T.}} \AND
\bauthor{\bsnm{Tibshirani},~\bfnm{Robert}\binits{R.}}
(\byear{2008}).
\btitle{Sparse inverse covariance estimation with the graphical lasso}.
\bjournal{Biostatistics}
\bvolume{9}
\bpages{432--441}.
\bid{doi={10.1093/biostatistics/kxm045}, issn={1468-4357},
mid={NIHMS248717}, pii={kxm045}, pmcid={3019769}, pmid={18079126}}
\end{barticle}
%

\bptok{imsref}%
\endbibitem

\bibitem{HCL13}
%
\begin{barticle}[auto:parserefs-M02]
\bauthor{\bsnm{Han},~\bfnm{Fang}\binits{F.}} \AND
\bauthor{\bsnm{Liu},~\bfnm{Han}\binits{H.}}
(\byear{2015}).
\btitle{Optimal rates of convergence for latent generalized
correlation matrix estimation in transelliptical distribution}.
\bjournal{Bernoulli}.
\bnote{To appear. Available at \arxivurl{arXiv:1305.6916}}.
\end{barticle}
%

\bptok{imsref}%
\endbibitem

\bibitem{Hoeffding63}
%
\begin{barticle}[mr]
\bauthor{\bsnm{Hoeffding},~\bfnm{Wassily}\binits{W.}}
(\byear{1963}).
\btitle{Probability inequalities for sums of bounded random variables}.
\bjournal{J. Amer. Statist. Assoc.}
\bvolume{58}
\bpages{13--30}.
\bid{issn={0162-1459}, mr={0144363}}
\end{barticle}
%

\bptok{imsref}%
\endbibitem

\bibitem{Horn91}
%
\begin{bbook}[mr]
\bauthor{\bsnm{Horn},~\bfnm{Roger~A.}\binits{R.A.}} \AND
\bauthor{\bsnm{Johnson},~\bfnm{Charles~R.}\binits{C.R.}}
(\byear{1991}).
\btitle{Topics in Matrix Analysis}.
\blocation{Cambridge}:
\bpublisher{Cambridge Univ. Press}.
\bid{doi={10.1017/CBO9780511840371}, mr={1091716}}
\end{bbook}
%

\bptok{imsref}%
\endbibitem

\bibitem{Hsu11}
%
\begin{barticle}[mr]
\bauthor{\bsnm{Hsu},~\bfnm{Daniel}\binits{D.}},
\bauthor{\bsnm{Kakade},~\bfnm{Sham~M.}\binits{S.M.}} \AND
\bauthor{\bsnm{Zhang},~\bfnm{Tong}\binits{T.}}
(\byear{2011}).
\btitle{Robust matrix decomposition with sparse corruptions}.
\bjournal{IEEE Trans. Inform. Theory}
\bvolume{57}
\bpages{7221--7234}.
\bid{doi={10.1109/TIT.2011.2158250}, issn={0018-9448}, mr={2883652}}
\end{barticle}
%

\bptok{imsref}%
\endbibitem

\bibitem{Hult02}
%
\begin{barticle}[mr]
\bauthor{\bsnm{Hult},~\bfnm{Henrik}\binits{H.}} \AND
\bauthor{\bsnm{Lindskog},~\bfnm{Filip}\binits{F.}}
(\byear{2002}).
\btitle{Multivariate extremes, aggregation and dependence in
elliptical distributions}.
\bjournal{Adv. in Appl. Probab.}
\bvolume{34}
\bpages{587--608}.
\bid{doi={10.1239/aap/1033662167}, issn={0001-8678}, mr={1929599}}
\end{barticle}
%

\bptok{imsref}%
\endbibitem

\bibitem{Kendall90}
%
\begin{bbook}[auto:parserefs-M02]
\bauthor{\bsnm{Kendall},~\bfnm{Maurice~George}\binits{M.G.}} \AND
\bauthor{\bsnm{Gibbons},~\bfnm{Jean Dickinson}\binits{J.D.}}
(\byear{1990}).
\btitle{Rank Correlation Methods},
\bedition{5}th ed.
\blocation{London}:
\bpublisher{Edward Arnold}.
\bid{mr={1079065}}
\end{bbook}
%

\bptok{imsref}%
\endbibitem

\bibitem{Kluppelberg09}
%
\begin{barticle}[mr]
\bauthor{\bsnm{Kl{\"u}ppelberg},~\bfnm{Claudia}\binits{C.}} \AND
\bauthor{\bsnm{Kuhn},~\bfnm{Gabriel}\binits{G.}}
(\byear{2009}).
\btitle{Copula structure analysis}.
\bjournal{J. R. Stat. Soc. Ser. B. Stat. Methodol.}
\bvolume{71}
\bpages{737--753}.
\bid{doi={10.1111/j.1467-9868.2009.00707.x}, issn={1369-7412}, mr={2749917}}
\end{barticle}
%

\bptok{imsref}%
\endbibitem

\bibitem{Kluppelberg08}
%
\begin{barticle}[mr]
\bauthor{\bsnm{Kl{\"u}ppelberg},~\bfnm{Claudia}\binits{C.}},
\bauthor{\bsnm{Kuhn},~\bfnm{Gabriel}\binits{G.}} \AND
\bauthor{\bsnm{Peng},~\bfnm{Liang}\binits{L.}}
(\byear{2008}).
\btitle{Semi-parametric models for the multivariate tail dependence
function---The asymptotically dependent case}.
\bjournal{Scand. J. Stat.}
\bvolume{35}
\bpages{701--718}.
\bid{doi={10.1111/j.1467-9469.2008.00602.x}, issn={0303-6898}, mr={2468871}}
\end{barticle}
%

\bptok{imsref}%
\endbibitem

\bibitem{Koltchinskii11}
%
\begin{barticle}[mr]
\bauthor{\bsnm{Koltchinskii},~\bfnm{Vladimir}\binits{V.}},
\bauthor{\bsnm{Lounici},~\bfnm{Karim}\binits{K.}} \AND
\bauthor{\bsnm{Tsybakov},~\bfnm{Alexandre~B.}\binits{A.B.}}
(\byear{2011}).
\btitle{Nuclear-norm penalization and optimal rates for noisy low-rank
matrix completion}.
\bjournal{Ann. Statist.}
\bvolume{39}
\bpages{2302--2329}.
\bid{doi={10.1214/11-AOS894}, issn={0090-5364}, mr={2906869}}
\end{barticle}
%

\bptok{imsref}%
\endbibitem

\bibitem{Kruskal58}
%
\begin{barticle}[mr]
\bauthor{\bsnm{Kruskal},~\bfnm{William~H.}\binits{W.H.}}
(\byear{1958}).
\btitle{Ordinal measures of association}.
\bjournal{J. Amer. Statist. Assoc.}
\bvolume{53}
\bpages{814--861}.
\bid{issn={0162-1459}, mr={0100941}}
\end{barticle}
%

\bptok{imsref}%
\endbibitem

\bibitem{Lindskog03}
%
\begin{bincollection}[auto:parserefs-M02]
\bauthor{\bsnm{Lindskog},~\bfnm{Filip}\binits{F.}},
\bauthor{\bsnm{McNeil},~\bfnm{Alexander}\binits{A.}} \AND
\bauthor{\bsnm{Schmock},~\bfnm{Uwe}\binits{U.}}
(\byear{2003}).
\btitle{Kendall's tau for elliptical distributions}.
In \bbooktitle{Credit Risk: Measurement, Evaluation and Management,
Contributions to Economics}
(\beditor{\bfnm{Georg}\binits{G.}~\bsnm{Bol}},
\beditor{\bfnm{Gholamreza}\binits{G.}~\bsnm{Nakhaeizadeh}},
\beditor{\bfnm{Svetlozar~T.}\binits{S.T.}~\bsnm{Rachev}},
\beditor{\bfnm{Thomas}\binits{T.}~\bsnm{Ridder}} \AND
\beditor{\bfnm{Karl-Heinz}\binits{K.-H.}~\bsnm{Vollmer}}, eds.)
\bpages{149--156}.
\blocation{Heidelberg}:
\bpublisher{Physica-Verlag}.
\end{bincollection}
%

\bptok{imsref}%
\endbibitem

\bibitem{Liu12}
%
\begin{barticle}[mr]
\bauthor{\bsnm{Liu},~\bfnm{Han}\binits{H.}},
\bauthor{\bsnm{Han},~\bfnm{Fang}\binits{F.}},
\bauthor{\bsnm{Yuan},~\bfnm{Ming}\binits{M.}},
\bauthor{\bsnm{Lafferty},~\bfnm{John}\binits{J.}} \AND
\bauthor{\bsnm{Wasserman},~\bfnm{Larry}\binits{L.}}
(\byear{2012}).
\btitle{High-dimensional semiparametric {G}aussian copula graphical models}.
\bjournal{Ann. Statist.}
\bvolume{40}
\bpages{2293--2326}.
\bid{doi={10.1214/12-AOS1037}, issn={0090-5364}, mr={3059084}}
\end{barticle}
%

\bptok{imsref}%
\endbibitem

\bibitem{NIPS2012_0380}
%
\begin{bincollection}[auto:parserefs-M02]
\bauthor{\bsnm{Liu},~\bfnm{Han}\binits{H.}},
\bauthor{\bsnm{Han},~\bfnm{Fang}\binits{F.}} \AND
\bauthor{\bsnm{Zhang},~\bfnm{Cun-Hui}\binits{C.-H.}}
(\byear{2012}).
\btitle{Transelliptical graphical models}.
In \bbooktitle{Adv. Neural Inf. Process. Syst.}
(\beditor{\bfnm{P.}\binits{P.}~\bsnm{Bartlett}},
\beditor{\bfnm{F.~C.~N.}\binits{F.C.N.}~\bsnm{Pereira}},
\beditor{\bfnm{C.~J.~C.}\binits{C.J.C.}~\bsnm{Burges}},
\beditor{\bfnm{L.}\binits{L.}~\bsnm{Bottou}} \AND
\beditor{\bfnm{K.~Q.}\binits{K.Q.}~\bsnm{Weinberger}}, eds.)
\bvolume{25}
\bpages{809--817}.
\bpublisher{Neural Information Processing Systems Foundation}.
\end{bincollection}
%

\bptok{imsref}%
\endbibitem

\bibitem{Liu09}
%
\begin{barticle}[mr]
\bauthor{\bsnm{Liu},~\bfnm{Han}\binits{H.}},
\bauthor{\bsnm{Lafferty},~\bfnm{John}\binits{J.}} \AND
\bauthor{\bsnm{Wasserman},~\bfnm{Larry}\binits{L.}}
(\byear{2009}).
\btitle{The nonparanormal: Semiparametric estimation of high
dimensional undirected graphs}.
\bjournal{J. Mach. Learn. Res.}
\bvolume{10}
\bpages{2295--2328}.
\bid{issn={1532-4435}, mr={2563983}}
\end{barticle}
%

\bptok{imsref}%
\endbibitem

\bibitem{Lounici14}
%
\begin{barticle}[mr]
\bauthor{\bsnm{Lounici},~\bfnm{Karim}\binits{K.}}
(\byear{2014}).
\btitle{High-dimensional covariance matrix estimation with missing
observations}.
\bjournal{Bernoulli}
\bvolume{20}
\bpages{1029--1058}.
\bid{doi={10.3150/12-BEJ487}, issn={1350-7265}, mr={3217437}}
\end{barticle}
%

\bptok{imsref}%
\endbibitem

\bibitem{Luo11}
%
\begin{barticle}[auto:parserefs-M02]
\bauthor{\bsnm{Luo},~\bfnm{Xi}\binits{X.}}
(\byear{2013}).
\btitle{Recovering model structures from large low rank and sparse
covariance matrix estimation}.
\bnote{Preprint. Available at \arxivurl{arXiv:1111.1133}.}
\end{barticle}
%

\bptok{imsref}%
\endbibitem

\bibitem{Mitra14}
%
\begin{barticle}[auto:parserefs-M02]
\bauthor{\bsnm{Mitra},~\bfnm{Ritwik}\binits{R.}} \AND
\bauthor{\bsnm{Zhang},~\bfnm{Cun-Hui}\binits{C.-H.}}
(\byear{2014}).
\btitle{Multivariate analysis of nonparametric estimates of large
correlation matrices}.
\bnote{Preprint. Available at \arxivurl{arXiv:1403.6195}.}
\end{barticle}
%

\bptok{imsref}%
\endbibitem

\bibitem{Negahban11}
%
\begin{barticle}[mr]
\bauthor{\bsnm{Negahban},~\bfnm{Sahand}\binits{S.}} \AND
\bauthor{\bsnm{Wainwright},~\bfnm{Martin~J.}\binits{M.J.}}
(\byear{2011}).
\btitle{Estimation of (near) low-rank matrices with noise and
high-dimensional scaling}.
\bjournal{Ann. Statist.}
\bvolume{39}
\bpages{1069--1097}.
\bid{doi={10.1214/10-AOS850}, issn={0090-5364}, mr={2816348}}
\end{barticle}
%

\bptok{imsref}%
\endbibitem

\bibitem{Petz94}
%
\begin{bincollection}[mr]
\bauthor{\bsnm{Petz},~\bfnm{D{\'e}nes}\binits{D.}}
(\byear{1994}).
\btitle{A survey of certain trace inequalities}.
In \bbooktitle{Functional Analysis and Operator Theory ({W}arsaw, 1992)}.
\bseries{Banach Center Publ.}
\bvolume{30}
\bpages{287--298}.
\blocation{Warsaw}:
\bpublisher{Polish Acad. Sci.}
\bid{mr={1285615}}
\end{bincollection}
%

\bptok{imsref}%
\endbibitem

\bibitem{Qi06aquadratically}
%
\begin{barticle}[mr]
\bauthor{\bsnm{Qi},~\bfnm{Houduo}\binits{H.}} \AND
\bauthor{\bsnm{Sun},~\bfnm{Defeng}\binits{D.}}
(\byear{2006}).
\btitle{A quadratically convergent {N}ewton method for computing the
nearest correlation matrix}.
\bjournal{SIAM J. Matrix Anal. Appl.}
\bvolume{28}
\bpages{360--385}.
\bid{doi={10.1137/050624509}, issn={0895-4798}, mr={2255334}}
\end{barticle}
%

\bptok{imsref}%
\endbibitem

\bibitem{Rohde11}
%
\begin{barticle}[mr]
\bauthor{\bsnm{Rohde},~\bfnm{Angelika}\binits{A.}} \AND
\bauthor{\bsnm{Tsybakov},~\bfnm{Alexandre~B.}\binits{A.B.}}
(\byear{2011}).
\btitle{Estimation of high-dimensional low-rank matrices}.
\bjournal{Ann. Statist.}
\bvolume{39}
\bpages{887--930}.
\bid{doi={10.1214/10-AOS860}, issn={0090-5364}, mr={2816342}}
\end{barticle}
%

\bptok{imsref}%
\endbibitem

\bibitem{Saunderson12}
%
\begin{barticle}[mr]
\bauthor{\bsnm{Saunderson},~\bfnm{J.}\binits{J.}},
\bauthor{\bsnm{Chandrasekaran},~\bfnm{V.}\binits{V.}},
\bauthor{\bsnm{Parrilo},~\bfnm{P.~A.}\binits{P.A.}} \AND
\bauthor{\bsnm{Willsky},~\bfnm{A.~S.}\binits{A.S.}}
(\byear{2012}).
\btitle{Diagonal and low-rank matrix decompositions, correlation
matrices, and ellipsoid fitting}.
\bjournal{SIAM J. Matrix Anal. Appl.}
\bvolume{33}
\bpages{1395--1416}.
\bid{doi={10.1137/120872516}, issn={0895-4798}, mr={3028972}}
\end{barticle}
%

\bptok{imsref}%
\endbibitem

\bibitem{Schmit07}
%
\begin{barticle}[mr]
\bauthor{\bsnm{Schmidt},~\bfnm{Erhard}\binits{E.}}
(\byear{1907}).
\btitle{Zur {T}heorie der linearen und nichtlinearen {I}ntegralgleichungen}.
\bjournal{Math. Ann.}
\bvolume{63}
\bpages{433--476}.
\bid{doi={10.1007/BF01449770}, issn={0025-5831}, mr={1511415}}
\end{barticle}
%

\bptok{imsref}%
\endbibitem

\bibitem{Sklar96}
%
\begin{bincollection}[mr]
\bauthor{\bsnm{Sklar},~\bfnm{A.}\binits{A.}}
(\byear{1996}).
\btitle{Random variables, distribution functions, and copulas---A
personal look backward and forward}.
In \bbooktitle{Distributions with Fixed Marginals and Related Topics
({S}eattle, WA, 1993)}
(\beditor{\binits{L.}\bfnm{Ludger} \bsnm{R{\"{u}}schendorf}},
\beditor{\binits{B.}\bfnm{Berthold} \bsnm{Schweizer}}
\AND
\beditor{\binits{M.~D.}\bfnm{Michael~D.} \bsnm{Taylor}}, eds.).
\bseries{Institute of Mathematical Statistics Lecture
Notes---Monograph Series}
\bvolume{28}
\bpages{1--14}.
\blocation{Hayward, CA}:
\bpublisher{IMS}.
\bid{doi={10.1214/lnms/1215452606}, mr={1485519}}
\end{bincollection}
%

\bptok{imsref}%
\endbibitem

\bibitem{Tropp12a}
%
\begin{barticle}[mr]
\bauthor{\bsnm{Tropp},~\bfnm{Joel~A.}\binits{J.A.}}
(\byear{2012}).
\btitle{User-friendly tail bounds for sums of random matrices}.
\bjournal{Found. Comput. Math.}
\bvolume{12}
\bpages{389--434}.
\bid{doi={10.1007/s10208-011-9099-z}, issn={1615-3375}, mr={2946459}}
\end{barticle}
%

\bptok{imsref}%
\endbibitem

\bibitem{Tropp14}
%
\begin{bmisc}[auto:parserefs-M02]
\bauthor{\bsnm{Tropp},~\bfnm{Joel~A.}\binits{J.A.}}
(\byear{2014}).
\bhowpublished{An introduction to matrix concentration inequalities.
Technical report,
California Institute of Technology}.
\end{bmisc}
%

\bptok{imsref}%
\endbibitem

\bibitem{Vershynin12}
%
\begin{bincollection}[mr]
\bauthor{\bsnm{Vershynin},~\bfnm{Roman}\binits{R.}}
(\byear{2012}).
\btitle{Introduction to the nonasymptotic analysis of random matrices}.
In \bbooktitle{Compressed Sensing}
(\beditor{\binits{Y.}\bfnm{Y.} \bsnm{Eldar}}
\AND
\beditor{\binits{G.}\bfnm{G.} \bsnm{Kutyniok}}, eds.)
\bseries{Compressed Sensing, Theory and Application}
\bpages{210--268}.
\blocation{Cambridge}:
\bpublisher{Cambridge Univ. Press}.
\bid{mr={2963170}}
\end{bincollection}
%

\bptok{imsref}%
\endbibitem

\bibitem{Watson92}
%
\begin{barticle}[mr]
\bauthor{\bsnm{Watson},~\bfnm{G.~A.}\binits{G.A.}}
(\byear{1992}).
\btitle{Characterization of the subdifferential of some matrix norms}.
\bjournal{Linear Algebra Appl.}
\bvolume{170}
\bpages{33--45}.
\bid{doi={10.1016/0024-3795(92)90407-2}, issn={0024-3795}, mr={1160950}}
\end{barticle}
%

\bptok{imsref}%
\endbibitem

\bibitem{Xue12a}
%
\begin{barticle}[mr]
\bauthor{\bsnm{Xue},~\bfnm{Lingzhou}\binits{L.}},
\bauthor{\bsnm{Ma},~\bfnm{Shiqian}\binits{S.}} \AND
\bauthor{\bsnm{Zou},~\bfnm{Hui}\binits{H.}}
(\byear{2012}).
\btitle{Positive-definite {$\ell_1$}-penalized estimation of large
covariance matrices}.
\bjournal{J. Amer. Statist. Assoc.}
\bvolume{107}
\bpages{1480--1491}.
\bid{doi={10.1080/01621459.2012.725386}, issn={0162-1459}, mr={3036409}}
\end{barticle}
%

\bptok{imsref}%
\endbibitem

\bibitem{Xue12b}
%
\begin{barticle}[mr]
\bauthor{\bsnm{Xue},~\bfnm{Lingzhou}\binits{L.}} \AND
\bauthor{\bsnm{Zou},~\bfnm{Hui}\binits{H.}}
(\byear{2012}).
\btitle{Regularized rank-based estimation of high-dimensional
nonparanormal graphical models}.
\bjournal{Ann. Statist.}
\bvolume{40}
\bpages{2541--2571}.
\bid{doi={10.1214/12-AOS1041}, issn={0090-5364}, mr={3097612}}
\end{barticle}
%

\bptok{imsref}%
\endbibitem

\bibitem{Yuan12}
%
\begin{barticle}[mr]
\bauthor{\bsnm{Yuan},~\bfnm{Ming}\binits{M.}}
(\byear{2012}).
\btitle{Comment: ``{M}inimax estimation of large covariance matrices under
{$\ell_1$}-norm'' [MR3027084]}.
\bjournal{Statist. Sinica}
\bvolume{22}
\bpages{1373--1375}.
\bid{issn={1017-0405}, mr={3027090}}
\end{barticle}
%

\bptok{imsref}%
\endbibitem

\bibitem{Zhang12}
%
\begin{bmisc}[auto:parserefs-M02]
\bauthor{\bsnm{Zhang},~\bfnm{Cun-Hui}\binits{C.-H.}} \AND
\bauthor{\bsnm{Zhang},~\bfnm{Tong}\binits{T.}}
(\byear{2012}).
\bhowpublished{A general framework of dual certificate analysis for
structured sparse recovery problems.
Technical report,
Rutgers Univ.}
\end{bmisc}
%
\bptok{imsref}%
\endbibitem
\end{thebibliography}
%
%

%




\printhistory
\end{document}